\documentclass[twoside]{article}
\usepackage[accepted]{aistats2025}
\usepackage[utf8]{inputenc} 
\usepackage[T1]{fontenc}    
\usepackage{hyperref}       
\usepackage{url}            
\usepackage{booktabs}       
\usepackage{amsfonts}       
\usepackage{nicefrac}       
\usepackage{microtype}      
\usepackage{xcolor}         
\usepackage{amsthm} 
\usepackage{glossaries}
\usepackage{amsfonts}
\usepackage{amsmath,graphicx,amssymb,amsthm}
\usepackage{comment}
\usepackage{bbm}
\usepackage{appendix}
\usepackage{tikz-cd}
\usepackage[boxruled,linesnumbered]{algorithm2e}
\usepackage{setspace}
\usepackage{url}
\usepackage{enumitem}
\usepackage{hyperref}
\usepackage{bm}
\usepackage{float}
\usepackage[normalem]{ulem}
\usepackage{wrapfig}
\usepackage{lipsum} 

\usepackage{subcaption} 

\newtheorem{set}{set}[section]
\newtheorem{theorem}[set]{Theorem}

\newtheorem{corollary}[set]{Corollary}

\newtheorem{definition}[set]{Definition}

\newtheorem{lemma}[set]{Lemma}

\newtheorem{proposition}[set]{Proposition}
\newtheorem{remark}[set]{Remark}

\DeclareMathOperator*{\argmin}{arg\,min}
\newcommand{\entmap}[2]{T^\epsilon_{{#1}\rightarrow {#2}}}

\usepackage[round]{natbib}

\SetAlCapHSkip{1em}
\SetNlSkip{.3em}
\begin{document}


\twocolumn[

\aistatstitle{Synthesis and Analysis of Data as Probability Measures With Entropy-Regularized Optimal Transport}

\aistatsauthor{Brendan Mallery$^{(\dagger)}$, James M. Murphy$^{(\dagger,*)}$, Shuchin Aeron $^{(\square,*)}$ }

\aistatsaddress{Tufts University}]

\begin{abstract}
We consider synthesis and analysis of probability measures using the entropy-regularized Wasserstein-2 cost and its unbiased version, the Sinkhorn divergence. The synthesis problem consists of computing the barycenter, with respect to these costs, of reference measures given a set of coefficients belonging to the simplex. The analysis problem consists of finding the coefficients for the closest barycenter in the Wasserstein-2 distance to a given measure. Under the weakest assumptions on the measures thus far in the literature, we compute the derivative of the entropy-regularized Wasserstein-2 cost.  We leverage this to establish a characterization of barycenters with respect to the entropy-regularized Wasserstein-2 cost as solutions that correspond to a fixed point of an average of the entropy-regularized displacement maps.  This characterization yields a finite-dimensional, convex, quadratic program for solving the analysis problem when the measure being analyzed is a barycenter with respect to the entropy-regularized Wasserstein-2 cost. We show that these coefficients, as well as the value of the barycenter functional, can be estimated from samples with dimension-independent rates of convergence, and that barycentric coefficients are stable with respect to perturbations in the Wasserstein-2 metric. We employ the barycentric coefficients as features for classification of corrupted point cloud data, and show that compared to neural network baselines, our approach is more efficient in small training data regimes.
\end{abstract}

\section{INTRODUCTION}

 Modeling data as probability measures is an emerging theme with applications in signal analysis \citep{cazelles2020wasserstein, cheng2023non}, image processing \citep{rabin2012wasserstein,nadeem2020multimarginal,simon2020barycenters}, natural language processing (NLP) \citep{xu2018distilled} and beyond. In this setting, each data point is treated as a probability distribution over an appropriate domain. We seek to develop useful data processing models and methods---akin to those widely used in statistical signal processing, such as principal component analysis (PCA) \citep{pearson1901liii}, archetypal analysis \citep{cutler1994archetypal}, and non-negative matrix factorization (NMF) \citep{lee2000algorithms}---that allow for \emph{synthesis} of new distributions from a canonical reference set of distributions, as well as \emph{analysis} of a given distribution in terms of these references.

{Wasserstein-2 barycenters} provide a natural method to perform both synthesis and analysis of data using the methods of computational optimal transport (OT) \citep{peyre2019computational}. One can synthesize a new distribution from a given set of reference measures with their barycenter \citep{dognin2019wasserstein,rabin2012wasserstein,chzhen2020fair}, and analyze a measure by assigning it barycentric coefficients with respect to a set of reference measures \citep{bonneel2016wasserstein, schmitz2018wasserstein, werenski2022measure, mueller2023geometrically}.  Wasserstein-2 barycenters can geometrically interpolate between the reference measures, which allows one to model data outside of their support that is useful for image processing \citep{simon2020barycenters} and time-series modeling \citep{cheng2021dynamical, cheng2023non}. This is in contrast to barycenters defined via maximum mean discrepancy distances, which are linear mixtures of the reference measures \citep{cohen2020estimating}, or via information theoretic divergences such as relative entropy \citep{cover1999elements}, which require overlapping support of the underlying measures. 

However, computation of optimal transport costs is not parallelizable and can incur computational complexity that is cubic in the number of samples. Furthermore, synthesizing even an approximate Wasserstein-2 barycenter to a fixed degree of accuracy in arbitrary dimensions cannot be done in polynomial time \citep{altschuler2022wasserstein}, and known rates of estimation of both synthesis and analysis scale exponentially in the ambient dimension $d$ \citep{werenski2022measure} due to the curse of dimensionality that is inherent to estimating the Wasserstein-2 costs and maps \citep{fournier2015rate, hutter2021minimax, pooladian2021entropic, deb2021rates}.

On the other hand, entropy-regularized optimal transport does not suffer from these computational and statistical issues. Fast computation via parallelization is enabled via the Sinkhorn-Knopp algorithm \citep{sinkhorn1967concerning, cuturi2013sinkhorn} and the population entropy-regularized optimal transport costs and maps can be estimated from samples with dimension-free rates of convergence \citep{mena2019,bigot2019central,pooladian2021entropic,rigollet2022sample,werenski2023estimation, masud2023multivariate, stromme2023sampling}.

Motivated by this, in this paper we consider entropy-regularized variants of Wasserstein-2 barycenters, which are defined in (\ref{eqn:syntheqn}) and (\ref{eqn:analysiseqn}). We are particularly concerned with analysis of data via (regularized) barycentric coefficients, which has until recently been relatively under-studied in comparison to synthesis. 

The rest of the paper is organized as follows. In Section \ref{subsec:background_notation} we present the necessary notation and background. We then detail our main contributions and related work in Section \ref{subsec:contributions}. The main technical results are established in Sections \ref{sec:analytic properties}, \ref{sec:sample_complexity_functional}, and \ref{sec:analysis_problem} along with numerical results. Applications to point cloud classification are presented in Section \ref{sec:point-cloud-classification}\footnote{Code to recreate experiments is available at https://github.com/brendanmallery9/Entropic-Barycenters}. We defer our proofs and additional experiments to the Appendix. 

\subsection{Notation and Background}
\label{subsec:background_notation}

\noindent\textbf{Notation:} $\Omega$ is a closed subset of $\mathbb{R}^d$ and $\mathcal{P}(\Omega)$ is the set of probability measures with support contained in $\Omega$. $|\Omega|$ is the diameter of $\Omega$ with respect to the Euclidean distance. For $\mu\in \mathcal{P}(\Omega)$, its expectation is denoted by $\mathbb{E}(\mu)$, its variance is denoted by $\texttt{Var}(\mu)$, and its second moment is denoted by $M_{2}(\mu)$. $\mathcal{P}_2(\Omega)$ denotes the subset of $\mu\in \mathcal{P}(\Omega)$ with finite $M_2(\mu)$. $\mathcal{G}(\Omega)$ denotes the set of subgaussian probability measures with support contained in $\Omega$, and $\mathcal{G}_\sigma(\Omega)$ is the set of subgaussian probability measures with support in $\Omega$ with subgaussian constant $\leq \sigma$. $\mathcal{P}^n(\Omega)$ denotes the subset of $\mathcal{P}(\Omega)$ supported on $n$ or fewer points. We let $\delta_{z}$ denote the Dirac probability measure supported at $z\in\mathbb{R}^{d}$. For a probability measure $\mu$, $\hat{\mu}^n$ is the random (empirical) measure defined by $\hat{\mu}^n=\frac{1}{n}\sum_{i=1}^n \delta_{X_i}$, where $X_1,X_2,...,X_n$ are i.i.d. samples from $\mu$. The ball of radius (with respect to the Euclidean distance) $t$ centered at $x\in \mathbb{R}^d$ is denoted by $B_t(x).$ For a function $f:\mathbb{R}^d\rightarrow \mathbb{R}$ and a finite, signed measure $\mu$ on $\mathbb{R}^d$, $\langle f, \mu\rangle :=\int f d\mu$. For any $\mu \ll \nu$, $\frac{d \mu}{d \nu}$ denotes the Radon-Nikodym derivative of $\mu$ with respect to $\nu$, i.e., a measurable function such that $\int_A \frac{d\mu}{d\nu} d\nu = \mu(A)$, for any measurable $A$. 
 
For $\mu,\nu\in\mathcal{P}_2(\Omega)$, the \emph{Wasserstein-2 distance} \citep{villani2009optimal} between $\mu$ and $\nu$ is:
\begin{equation*}
OT_2(\mu,\nu):=\inf_{\zeta\in \Pi(\mu,\nu)} \sqrt{\int \frac{1}{2}\|x-y\|^2d\zeta(x,y)},
\end{equation*} where $\Pi(\mu,\nu)\subset\mathcal{P}_{2}(\Omega\times\Omega)$ is the set of couplings with marginals $\mu$ and $\nu$. For any $\epsilon>0$, the \emph{entropy-regularized, or entropy-regularized Wasserstein-2 cost} is \citep{nutz2021introduction}: 
\begin{align}
OT_2^\epsilon(\mu,\nu):=&\inf_{\zeta\in\Pi(\mu,\nu)}\int \frac{1}{2}\|x-y\|^2 d\zeta(x,y)\label{eqn:entot} \\
+&\epsilon KL(\zeta\|\mu\otimes\nu)\notag,
\end{align}
where $KL(\zeta\|\mu\otimes\nu)$ is the relative entropy between $\zeta$ and $\mu\otimes \nu$ \citep{cover1999elements}. For any $\epsilon>0$, we define the \emph{Sinkhorn divergence} \citep{genevay2018learning}: $\mathcal{S}_2^\epsilon(\mu,\nu)=OT_2^\epsilon(\mu,\nu)-\frac{1}{2}OT_2^\epsilon(\mu,\mu)-\frac{1}{2}OT_2^\epsilon(\nu,\nu).$ 

The entropy-regularized cost admits a dual formulation given by:
\begin{align}
& OT^\epsilon_2(\mu,\nu)\notag \\ & =\sup_{(f,g)\in L^1(\mu)\times L^1(\nu)}  \int f(x) d\mu(x) +\int g(y) d\nu(y) \notag\\
& -\epsilon\int \left(e^{{-\frac{1}{\epsilon}\left(\frac{1}{2}\|x-y\|^2-f(x)-g(y)\right)}}-1\right)d\mu(x)d\nu(y) \label{eqn:dualform}.
\end{align} 
The optimizers are a pair of functions $(f_{\mu\rightarrow \nu}^\epsilon,g^\epsilon_{\mu\rightarrow \nu})$, unique $\mu$ (resp. $\nu$)-a.e. up to translation $(f_{\mu\rightarrow \nu}^\epsilon,g_{\mu\rightarrow \nu}^\epsilon)\mapsto (f_{\mu\rightarrow \nu}^\epsilon-C,g_{\mu\rightarrow \nu}^\epsilon+C), \,\, C\in\mathbb{R}$, which we refer to as \emph{entropic potentials}. They satisfy:
\begin{align}\label{eqn:fduality}
&f_{\mu\rightarrow\nu}^\epsilon(x)=-\epsilon \log \left(\int e^{-\frac{1}{\epsilon}\left(\frac{1}{2}\|x-y\|^2 -g^\epsilon_{\mu\rightarrow\nu}(y)\right)}d\nu(y)\right),
\end{align}
\begin{align}\label{eqn:gduality}
& g_{\mu\rightarrow\nu}^\epsilon(y)=-\epsilon \log \left(\int e^{-\frac{1}{\epsilon}\left(\frac{1}{2}\|x-y\|^2 -f^\epsilon_{\mu\rightarrow\nu}(x)\right)}d\mu(x)\right),
\end{align}
for $\mu$-a.e. $x\in \mathbb{R}^d$ and $\nu$-a.e. $y\in \mathbb{R}^d$ \citep{nutz2021introduction}.
 These equations allow one to define canonical extensions of $f^\epsilon_{\mu\rightarrow \nu}$ and $g^\epsilon_{\mu\rightarrow \nu}$ to \textit{all} $x,y\in \mathbb{R}^d$ such that \eqref{eqn:fduality} and \eqref{eqn:gduality} are satisfied for all $x,y\in \mathbb{R}^d$. We refer to these as \emph{extended entropic potentials}. These extensions are unique for a fixed pair of entropic potentials $f^\epsilon_{\mu\rightarrow \nu}$ and $g^\epsilon_{\mu\rightarrow \nu}$. 
 
 Henceforth, we will only deal with extended entropic potentials. The extensions of $f^\epsilon_{\mu\rightarrow \nu}$ and $g^\epsilon_{\mu\rightarrow \nu}$ are a priori only finite $\mu$ and $\nu$-almost everywhere respectively, but if $\mu,\nu\in\mathcal{G}(\mathbb{R}^{d})$ then $f_{\mu\rightarrow\nu}^{\epsilon}$ and $g_{\mu\rightarrow\nu}^{\epsilon}$ are finite and smooth everywhere \citep{mena2019}. However it can be shown that $f_{\mu\rightarrow \nu}^\epsilon$ is finite and smooth everywhere if $\nu$ is subgaussian and $\mu$ is in $\mathcal{P}_2(\mathbb{R}^d)$, i.e., only subgaussianity of the ``target'' measure is required for one of the potentials to be smooth and finite everywhere (see Proposition \ref{prop:sinkhorn_differentiable}) 

Let $\mathcal{V}=\{\nu_j\}_{j=1}^m\subset \mathcal{P}_{2}(\mathbb{R}^d)$. Let $\allowdisplaybreaks\lambda$ denote the vector of coefficients belonging to the $(m-1)$-dimensional simplex $\Delta^m = \{(\lambda_1,\lambda_2,...,\lambda_m)\in \mathbb{R}^m : \sum_{i=1}^{m} \lambda_i = 1, \lambda_i\geq 0\}$.

\begin{definition}\label{def:barycenter functionals}
Let $\mu\in\mathcal{P}_2(\mathbb{R}^d)$ and let $\epsilon>0$. 
\begin{itemize}
    \item The \emph{$\epsilon$-entropic barycenter functional} for $(\lambda,\mathcal{V})$ is defined as $F_{\lambda,\mathcal{V}}^\epsilon(\mu):=\sum_{j=1}^m\lambda_jOT^\epsilon_2(\mu,\nu_j)$.
    \item The \emph{$\epsilon$-Sinkhorn barycenter functional} for $(\lambda,\mathcal{V})$ is defined as: $S_{\lambda,\mathcal{V}}^\epsilon(\mu):=\sum_{j=1}^m\lambda_j\mathcal{S}^\epsilon_2(\mu,\nu_j) .$
\end{itemize} 
\end{definition}
We write $\mathcal{F}^\epsilon_{\lambda,\mathcal{V}}$ to refer to either $F^\epsilon_{\lambda,\mathcal{V}}$ or $S^\epsilon_{\lambda,\mathcal{V}}$. We note that $F^\epsilon_{\lambda,\mathcal{V}}$ is convex on $\mathcal{P}_2(\mathbb{R}^d)$, and $S^\epsilon_{\lambda,\mathcal{V}}$ is convex when restricted to the set of subgaussian measures (Proposition 4 in \citep{janati2020debiased}).  We now define the synthesis and analysis problems. 

 Given $\lambda$ and $\mathcal{V}$, the \emph{synthesis} problem refers to solving for:
\begin{equation} \argmin_{\mu\in \mathcal{P}_2(\Omega)}\mathcal{F}^\epsilon_{\lambda,\mathcal{V}}(\mu).\label{eqn:syntheqn}\end{equation} 
Given $\mu$ and $\mathcal{V}$, the \emph{analysis} problem refers to solving for: \begin{equation} \argmin_{\lambda\in \Delta^m}OT_2(\mu,\rho_\lambda),\label{eqn:analysiseqn}\end{equation} where $\rho_\lambda=\argmin_{\rho\in \mathcal{P}_2(\Omega)}\mathcal{F}^\epsilon_{\lambda,\mathcal{V}}(\rho).$ 
Solving (\ref{eqn:syntheqn}) produces a measure that can be viewed as a 
nonlinear combination of the reference measures $\mathcal{V}$. Solving (\ref{eqn:analysiseqn}) produces coefficients $\lambda$ that parameterize the ``projection'' of $\mu$ onto the set of solutions to (\ref{eqn:syntheqn}), which we refer to as the barycentric coefficients of $\mu$ with respect to the reference measures. The synthesis and analysis problems are in some sense inverse to one another when $\mu$ in (\ref{eqn:analysiseqn}) is restricted to the set of minimizers of $\mathcal{F}^\epsilon_{\lambda_*,\mathcal{V}}$, for some $\lambda_*\in \Delta^m$. Indeed, if $\mu\in \argmin_{\rho\in \mathcal{P}_2(\Omega)}\mathcal{F}^\epsilon_{\lambda_*,\mathcal{V}}(\rho)$, then $\min_{\lambda\in \Delta^m}OT_2(\mu,\rho_\lambda)=OT_2(\mu,\rho_{\lambda_*})=0$. In other words, the analysis of a $(\lambda_*,\mathcal{V})$-barycenter recovers $\lambda_*$. In this paper, we show that for this special case it is possible to reduce (\ref{eqn:analysiseqn}) to a convex, quadratic program (see Section \ref{sec:analysis_problem}).

We end this section with definitions from functional calculus.
\begin{definition}
\label{def:derivative_def}
Let $\mathcal{U}\subseteq \mathcal{P}_2(\mathbb{R}^d)$ be convex. A functional $\mathcal{F}:\mathcal{U}\rightarrow\mathbb{R}$ admits a \emph{derivative} at $\mu\in \mathcal{U}$ (is differentiable at $\mu$) if there exists a continuous function $\delta_\mu\mathcal{F}(\mu):\mathbb{R}^d\rightarrow \mathbb{R}$, such that for any displacement $\chi:=\rho-\mu$, where $\rho\in \mathcal{U}$, $\delta_\mu\mathcal{F}(\mu)$ is $\chi$-integrable and satisfies:
\begin{equation}\label{eqn:FirstVariation}
\lim_{t\rightarrow 0^{+}} \frac{\mathcal{F}(\mu+t\chi)-\mathcal{F}(\mu)}{t}=\int \delta_\mu\mathcal{F}(\mu) d\chi.
\end{equation}
A functional which is differentiable for all $\mu\in \mathcal{U}$ is said to be \emph{differentiable on $\mathcal{U}$}. 
\end{definition}
When the argument of $\mathcal{F}$ is clear, we will write $\delta_\mu\mathcal{F}(\mu)$ as $\delta \mathcal{F}(\mu)$. Key to our investigation is the notion of critical points of functionals of probability measures.
\begin{definition}\label{defn:critical_points}
Let $\mathcal{U}$ be either $\mathcal{P}_2(\mathbb{R}^d)$ or $\mathcal{G}(\mathbb{R}^d)$, and let $\mathcal{F}$ be a functional on $\mathcal{U}$ which is differentiable at $\mu\in \mathcal{U}$. We say that $\mu$ is a \emph{critical point} of $\mathcal{F}$ on $\mathcal{U}$ if $\nabla\delta \mathcal{F}(\mu)(x)=0$ for $\mu$-almost every $x$. 
\end{definition}
We will simply say that $\mu$ is a critical point of $\mathcal{F}$ if the domain of $\mathcal{F}$ is clear. Importantly, minimizers of a differentiable functional on $\mathcal{U}$ are critical points (see Corollary \ref{cor:mins_are_fixed}). We remark that under sufficient regularity assumptions on $\mathcal{F}$ (e.g., as made in \citep{chizat2022mean}), $\nabla\delta\mathcal{F}(\mu)$ corresponds to what is referred to as the Wasserstein-2 gradient of the functional $\mathcal{F}$ at $\mu$ \citep{ambrosio2005gradient}.

\subsection{Main Contributions}
\label{subsec:contributions}

\subsubsection{Critical Points and Optimality Criteria for $\mathcal{F}^\epsilon_{\lambda,\mathcal{V}}$} 

In Section \ref{sec:analytic properties}, we rigorously establish analytic properties of $\mathcal{F}^\epsilon_{\lambda,\mathcal{V}}$ for subgaussian $\mathcal{V}$. First, we establish existence of minima for $\min_{\mu\in \mathcal{P}_2(\Omega)}F^\epsilon_{\lambda,\mathcal{V}}$ and $\min_{\mu\in \mathcal{G}_\sigma(\Omega)}S^\epsilon_{\lambda,\mathcal{V}}$ for fixed $\sigma$. We then establish Theorem \ref{thm:entot_differentiable}, in which we prove that the entropic cost $OT_2^\epsilon(\mu,\nu)$ (as a function of $\mu$) is differentiable when $\nu$ is subgaussian. This extends Proposition 3 in \citep{janati2020debiased}, which establishes this when both $\mu$ and $\nu$ are subgaussian, and to our knowledge provides the most general version of this result available in the literature. From this we obtain a first-order characterization of critical points of $\mathcal{F}^\epsilon_{\lambda,\mathcal{V}}$. We show in Proposition \ref{prop:subgauss} that when $\mathcal{V}$ is subgaussian, then any critical point is subgaussian. This generalizes results in \citep{yang2024estimating} for the case of $\mathcal{V}=\{\nu\}$. Finally, we obtain a sufficient criteria to ensure a critical point of $\mathcal{F}^\epsilon_{\lambda,\mathcal{V}}$ is a minimizer in Corollary \ref{cor:optimality}. 

\subsubsection{Sample Complexity for Synthesis and Analysis} 
In Sections \ref{sec:sample_complexity_functional} and \ref{sec:analysis_problem} we investigate the sample complexity of synthesis and analysis for $\mathcal{F}^\epsilon_{\lambda,\mathcal{V}}$. 
 
 Our main result for synthesis is Theorem \ref{thm:synthsampcomplex}, which establishes that estimating the optimal value of the barycenter functionals using samples from the reference measures can be achieved with \emph{dimension-free} error rates. Our proof is a generalization of the arguments in \citep{yang2024estimating} where this result is proved for $m=1$, and contrasts with the Wasserstein-2 case where rates are typically exponential in the dimension (e.g., Theorem 2 in \citep{chizat2020faster}). A similar result for measures supported on a bounded set was achieved in \citep{luise2019sinkhorn}, though they are also able to guarantee bounds with high probability.  We then discuss how a potentially stronger sample complexity result fails to be true in the case of the Sinkhorn divergence barycenter.
 
Next, we consider the analysis problem. In Proposition \ref{prop:analysis_grad}, we establish that when $\mu$ is an $\epsilon$-entropic (resp. Sinkhorn) barycenter for a reference set of measures $\mathcal{V}$, then the analysis problem (\ref{eqn:analysiseqn}) can be solved via a convex, quadratic program. In Theorem \ref{thm:coeff_theorem}, we use this characterization to show that under weak assumptions, solving the analysis problem for an $\epsilon$-entropic barycenter $\mu$ is no harder than estimating $\nabla\delta\mathcal{F}^\epsilon_{\lambda,\mathcal{V}}(\mu)$ from samples. We apply this in two important cases. First, if $\Omega$ is compact, $\lambda$ may be estimated for $\mathcal{F}^\epsilon_{\lambda,\mathcal{V}}$ at a rate proportional to $n^{-1/2}$ via results obtained in \citep{rigollet2022sample}. Second, if $\Omega=\mathbb{R}^d$ and the reference measures are strongly log-concave with mean zero, then $\lambda$ may be estimated for $\mathcal{F}^\epsilon_{\lambda,\mathcal{V}}$ at a rate proportional to $n^{-1/3}$, via results from \citep{werenski2023estimation}. Finally, we show that the barycentric coefficients are stable (with respect to the Wasserstein-2 distance) when the analyzed measure is perturbed (Proposition \ref{prop:stab_analysis}). We numerically verify the rates in Theorem \ref{thm:coeff_theorem} in Section \ref{sec:numericalexp} with several choices of reference measures, observing even faster convergence rates than established in Theorem \ref{thm:coeff_theorem}.  We provide an additional demonstration of the synthesis and analysis pipeline with probability measures obtained from MNIST digits \citep{lecun1998mnist} in Appendix \ref{sec:synth_and_analysis_MNIST}.

\subsubsection{Application to Point Cloud Classification with Barycentric coefficients} 

We demonstrate the utility of our approach by using the estimated barycentric coefficients as features for point cloud classification. Our method is described in Algorithm \ref{alg:point_cloud_classification}, that makes use of Proposition \ref{prop:analysis_grad}. In our experiments (detailed in Section \ref{sec:point-cloud-classification}), our method is shown to achieve high accuracy using a small set of reference measures from each class, even when the test data is corrupted with various forms of noise and occlusions. We compare our method to PointNet \citep{qi2017pointnet}, a well-known neural network architecture designed for point cloud classification. We observe that our method is able to outperform PointNet even when PointNet is trained on a much larger data set than is required by our method, suggesting the effectiveness of our approach in low-data settings. We compare our method to the analogous classification scheme using the unregularized barycenter functional (in which case, our method agrees with the method studied in \citep{gunsilius2024tangential}) and the doubly-regularized barycenter functional introduced in \citep{chizat2023doubly}. We provide an application to point cloud completion in Appendix \ref{sec:completion}.

\subsection{Related Work}\label{subsec:relatedwork}

\noindent Synthesis for $F^\epsilon_{\lambda,\mathcal{V}}$ was first studied in \citep{cuturi2014fast} as a computationally efficient surrogate for Wasserstein-2 barycenters. Synthesis for the functional $S^\epsilon_{\lambda,\mathcal{V}}$ was first considered in the fixed support setting in \citep{luise2018differential}. These barycenters were studied at length in \citep{janati2020debiased} under a subgaussian assumption on the reference measures. Several free support synthesis algorithms have been proposed \citep{luise2019sinkhorn},\citep{shen2020sinkhorn}. Recent work of \citep{chizat2023doubly} introduced \emph{doubly-regularized} entropic barycenters, which have good regularity and sample complexity guarantees \citep{vavskevivcius2023computational}. 

The analysis problem has received comparatively less attention. The earliest work in this setting is \citep{bonneel2016wasserstein}, where barycentric coefficients are applied to perform regression with histogram data for computer vision applications. In \citep{schmitz2018wasserstein} barycentric coefficients are applied to dictionary learning for histograms, and these ideas are extended in \citep{mueller2023geometrically} where the focus is on sparse representations. In \citep{werenski2022measure} the analysis problem using unregularized Wasserstein-2 barycenters is considered with finite sample guarantees but under strong regularity assumptions on measures. In \citep{gunsilius2024tangential} similar ideas utilizing a model based on the tangential structure of the Wasserstein-2 space are pursued with applications to causal inference, albeit without finite sample guarantees. In \citep{janati2020debiased}, barycentric coefficients for the Sinkhorn barycenter functional are used to classify image data, modeled as probability measures on a fixed support. To our knowledge, our work is the first to consider analysis with regularized barycenter functionals in the free support setting under mild assumptions on the measures, with finite sample guarantees as well as provable stability guarantees under measure perturbations.

\section{ANALYTIC PROPERTIES OF $\mathcal{F}^\epsilon_{\lambda,\mathcal{V}}$}\label{sec:analytic properties}

We begin by establishing existence results for the functionals $\mathcal{F}^\epsilon_{\lambda,\mathcal{V}}$:

\begin{proposition}\label{prop:existence}
Let $\mathcal{V}=\{\nu_{j}\}_{j=1}^{m}\subset\mathcal{P}_{2}(\Omega)$. Then for any $\epsilon>0$, $F^\epsilon_{\lambda,\mathcal{V}}$ admits a minimizer over $\mathcal{P}_2(\Omega)$, and for any $\sigma>0$, $S^\epsilon_{\lambda,\mathcal{V}}$ admits a minimizer over $\mathcal{G}_\sigma(\Omega)$. If $\Omega$ is bounded, then the minimizer of $S^\epsilon_{\lambda,\mathcal{V}}$ is unique. 
\end{proposition}

Next, we characterize the derivatives and critical points for the functionals $F^\epsilon_{\lambda,\mathcal{V}}$ and $S^\epsilon_{\lambda,\mathcal{V}}$ when $\mathcal{V}\subset \mathcal{G}(\mathbb{R}^d)$. The following is a generalization of Proposition 3 in \citep{janati2020debiased}, which characterizes the derivatives of $OT_2^\epsilon(\mu,\nu)$ when both $\mu$ and $\nu$ are subgaussian. We relax this assumption, and show that as long as $\nu$ is subgaussian, $\mu\mapsto OT_2^\epsilon(\mu,\nu)$ admits a derivative in the sense of Definition \ref{def:derivative_def}.
\begin{theorem}\label{thm:entot_differentiable}
Let $\nu\in\mathcal{G}(\mathbb{R}^d)$. Then $OT^\epsilon_2(\cdot,\nu): \mathcal{P}_2(\mathbb{R}^d) \rightarrow \mathbb{R}$ is differentiable, with derivative $\delta_{\mu} OT^{\epsilon}_{2}(\mu,\nu)=f^\epsilon_{\mu\rightarrow \nu}$, which is unique up to an additive constant. 
\end{theorem}
We sketch the proof, which is deferred to Section \ref{sec:proof_of_entot_diff}. We essentially follow the standard argument used in [Proposition B.1, \citep{janati2020debiased}]. The most technical step is establishing pointwise convergence of $f_{\mu+t\chi\rightarrow \nu}^\epsilon$ to $f_{\mu\rightarrow\nu}^\epsilon$ as $t\rightarrow 0^+$ for $\chi:=(\rho-\mu)$ for any $\rho\in \mathcal{P}_2(\mathbb{R}^d)$ using \emph{only} the subgaussian assumption on $\nu$, in contrast to \citep{janati2020debiased} where both measures are assumed to be subgaussian. We achieve this by using recent results from \citep{nutz2022entropic} to establish convergence in probability of the potentials, which can then be upgraded to pointwise convergence using (\ref{eqn:fduality}) and the subgaussian condition on $\nu$ that in turn implies that $|f_{\mu\rightarrow\nu}^\epsilon (x)|$ is at most of quadratic growth in $x$ when $\nu$ is subgaussian.

We remark that the proof of Theorem \ref{thm:entot_differentiable} also follows by following a strategy pioneered in \citep{mena2019} in the context of convergence of potentials associated to empirical measures. 
 Indeed, it can be shown that for each fixed $x$, $\nabla f_{\mu+t\chi\rightarrow\nu}(x)$ is uniformly bounded over $t$.  Hence, any family of potentials $\{f_{\mu+t\chi\rightarrow\nu}\}_{t\in[0,1]}$ considered in the context of Theorem \ref{thm:entot_differentiable} is locally equicontinuous. By the Arzela-Ascoli Lemma \citep{royden2010real} this implies that there exists a subsequence that converges pointwise to $f_{\mu\rightarrow\nu}$. Then the pointwise convergence of the original sequence $\{f_{\mu_{t}\rightarrow\nu}\}_{t\in[0,1]}$ can be deduced under the normalization conditions imposed as in our proof of Theorem \ref{thm:entot_differentiable}. We leave the details to the reader. We note that proof of Theorem \ref{thm:entot_differentiable} as supplied does not require the following Proposition 2.3 to establish the gradients of the potentials, so it is of independent interest and may generalize to other scenarios for other functionals.

We now note the following proposition regarding regularity properties of $f_{\mu\rightarrow\nu}^{\epsilon}$.

\begin{proposition}\label{prop:sinkhorn_differentiable}
Let $\nu\in \mathcal{G}(\mathbb{R}^d)$ and let $\mu\in \mathcal{P}_2(\mathbb{R}^d)$. Then $f^\epsilon_{\mu\rightarrow \nu}$ is infinitely differentiable, with gradient $\nabla f^\epsilon_{\mu\rightarrow \nu}(x)=x-T^\epsilon_{\mu\rightarrow \nu}(x)$, where\begin{equation}\label{eqn:gradformula}T^\epsilon_{\mu\rightarrow \nu}(x):=\dfrac{\displaystyle\int y e^{\left(\frac{1}{\epsilon}\left(-\frac{1}{2}\|x-y\|^2+g^\epsilon_{\mu\rightarrow \nu}(y)\right)\right)}d\nu(y)}{\displaystyle\int e^{\left(\frac{1}{\epsilon}\left(-\frac{1}{2}\|x-y\|^2+g^\epsilon_{\mu\rightarrow \nu}(y)\right)\right)}d\nu(y)}.\end{equation}
\end{proposition}
\noindent Following \citep{pooladian2021entropic}, we refer to $T^\epsilon_{\mu\rightarrow \nu}$ as the \emph{entropic map} from $\mu$ to $\nu$. Proposition \ref{prop:sinkhorn_differentiable} allows us to state the following result that characterizes the critical points of the barycenter functionals in terms of entropic maps.

\begin{corollary}\label{cor:first_var for_functionals}\label{lem:SufficientFixed}
Let $\mathcal{V}\subset\mathcal{G}(\mathbb{R}^d)$.
\begin{itemize}[leftmargin=*]
  \item Let $\mu\in \mathcal{P}_2(\mathbb{R}^d)$. Then $F^\epsilon_{\lambda,\mathcal{V}}$ admits a derivative at $\mu$, given by $\delta F_{\lambda,\mathcal{V}}^\epsilon(\mu)=\sum_{j=1}^m \lambda_j f^\epsilon_{\mu\rightarrow \nu_j}$. Hence $\mu^\epsilon$ is a critical point for $F^\epsilon_{\lambda,\mathcal{V}}$on $\mathcal{P}_2(\mathbb{R}^d)$ if and only if $x-\sum_{j=1}^m\lambda_jT^\epsilon_{\mu^\epsilon\rightarrow\nu_j}(x)=0$ for $\mu^\epsilon$-almost every $x$.
  \item Let $\mu \in \mathcal{G}(\mathbb{R}^d)$. Then $S^\epsilon_{\lambda,\mathcal{V}}$ admits a derivative at $\mu$, given by $\delta S_{\lambda,\mathcal{V}}^\epsilon(\mu)=\sum_{j=1}^m\lambda_j f^\epsilon_{\mu\rightarrow \nu_j}-f^\epsilon_{\mu\rightarrow \mu}$. Hence $\mu^\epsilon\in \mathcal{G}(\mathbb{R}^d)$ is a critical point for $S^\epsilon_{\lambda,\mathcal{V}}$ on $\mathcal{G}(\mathbb{R}^d)$ if and only if $T^\epsilon_{\mu^\epsilon\rightarrow \mu^\epsilon}(x)-\sum_{j=1}^m\lambda_j T^\epsilon_{\mu^\epsilon\rightarrow \nu_j}(x)=0$ for $\mu^\epsilon$-almost every $x$.
\end{itemize} 
\end{corollary}

We highlight that Theorem \ref{thm:entot_differentiable} allows us to establish differentiability of $F_{\lambda,\mathcal{V}}^\epsilon$ on the entirety of $\mathcal{P}_2(\mathbb{R}^d)$, but does not allow us to conclude the same for $S^\epsilon_{\lambda,\mathcal{V}},$ which we leave as an open problem.

The critical point conditions in Corollary \ref{cor:first_var for_functionals} are crucial for our method of solving the analysis problem, as explained in Section \ref{sec:analysis_problem}. We give the following sufficient condition for a critical point to be optimal.

\begin{corollary}\label{cor:optimality}
Let $\epsilon>0$, and let $\nu_j\in \mathcal{G}(\mathbb{R}^d)$ $1\leq j\leq n.$ Let $\mu^\epsilon\in\mathcal{P}_2(\mathbb{R}^d)$ (resp. $\mu^\epsilon \in \mathcal{G}(\mathbb{R}^d)$) be a critical point for $F^\epsilon_{\lambda,\mathcal{V}}$ (resp. $S^\epsilon_{\lambda,\mathcal{V}}$). If $\texttt{supp}(\mu^\epsilon)=\mathbb{R}^d$, then $\mu^\epsilon$ minimizes $F^\epsilon_{\lambda,\mathcal{V}}$ (resp. minimizes $S^\epsilon_{\lambda,\mathcal{V}}$ over the set $\mathcal{G}(\mathbb{R}^d)$). 
\end{corollary}

It is unknown what conditions on $\mathcal{V}$ guarantee the existence of a critical point with support equal to $\mathbb{R}^d$. On the other hand, there always exists a singular critical point given by a Dirac mass.

\begin{lemma}\label{lemma:diracfixed}
Let $\nu_j\in \mathcal{G}(\mathbb{R}^d)$ for all $1\leq j\leq m$. Let $x_\mathcal{V}:=\sum_{j=1}^m\lambda_j\mathbb{E}(\nu_j)$. Then the Dirac mass $\delta_{x_\mathcal{V}}$ is a critical point for $F^\epsilon_{\lambda,\mathcal{V}}.$
\end{lemma}

Finally, we establish a control on the growth of critical points of $F^\epsilon_{\lambda,\mathcal{V}}$, which will be useful for establishing our sample complexity results in the sequel:
\begin{proposition}\label{prop:subgauss}
Suppose $\mathcal{V}\subset \mathcal{G}_\sigma(\mathbb{R}^d),$ and let $\mu^\epsilon$ be a critical point for $F^\epsilon_{\lambda,\mathcal{V}}$. Then $\mu^\epsilon$ is $\sigma$-subgaussian. In particular, $\mu^*\in\argmin_{\mu\in \mathcal{P}_2(\mathbb{R}^d)}F_{\lambda,\mathcal{V}}(\mu)$ is $\sigma$-subgaussian. 
\end{proposition}

\section{SAMPLE COMPLEXITY OF ESTIMATING \textbf{$\min_\mu\mathcal{F}^\epsilon_{\lambda,\mathcal{V}}(\mu)$}}
\label{sec:sample_complexity_functional}

In this section, we show that the rate of estimating $\min_{\mu\in \mathcal{P}_2(\mathbb{R}^d)}\mathcal{F}^\epsilon_{\lambda,\mathcal{V}}(\mu)$ as a function of the number of samples from the reference measures does not depend on the dimension. 
\begin{theorem}\label{thm:synthsampcomplex}
Let $\mathcal{V}=\{\nu_1,...,\nu_m\}\subset\mathcal{G}_{\sigma}(\mathbb{R}^d)$ and let $\hat{\mathcal{V}}^n=\{\hat{\nu}^n_1,...,\hat{\nu}^n_m\}.$ Let $\mu^{n}$ denote a (random) minimizer of $\min_{\mu\in \mathcal{P}^n(\mathbb{R}^d)}F^\epsilon_{\lambda,\hat{\mathcal{V}}^n}(\mu)$. Then: \begin{equation}\mathbb{E}\left(\bigg|\min_{\mu\in \mathcal{P}_2(\mathbb{R}^d)}F^\epsilon_{\lambda,\mathcal{V}}(\mu)-F^\epsilon_{\lambda,\mathcal{V}}(\mu^n)\bigg|\right)\leq mC^*_{d,\epsilon,\sigma}/\sqrt{n}.\label{eqn:thm3.1F}\end{equation}
Suppose that, for any $\mathcal{V}\subset \mathcal{G}(\mathbb{R}^d)$,  $U_n:=\argmin_{\mu\in \mathcal{P}^n(\mathbb{R}^d)}S^\epsilon_{\lambda,\mathcal{V}}(\mu)$ is nonempty and furthermore $U_n\subset \mathcal{G}_{\tilde{\sigma}}(\mathbb{R}^d)$, where $\tilde{\sigma}=\max_{1\leq j \leq m}\{\|\nu_j\|_\mathcal{G}\}$. Let $\mu^n \in \argmin_{\mu\in \mathcal{P}^n(\mathbb{R}^d)}S^\epsilon_{\lambda,\hat{\mathcal{V}^n}}(\mu)$. Then:
\begin{equation} \mathbb{E}\left(\bigg|\min_{\mu\in \mathcal{G}_\sigma(\mathbb{R}^d)}S^\epsilon_{\lambda,\mathcal{V}}(\mu)-S^\epsilon_{\lambda,\mathcal{V}}(\mu^n)\bigg|\right)\leq mC^*_{d,\epsilon,\sigma}/\sqrt{n}.\label{eqn:thm3.1S} \end{equation}
Here, $C^*_{d,\epsilon,\sigma}$ is a constant depending only on $d,\epsilon$ and $\sigma.$ 
\end{theorem}

Note that in contrast to $F_{\lambda, \mathcal{V}}^{\epsilon}$, for $S^\epsilon_{\lambda,\mathcal{V}}$ in (\ref{eqn:thm3.1S}) we restrict the minimization to the set $\mathcal{G}_\sigma(\mathbb{R}^d)$. This choice is justified by Proposition \ref{prop:existence} where we show existence of minimizers for $S^\epsilon_{\lambda,\mathcal{V}}$ when restricted to $\mathcal{G}_\sigma(\mathbb{R}^d)$. Our statistical results for synthesis in \ref{thm:synthsampcomplex} are \emph{algorithm independent} and apply when the support of the measures can be unbounded.  This improves upon [Theorem 7, \citep{luise2019sinkhorn}], which is restricted to the bounded support setting. However, their result comes with an algorithm that provably achieves the bound. It is natural to therefore ask whether there exist algorithms that can achieve the bounds in Theorem \ref{thm:synthsampcomplex}. We leave this to future work.

A natural question is whether a stronger sample complexity result is true. Let $\mu^*$ denote a minimizer of $\mathcal{F}^\epsilon_{\lambda,\mathcal{V}}$ on $\mathcal{P}_{2}(\Omega)$ with $\Omega$ bounded. Is it true that some $\mu^*_n\in\argmin_{\mu\in \mathcal{P}^n(\Omega)}\mathcal{F}^\epsilon_{\lambda,\hat{\mathcal{V}}^n}$ approximates $\mu^*$ in $OT_2$ on $\mathcal{P}_2(\Omega)$ with parametric sample complexity, i.e., $\mathbb{E}\left(OT_2(\mu^*,\mu^*_n)\right)\lesssim \frac{1}{\sqrt{n}}$? We claim this cannot be true for the functional $S^\epsilon_{\lambda,\mathcal{V}}$. Indeed, let $\mathcal{V}=\{\nu\}$ for any $\nu\in \mathcal{P}_2(\Omega)$ for some bounded $\Omega$. By Proposition \ref{prop:existence}, $\argmin_{\mu\in \mathcal{P}_2(\Omega)}S^\epsilon_{\lambda,\mathcal{V}}(\mu)=\nu$, as $S^\epsilon_{\lambda,\mathcal{V}}$ is debiased and strictly convex on $\mathcal{P}_2(\Omega)$. 
 Similarly, if $\mathcal{V}=\{\nu^{n}\}$ for any $\nu^n\in \mathcal{P}^n(\Omega)$, then $\argmin_{\mu\in\mathcal{P}^n(\Omega)}S^\epsilon_{\lambda,\mathcal{V}}(\mu)=\nu^n$, again by debiasing and strict convexity. Thus, \begin{align*} &\mathbb{E}(OT_2(\argmin_{\mu\in \mathcal{P}_2(\Omega)}S^\epsilon_{\lambda,\{\nu\}}(\mu),\argmin_{\mu\in \mathcal{P}^n(\Omega)}S^\epsilon_{\lambda,\{\hat{\nu}^n\}}(\mu))) \\  = &\mathbb{E}(OT_2(\nu,\hat{\nu}^n))\gtrsim \frac{1}{n^{1/d}},\end{align*}
where the inequality is classical \citep{dudley1969speed}. 

\section{SAMPLE COMPLEXITY AND STABILITY FOR THE ANALYSIS PROBLEM}\label{sec:analysis_problem}

As explained in Section \ref{subsec:background_notation}, if $\mu^*\in \argmin_{\mu\in \mathcal{P}_{2}(\mathbb{R}^d)}\mathcal{F}^\epsilon_{\lambda_*,\mathcal{V}}(\mu)$, solving the analysis problem (\ref{eqn:analysiseqn}) is equivalent to computing $\lambda_*\in \Delta^m$. Using the criteria from Corollary \ref{cor:first_var for_functionals}, we adapt the proof of Proposition 1 in \citep{werenski2022measure} to show that solving the analysis program is equivalent to solving a convex, quadratic program.

\begin{proposition}\label{prop:analysis_grad}
Let $\mathcal{V}\subset\mathcal{G}(\mathbb{R}^d)$.\vspace{-5pt}
\begin{itemize}[leftmargin=10pt]
    \item Let $\mu\in \mathcal{P}_2(\mathbb{R}^d)$, and define $A^\epsilon_\mu\in \mathbb{R}^{m\times m}$ by $[A^\epsilon_\mu]_{ij}:=\int_{\mathbb{R}^d} \langle T^\epsilon_{\mu\rightarrow\nu_i}(x)-x,\;T^\epsilon_{\mu\rightarrow \nu_j}(x)-x\rangle d\mu(x)$. Then $\|\nabla \delta F^\epsilon_{\lambda,\mathcal{V}}\|^2_{L^2(\mu)}=\lambda^\top A^\epsilon_\mu \lambda$, and hence $\mu$ is a critical point iff $\min_{\lambda\in \Delta^m} \lambda^\top A^\epsilon_\mu \lambda=0$.
    \vspace{-5pt}
    \item  Let $\tilde{\mu}\in\mathcal{G}(\mathbb{R}^d)$, and define $S^\epsilon_{\tilde{\mu}}\in \mathbb{R}^{m\times m}$ by $[S^\epsilon_{\tilde{\mu}}]_{ij}:=\int_{\mathbb{R}^d} \langle T^\epsilon_{\tilde{\mu}\rightarrow\nu_i}(x)-T^\epsilon_{\tilde{\mu}\rightarrow \tilde{\mu}}(x),\;T^\epsilon_{\tilde{\mu}\rightarrow \nu_j}(x)-T^\epsilon_{\tilde{\mu}\rightarrow \tilde{\mu}}(x) \rangle d\tilde{\mu}(x)$. Then $\|\nabla\delta S^\epsilon_{\lambda,\mathcal{V}}\|^2_{L^2(\tilde{\mu})}=\lambda^\top S^\epsilon_{\tilde{\mu}} \lambda$,  and hence $\tilde{\mu}$ is a critical point of $S^\epsilon_{\lambda,\mathcal{V}}$ iff $\min_{\lambda\in \Delta^m} \lambda^\top S^\epsilon_{\tilde{\mu}} \lambda=0$.
\end{itemize}
Furthermore, if $\texttt{supp}(\mu)=\mathbb{R}^d$ (resp. $\texttt{supp}(\tilde{\mu})=\mathbb{R}^d$), then $\mu$ is an entropic barycenter (resp. $\tilde{\mu}$ is a Sinkhorn barycenter).
\end{proposition}

Our main result in this section is that the sample complexity of solving the analysis problem for a critical point $\mu$ is equivalent to the sample complexity of estimating $\nabla \delta\mathcal{F}^\epsilon_{\lambda,\mathcal{V}}$.
\begin{algorithm}[h!]
    \caption{Coefficient Recovery}\label{alg:coefficient_recovery}
    \KwData{Input measure $\mu\in \mathcal{P}(\Omega)$; \\
            Reference measures $\nu_1,...,\nu_m\in \mathcal{P}(\Omega)$; \\
            Regularization parameter $\epsilon>0$;\\ 
            Number of samples $n$; \\
            Functional $\mathcal{F}^\epsilon_{\lambda,\mathcal{V}}\in \{F^\epsilon_{\lambda,\mathcal{V}},S^\epsilon_{\lambda,\mathcal{V}}\}$;\\
            Estimators $\hat{T}(\hat{\mu}^n,\hat{\nu}_j^n)$ for $T_{\mu\rightarrow\nu_{j}}^{\epsilon}$, $\forall \;1\leq j \leq m$; \\
            Estimator $ \hat{T}(\hat{\mu}^n)$ for $T_{\mu\rightarrow\mu}^{\epsilon}$.
    }
    \vspace{2mm} 
    \hrule 
    \vspace{2mm} 

    Sample $X_1,X_2,...,X_{2n}\sim \mu$\;
    
    \For{$1\leq j \leq m$}{
        Sample $Y_1^j,Y_2^j,...,Y_n^j\sim \nu_j$\;
        $\hat{\nu}^n_j\leftarrow$ uniform measure on $Y_1^j,Y_2^j,...,Y_n^j$\;
    }
                $\hat{\mu}^{n}\leftarrow$ uniform measure on $X_1,X_2,...,X_n$\;

    \For{$1\leq i,j\leq m$}{
        \If{$\mathcal{F}^\epsilon_{\lambda,\mathcal{V}}==F^\epsilon_{\lambda,\mathcal{V}}$}{

            $[\hat{M}_\mu]_{ij}\leftarrow \frac{1}{n}\sum_{k=n+1}^{2n}\langle \hat{T}(\hat{\mu}^n,\hat{\nu}_i^n)(X_k)-X_k, \hat{T}(\hat{\mu}^n, \hat{\nu}_j^n)(X_k)-X_k\rangle$\;
        }
        \Else{

            $[\hat{M}_\mu]_{ij}\leftarrow\frac{1}{n}\sum_{k=n+1}^{2n}\langle \hat{T}(\hat{\mu}^n,\hat{\nu}_i^n)(X_k)-\hat{T}(\hat{\mu}^n)(X_k), \hat{T}(\hat{\mu}^n, \hat{\nu}_j^n)(X_k)-\hat{T}(\hat{\mu}^{n})(X_k)\rangle$\;
        }
    }
    \vspace{1mm} 
    \hrule 
    \vspace{1mm}
    \textbf{Output:} $\hat{\lambda}^n=\argmin_{\lambda\in \Delta^m}\lambda^\top\hat{M}^\epsilon_\mu\lambda$\;
\end{algorithm}

\begin{theorem}\label{thm:coeff_theorem}
Let $\Omega\subseteq \mathbb{R}^d$ and let $\mathcal{V}\subset \mathcal{G}_\sigma(\Omega).$  Let $\mu\in \mathcal{P}_2(\Omega)$ and suppose that $A_\mu^\epsilon$ has an eigenvalue of 0 with unique eigenvector $\lambda_*\in \Delta^{m}$. Suppose there exists an estimator $\hat{T}(\hat{\mu}^n,\hat{\nu}^n_j)$ and a nonincreasing function $\theta(n):\mathbb{N}\rightarrow \mathbb{R}_{\geq 0}$ such that
\begin{equation}\mathbb{E}\left(\|T^\epsilon_{\mu\rightarrow \nu_j}-\hat{T}(\hat{\mu}^n,\hat{\nu}^n_j)\|_{L^2(\mu)}^2\right)\leq \theta(n),\label{eqn:thetabound1}\end{equation}for all $1\leq j\leq m.$ Let $\hat{\lambda}^n$ be the output of Algorithm \ref{alg:coefficient_recovery}, applied to $F^\epsilon_{\lambda,\mathcal{V}}$ at $\mu$. Then:
\begin{equation}
\mathbb{E}\left(\|\hat{\lambda}^n-\lambda_*\|^2_2\right)\leq \frac{Jm^3 d\sigma^2}{\alpha_2}\max\left\{\frac{1}{\sqrt{n}},\sqrt{\theta(n)+\theta(n)^2}\right\},\label{eqn:ent_coeff_recovery}
\end{equation}
where $J$ is an absolute constant, and $\alpha_2$ is the smallest nonzero eigenvalue of $A_\mu^\epsilon$. 

Let $\mu \in \mathcal{G}_\sigma(\Omega)$, and suppose that $S^\epsilon_\mu$ has an eigenvalue of $0$ with multiplicity $1$ with eigenvector $\lambda_*\in \Delta^m$. Suppose that (\ref{eqn:thetabound1}) holds for all $1\leq j\leq m$, and furthermore that there exists an estimator $\hat{T}(\hat{\mu}^n)$ such that:
\begin{equation}
\mathbb{E}\left(\|T^\epsilon_{\mu\rightarrow \mu}-\hat{T}(\hat{\mu}^n)\|_{L^2(\mu)}^2\right)\leq \theta(n).\label{eqn:thetabound2}
\end{equation}
Let $\hat{\lambda}^n$ be the output of Algorithm \ref{alg:coefficient_recovery}, applied to $S^\epsilon_{\lambda,\mathcal{V}}$ at $\mu$. Then:
\begin{equation}
\mathbb{E}\left(\|\hat{\lambda}^n-\lambda_*\|^2_2\right)\leq \frac{\tilde{J}m^3d\sigma^2}{\alpha_2}\max\left\{\frac{1}{\sqrt{n}},\sqrt{\theta(n)+\theta(n)^2}\right\},\label{eqn:sink_coeff_recovery}
\end{equation}
where $\tilde{J}$ is an absolute constant, and $\alpha_2$ is the smallest nonzero eigenvalue of $S^\epsilon_\mu$. 
\end{theorem}

We present two applications of Theorem \ref{thm:coeff_theorem} based on \citep{rigollet2022sample} and \citep{werenski2023estimation}. In both cases we may conclude that the rate of estimation of regularized barycentric coefficients is not cursed by dimensionality.

\begin{corollary}\label{cor:bounded analysis}
Let $\Omega\subset\mathbb{R}^d$ be bounded, and let $\mathcal{V}\subset \mathcal{P}_2(\Omega)$. Let $X_1,X_2,...,X_{2n}\sim \mu$, and let $\hat{\mu}^n$ denote the uniform distribution on $X_1,...,X_n$, $\hat{\mu}_1$ denote the uniform distribution on $X_1,...,X_{\lfloor n/2\rfloor}$ and $\hat{\mu}_2$ denote the uniform distribution on $X_{\lfloor n/2\rfloor+1},...,X_{2\lfloor n/2\rfloor}$. Then Theorem \ref{thm:coeff_theorem} holds, with $\hat{T}(\hat{\mu}^n,\hat{\nu}_j^n)=T^\epsilon_{\hat{\mu}^n\rightarrow \hat{\nu}_j^n}$, $\hat{T}(\hat{\mu}^n)=\hat{T}^\epsilon_{\hat{\mu}_1\rightarrow \hat{\mu}_2}$ and $\theta(n)=\frac{\tilde{C}_{\Omega,\epsilon}}{n}$, where $\tilde{C}_{\Omega,\epsilon}$ is a constant depending only on $|\Omega|$ and $\epsilon$. 
\end{corollary}

\begin{corollary}\label{cor:logconcave_analysis}
Let $\nu_j$ be $c$-strongly log-concave with $\mathbb{E}(\nu_j)=0$ for all $1\leq j\leq m$. Then there exists an estimator $\hat{T}$ such that Theorem \ref{thm:coeff_theorem} holds for $F^\epsilon_{\lambda,\mathcal{V}}$, with $\theta(n)=\frac{\tilde{K}_{d,\epsilon,c}}{n^{1/3}}$, where $\tilde{K}_{d,\epsilon,c}$ is a constant depending on $d,\epsilon$ and $c$.
\end{corollary}

Finally, we prove a stability property, which implies that the encoding obtained via solving the analysis problem is stable with respect to perturbations of the measure being analyzed, suggesting an advantage of using regularized barycenter functionals for representation of measures in $\mathcal{P}_2(\Omega)$.
\begin{proposition}\label{prop:stab_analysis}
Let $\Omega\subset \mathbb{R}^d$ be bounded, and let $\mu,\rho\in \mathcal{P}(\Omega)$ and $\mathcal{V}\subset \mathcal{P}(\Omega)$. Suppose that $A_\mu^\epsilon$ (resp. $S^\epsilon_\mu$) has an eigenvalue 0 with unique eigenvector $\lambda_\mu\in\Delta^m$. Let $\lambda_\rho\in \argmin_{\lambda\in \Delta^{m}}\lambda^\top A_\rho^\epsilon \lambda$ (resp. $\lambda_\rho\in \argmin_{\lambda\in \Delta^{m}}\lambda^\top S_\rho^\epsilon \lambda$). Then there exists a constant $H_{\Omega,\epsilon},$ depending only on $\Omega$ and $\epsilon$, such that $\|\lambda_\mu-\lambda_\rho\|_2^2\leq \frac{m^3 H_{\Omega,\epsilon}}{\alpha_2}OT_2(\mu,\rho)$, where $\alpha_2>0$ is the smallest nonzero eigenvalue of $A_\mu^\epsilon$ (resp. $S^\epsilon_\mu$).  
\end{proposition}

\begin{figure}
\centering
    \begin{subfigure}[b]{0.23\textwidth}
        \centering
        \includegraphics[width=\textwidth]{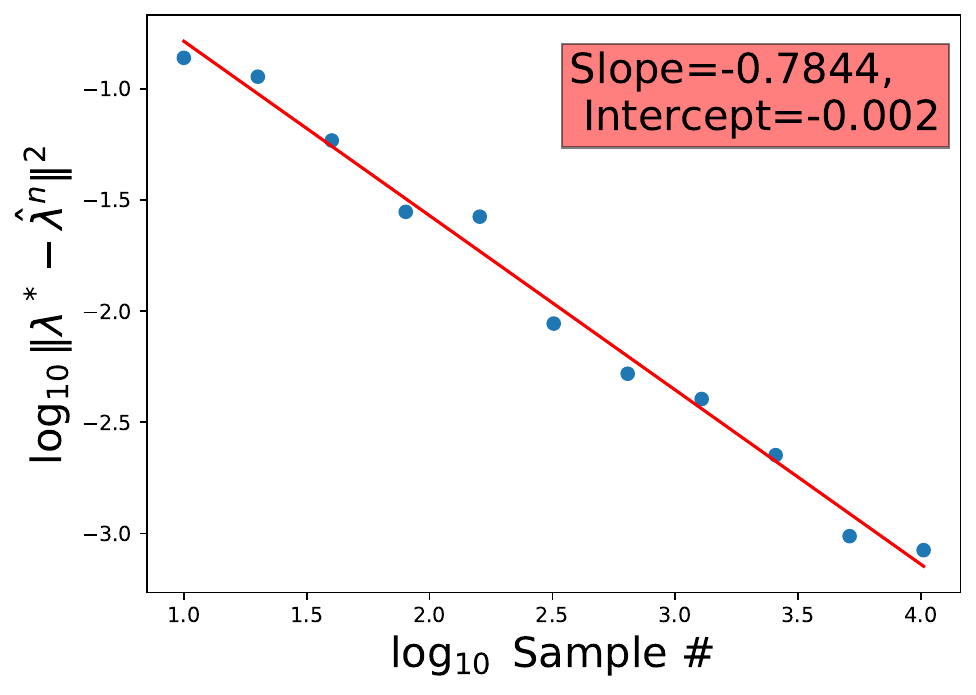}
    \end{subfigure}
    \begin{subfigure}[b]{0.23\textwidth}
        \centering
         \raisebox{0.07cm}{\includegraphics[width=\textwidth]{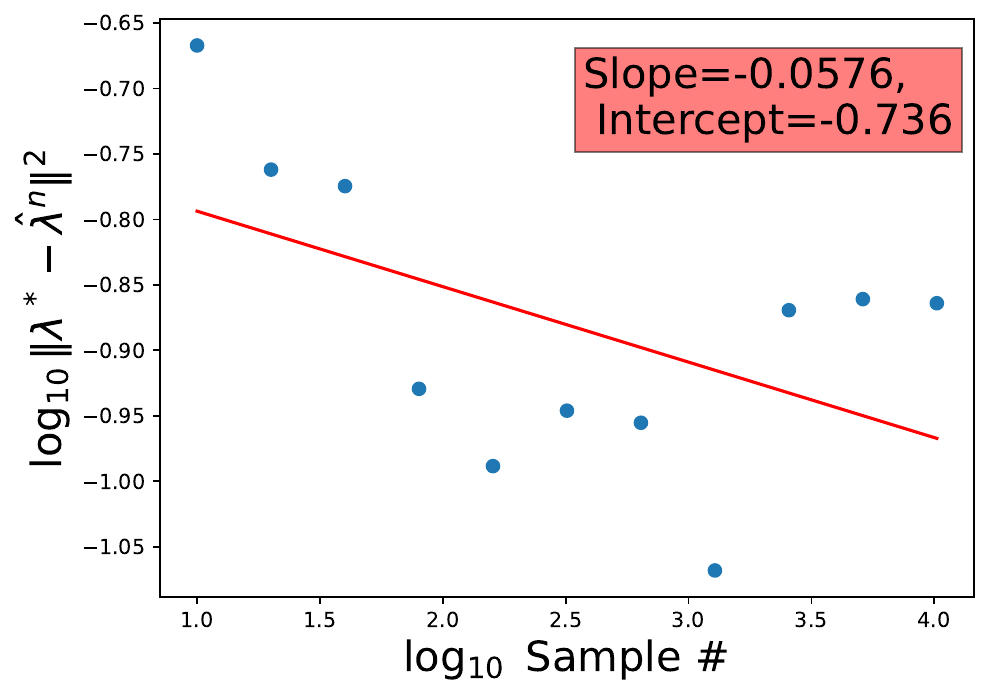}}
    \end{subfigure}
    \begin{subfigure}[b]{0.23\textwidth}
        \centering
        \includegraphics[width=\textwidth]{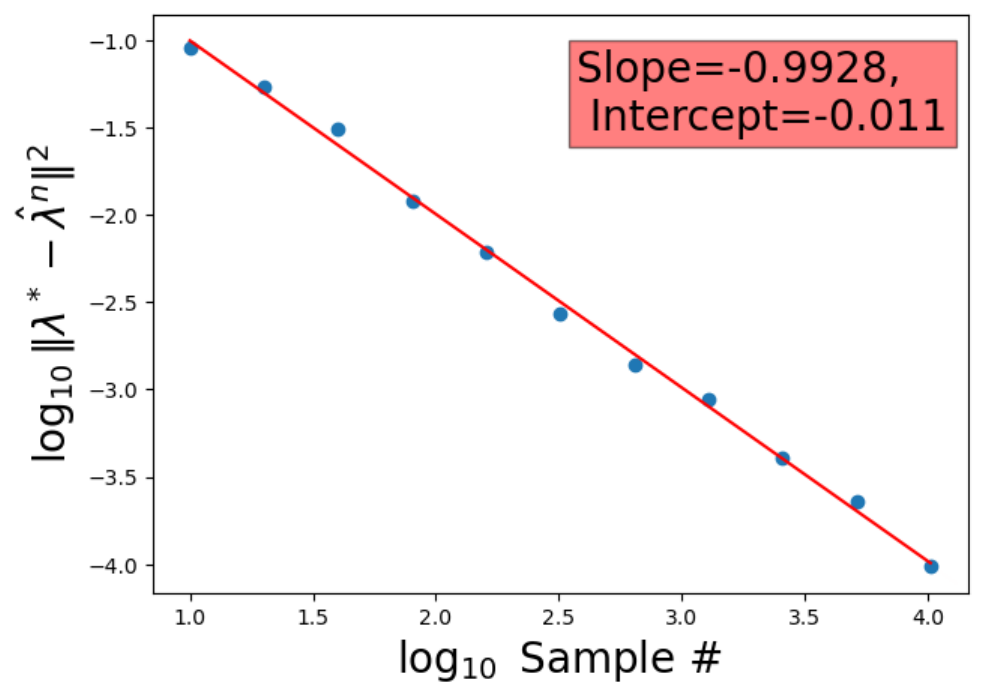}
    \end{subfigure}
    \begin{subfigure}[b]{0.23\textwidth}
        \centering
        \begin{minipage}[b]{\textwidth}
            \hspace{.025cm} 
            \includegraphics[width=\textwidth]{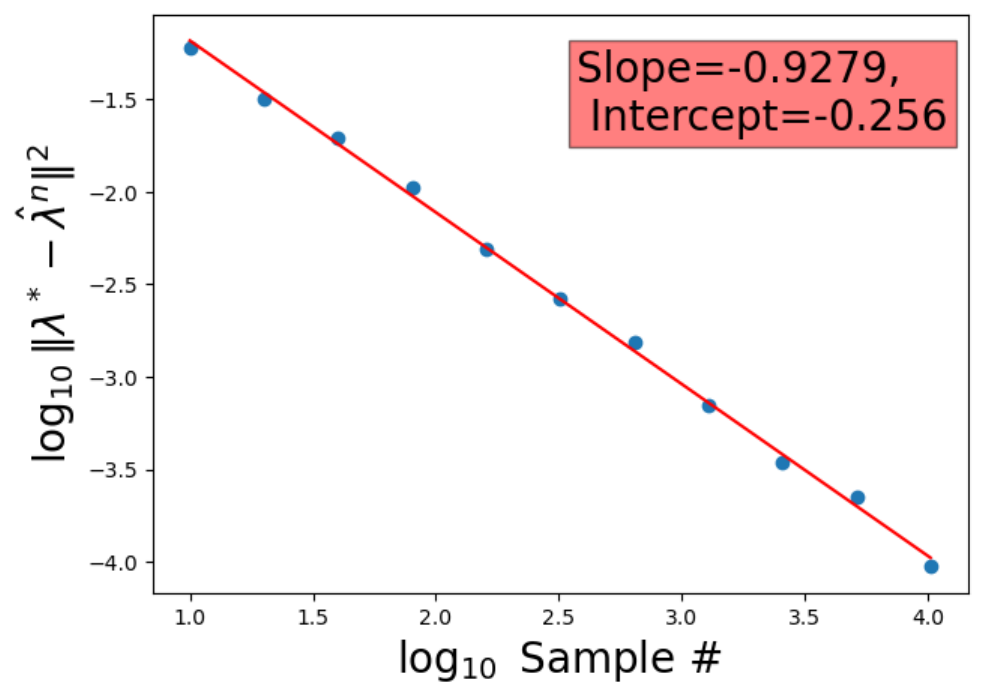}
        \end{minipage}
    \end{subfigure}
    \caption{(Top left) Average $\log$ $\ell^2$-loss for  two random 1D Gaussian measures with random weights, $\epsilon=2$. (Top right) Average  $\log$ $\ell^2$-loss for three random 1D Gaussian measures with random weights, $\epsilon=2$. (Bottom left) Average $\log$ $\ell^2$-loss for three random 5D Gaussian measures, $\lambda=(0.2201,0.0269,0.7530)$, $\epsilon=1$. (Bottom right) Average $\log$ $\ell^2$-loss for three uniform-measures on random 5D cubes, $\lambda=(0.5112,0.4477,0.0411)$, $\epsilon=0.1$. }
    \label{fig:synthetic}
\end{figure}

\subsection{Numerical Verification of Theorem \ref{thm:coeff_theorem}}
\label{sec:numericalexp}

We consider several numerical experiments validating Theorem \ref{thm:coeff_theorem}. In these experiments we focus only on $F^\epsilon_{\lambda,\mathcal{V}}$ and remark that similar results are achievable for $S^\epsilon_{\lambda,\mathcal{V}}$. Our choices of $\mathcal{V}$ are 1-dimensional Gaussians, 5-dimensional Gaussians, and uniform measures on 5-dimensional cubes. In each experiment, we generate a random set of $m$ reference measures and weights $\lambda \in\Delta^m$. We sample from the $(\lambda,\mathcal{V})$-barycenter and the reference measures $\mathcal{V}$, and using these samples we compute $\hat{\lambda}^n$ by applying Algorithm \ref{alg:coefficient_recovery}. Our choice of estimator for the map $T_{\mu\rightarrow \nu_j}^\epsilon$ is $\hat{T}(\hat{\mu}^{n},\hat{\nu}_{j}^{n})=T^\epsilon_{\hat{\mu}^n\rightarrow \hat{\nu}_j^n}$, for each $1\leq j \leq m$. For 1-dimensional Gaussian reference measures, we may sample from $\argmin_{\mu\in \mathcal{P}_2(\mathbb{R})}F^\epsilon_{\lambda,\mathcal{V}}(\mu)$ directly by applying the closed-form expression for entropic barycenters with Gaussian reference measures, given in Theorem 2 in \citep{janati2020debiased}. As we do not have access to closed-form expressions for the entropic barycenters in the higher-dimensional examples, we instead apply a free-support synthesis algorithm (see details in Appendix \ref{sec:synthesis_algo}) to compute approximate critical points $\mu^*$, which we then sample from to estimate $\hat{\lambda}^n$. We then compute  $\|\hat{\lambda}^n-\lambda\|^2_2$. We repeat this over $100$ trials and report the logarithm of the mean-squared error against $\log(n)$ as $n$ ranges over $\{10,20,40,...,10240\}$. The results are plotted in Figure \ref{fig:synthetic}. See Appendix \ref{sec:samplecomplex_implementation} for implementation details. \\

 \noindent\textbf{1-Dimensional Gaussians:}  We see for $m=2$, the errors decays at a rate close to $n^{-1}$, which is faster than predicted by Theorem \ref{thm:coeff_theorem}. In the case of $m=3$, we find that the $\ell^2$-distance between the coefficients essentially does not decay in $n$ (top right). This is due to the fact that in one dimension, it is more likely that the eigenspace of $A^\epsilon_{\mu}$ associated to the eigenvalue zero has dimension greater than 1, which violates the conditions in Theorem \ref{thm:coeff_theorem}, leading to non-uniqueness of coefficients corresponding to the measure being analyzed. Nevertheless, we empirically show (see Appendix \ref{sec:nonunique_coeffs}) that the expected $OT_2^2$-error between the barycenters associated to $\lambda$ and $\hat{\lambda}^n$ rapidly decays as $n$ increases, implying successful recovery of a \emph{valid} set of barycentric coefficients via Algorithm \ref{alg:coefficient_recovery}.\\

  \noindent\textbf{5-Dimensional Gaussians and Uniform Measures:}  For the 5-dimensional Gaussian reference measures, we see that the decay in $\ell^2$ error between coefficients $\lambda$ and $\hat{\lambda}^n$ is at a rate of approximately $n^{-1}$. We conjecture that the improved rate could be due to {the fact that the entropic maps between Gaussians are affine  (see Theorem 1 in \citep{janati2020entropic}).} We further tested our method with reference measures given by three random translates of $\mbox{Unif}([0,2]^5)$. We see that the error decays at a rate of approximately $n^{-9/10}$, again exceeding the rate predicted by our theory, though marginally slower than the Gaussian case.

\begin{table*}[ht]
\centering
\hfill
\resizebox{\textwidth}{!}{
\begin{tabular}{c c c c c c c}
\hline
 & dropout\_local\_1 & dropout\_local\_2 & dropout\_global & jitter & add\_local& clean \\
\hline
$F^\epsilon_{\lambda,\mathcal{V}}$, $\epsilon=0.75$ & $0.397 \pm 0.055$ & $0.387 \pm 0.051$ & $0.399 \pm 0.057$ & $0.392 \pm 0.055$ & $0.389 \pm 0.051$ & $0.397 \pm 0.057$ \\
$S^\epsilon_{\lambda,\mathcal{V}}$, $\epsilon=0.75$ & $0.600 \pm 0.020$ & $0.562 \pm 0.0233$ & $0.626 \pm 0.017$ & $0.758 \pm 0.020$ & $0.578 \pm 0.017$ & $0.740 \pm 0.010$ \\
$F^\epsilon_{\lambda,\mathcal{V}}$, $\epsilon=0.089$ & $0.544 \pm 0.037$ & $0.511 \pm 0.036$ & $0.559 \pm 0.037$ & $0.532 \pm 0.038$ & $0.510 \pm 0.037$ & $0.558 \pm 0.039$ \\
 $S^\epsilon_{\lambda,\mathcal{V}}$, $\epsilon=0.089$ & $0.862 \pm 0.019$ & $0.81 \pm 0.018$ & $0.886 \pm 0.010$ & $0.904 \pm 0.010$ & $0.846 \pm 0.010$ & $0.886 \pm 0.022$ \\
$F^\epsilon_{\lambda,\mathcal{V}}$,$\epsilon=0.009$ & $0.908 \pm 0.020$ & $0.902 \pm 0.022$ & $0.913 \pm 0.018$ & $0.914 \pm 0.018$ & $0.888 \pm 0.019$ & $0.901 \pm 0.017$ \\
$S^\epsilon_{\lambda,\mathcal{V}}$,$\epsilon=0.009$ & $0.905 \pm 0.022$ & \bm{$0.902 \pm 0.021$} & $0.919 \pm 0.020$ & $0.920 \pm 0.018$ & \bm{$0.900 \pm 0.018$} & \bm{$0.908 \pm 0.024$} \\
Unregularized ($\epsilon=0$) & \bm{$0.912 \pm 0.019$} & $0.901 \pm 0.018$& \bm{$0.924 \pm 0.020$} & \bm{$0.921 \pm 0.020$} & $0.897 \pm 0.017$ & \bm{$0.908 \pm 0.018$} \\
Doubly Regularized $\epsilon=0.009,\tau=0.01$ & $0.874 \pm 0.015$ & $0.856 \pm 0.019$ & $0.891 \pm 0.019$ & $0.896 \pm 0.019$ & $0.865 \pm 0.019$ & $0.884 \pm 0.022$ \\
PointNet& $0.68 \pm 0.035$ & $0.672 \pm 0.036$ & $0.661 \pm 0.037$ & $0.684 \pm 0.031$ & $0.6567 \pm 0.029$ & $0.678 \pm 0.029$ \\
\hline
\end{tabular}
}
\vspace{1mm}
\caption{Comparison of PointNet (trained on 400 point clouds) with classification based on barycentric coefficients (with 15 reference measures) for point cloud classification. Confidence intervals calculated as $t_{\alpha,4}*\frac{\hat{\sigma}}{\sqrt{n}}$ with $\hat{\sigma}$ the sample standard deviation and $t_{\alpha,4}$ is the Student's t-value with $\alpha=0.05$ and $4$ degrees of freedom.} \label{table:accuracy}
\end{table*}

\section{POINT CLOUD CLASSIFICATION}
\label{sec:point-cloud-classification}
\begin{algorithm}[t!]
    \caption{Point Cloud Classification Algorithm}\label{alg:point_cloud_classification}
    \KwData{ $m$ classes;\\
        $b$ labeled point clouds for each class: $\mathcal{A}^y = \{P^y_1, P^y_2, \ldots, P^y_{b}\}$ for $1 \leq y \leq m$;\\
        Unlabelled point cloud: $Q$;\\
        Regularization parameter: $\epsilon$;\\
        Functional: $\mathcal{F}^\epsilon_{\lambda, \mathcal{V}} \in \{F^\epsilon_{\lambda, \mathcal{V}}, S^\epsilon_{\lambda, \mathcal{V}}\}$.
    }
    \vspace{2mm} 
    \hrule 
    \vspace{2mm} 

$\mu\leftarrow $Empirical measure supported on $Q$;

    \For{$1 \leq y \leq m$}{
        \For{$1 \leq j \leq b$}{
            Set $\nu^y_{j} \leftarrow$ Empirical measure supported on $P^y_{j}$\;
        }
    }
   $\mathcal{V} \leftarrow\{\nu^1_{1}, \ldots, \nu^1_{b}, \nu^2_{1}, \ldots, \nu^2_{b}, \ldots, \nu^m_{1}, \ldots, \nu^m_{b} \}$\;

  \For{$1\leq i_1,i_2\leq m$ and $1\leq j_1,j_2 \leq b$}{
        \If{$\mathcal{F}^\epsilon_{\lambda,\mathcal{V}}==F^\epsilon_{\lambda,\mathcal{V}}$}{
            $[M_{\mu}]_{b(i_1-1)+j_1,b(i_2-1)+j_2}\leftarrow \frac{1}{n}\sum_{x \in {\texttt{supp}(\mu)}} \langle T^\epsilon_{\mu\rightarrow\nu^{i_1}_{j_1}}(x) - x, T^\epsilon_{\mu\rightarrow \nu^{i_2}_{j_2}} (x)- x\rangle$ ;
        }
        \Else{
            $[M_{\mu}]_{b(i_1-1)+j_1,b(i_2-1)+j_2}\leftarrow \frac{1}{n}\sum_{x \in {\texttt{supp}(\mu)}} \langle T^\epsilon_{\mu\rightarrow \nu_{j_1}^{i_1}}(x) -T^\epsilon_{\mu\rightarrow \mu}(x) , T^\epsilon_{\mu\rightarrow \nu_{j_2}^{i_2}} (x) -T^\epsilon_{\mu\rightarrow \mu} (x) \rangle $ ;
        }
    }
    $\lambda\leftarrow \argmin_{\lambda\in\Delta^{m b}}\lambda^\top M_{\mu}\lambda;$

Define $B\in \mathbb{R}^{m\times mb}$ such that:
    \[
    [B]_{ij} := \begin{cases} 
      1 & \text{if } j \in \{b(i-1)+1,..., bi\}, \\ 
      0 & \text{otherwise}. 
   \end{cases}
    \]
    
$\tilde{\lambda}\leftarrow B\lambda$
\vspace{1mm} 
    \hrule 
    \vspace{1mm}
  \textbf{Output:} Assign point cloud $Q$ the class $ \tilde{y} \leftarrow \texttt{argmax}_{1\leq y \leq m} [\tilde{\lambda}]_{y}$

\end{algorithm}

We demonstrate the utility of the analysis of measures via regularized barycenters for 3D point cloud-classification. We obtained our data from the PointCloud-C dataset \citep{ren2022benchmarking}, a repository of point clouds separated into classes and corrupted with various regimes of noise and occlusions. We select 100 clean (uncorrupted) point clouds from five classes (airplanes, beds, guitars, monitors, vases) and perform the following experiment. We run Algorithm \ref{alg:Classification Experiment} (detailed in Subsection \ref{sec:pointcloud implementation} in the Appendix) on these point clouds with $m=5$, $q=100$, $b=3$, $n_{train}=80$, $n_{test}=20$, $K=5$, and vary over $\mathcal{F}^\epsilon_{\lambda,\mathcal{V}}\in \{F^\epsilon_{\lambda,\mathcal{V}}, S^\epsilon_{\lambda,\mathcal{V}}$\} and $\epsilon\in \{0.009,0.089,0.75\}$. We also augment our \emph{test} set with 500 corrupted point clouds, which are copies of the 100 clean test point clouds with 5 types of corruptions applied. We obtain the corrupted data from the \texttt{dropout$\_$local$\_$1, dropout$\_$local$\_$2,dropout$\_$global$\_$4,jitter$\_$4} and \texttt{add$\_$local$\_$4} datasets; see \citep{ren2022benchmarking} for details. We do this by applying the same train-test splits to the corrupted point clouds, discarding the corrupted training point clouds, and then using the previously selected clean training point clouds as $\mathcal{A}^y_{train}$ in Algorithm \ref{alg:Classification Experiment}.

In Table \ref{table:accuracy}, we display results for classification using $\mathcal{F}^\epsilon_{\lambda,\mathcal{V}}$ with different choices of $\epsilon$. We compare our results to PointNet \citep{qi2017pointnet}, a convolutional neural network (CNN) used for point cloud classification. Our implementation is adapted from \citep{Karaev2023nikitakaraevv}.  We also compare against the same classification method with the unregularized barycenter functional (corresponding to $\mathcal{F}^\epsilon_{\lambda,\mathcal{V}}$, with $\epsilon=0$) and the doubly-regularized barycenter functional \citep{chizat2023doubly}, with inner regularization $\epsilon=0.009$ and outer-regularization $\tau=0.01$. See Section \ref{sec:pointcloud implementation} for details. We see that classification using barycentric coefficients is significantly more accurate than PointNet. Furthermore, our method required $<4\%$ of the data used to train PointNet to achieve this accuracy, suggesting its applicability to problems where training data is rare or expensive to acquire. We also note that the accuracy remained high even in the presence of corrupted data, indicating the robustness of our method. However, we note that PointNet has on the order of $10^6$ parameters, and hence 400 training examples may be insufficient to truly assess its accuracy on our test set. In general we see that classification using the debiased functional $S_{\lambda,\mathcal{V}}^\epsilon$ outperforms classification using $F_{\lambda,\mathcal{V}}^\epsilon$, with a more noticeable difference in performance at high $\epsilon$. In general, the Sinkhorn functional with $\epsilon=0.009$ and the unregularized functional outperform all other methods. Both $F^\epsilon_{\lambda,\mathcal{V}}$ and $S^\epsilon_{\lambda,\mathcal{V}}$ outperformed the doubly regularized functional in the low $\epsilon$ regime, perhaps due to the additional complications introduced by score estimation in this setting. Indeed, the point clouds under consideration can be viewed as samples from a distribution supported on a two-dimensional surface in $\mathbb{R}^3$, which would not admit a Lebesgue density. 

In Appendix \ref{sec:completion}, we further highlight the applicability of the analysis coefficients with a point cloud completion method, which we test on the PointCloud-C dataset. 

\section{CONCLUSIONS AND OPEN PROBLEMS}
\label{sec:Conclusions}

We presented new functional analytic, statistical, and stability results for the synthesis and analysis of probability measures with entropy-regularized optimal transport barycenter functionals with applications to sample-efficient classification of point cloud data. 
Several open problems for future work are to: (a) characterize the \emph{projection} error in analyzing a measure (\ref{eqn:analysiseqn}) when $\mu$ is not a barycenter; (b) derive a stronger sample complexity result for synthesis of the barycenter for $F_{\lambda, \mathcal{V}}^{\epsilon}$, and (c) improve the statistical estimation rates in Theorem \ref{thm:coeff_theorem} to match the observed rates in the experiments.

\vspace{10pt}

\noindent\textbf{Acknowledgments:}  BM, JM, and SA acknowledge partial support from NSF DMS-2309519.  BM and JM acknowledge partial support from NSF DMS-2318894. BM was also supported in part by CCF-1553075.  JM acknowledges partial support from a Tufts Springboard grant. SA acknowledges partial support from the NSF under Cooperative Agreement PHY-2019786 (The NSF AI Institute for Artificial Intelligence and Fundamental Interactions, \url{http://iaifi.org/}).

\bibliographystyle{plainnat}
\bibliography{mybib}{}

\newpage 
\appendix
\onecolumn
\section{BACKGROUND}

\subsection{Background on Subgaussian Measures}
\begin{definition} \label{def:subgauss}
A $\mathbb{R}^d$-valued random variable $X\sim\mu$ is \emph{subgaussian} if any of the following equivalent properties holds \citep{vershynin2018high}:
\begin{enumerate} 
\item There exists $\sigma\geq 0$ such that $\sup_{v\in S^{d-1}} \mathbb{E}_{X\sim \mu}\left(e^{\frac{|\langle X,v\rangle|^2}{\sigma^2}}\right)\leq 2$, where $S^{d-1}$ is the unit sphere in $\mathbb{R}^d.$
\item There exists $\sigma\geq 0$ and an absolute constant $c_G$ such that for all $t\geq 0$ we have: \begin{equation*}
\sup_{v\in S^{d-1}}\mathbb{P}_{X\sim \mu}(|\langle X,v\rangle|\geq t)\leq 2\exp\left(-\frac{c_Gt^2}{\sigma^2}\right).
\end{equation*}
\item  There exists $\sigma\geq 0$ and an absolute constant $C_G$ such that for all $p\geq 1$, we have: \begin{equation*}
\sup_{v\in S^{d-1}}\mathbb{E}_{X\sim \mu}(|\langle X,v\rangle|^p)^{1/p}\leq C_G\sigma \sqrt{p}.
\end{equation*}

\end{enumerate}
The smallest $\sigma$ such that condition 1. holds is referred to as the \emph{subgaussian norm} of $X$, denoted $\|X\|_{\mathcal{G}}.$ Up to the absolute constants $c_G$ and $C_G$, $\|X\|_{\mathcal{G}}$ is the smallest constant that makes any of the above inequalities valid. 

We say that $\mu$ is \emph{subgaussian} if $X\sim \mu$ is subgaussian, and define the \emph{subgaussian norm} of $\mu$ to be $\|\mu\|_\mathcal{G}:=\|X\|_\mathcal{G}$. The set of probability measures on $\mathbb{R}^d$ with subgaussian norm bounded above $\sigma$ is referred to as the set of \emph{$\sigma$-subgaussian measures}, and denoted $\mathcal{G}_\sigma(\mathbb{R}^d).$
\end{definition}

\begin{lemma}\label{lemma:normsubgauss}
(Lemma 1 from \citep{jin2019short}) Let $\mu\in \mathcal{G}_\sigma(\mathbb{R}^d)$ and let $X\sim \mu$. Then there exists an absolute constant $q_G$ such that $\|X\|$ is $q_G\sqrt{d}\sigma$-subgaussian, i.e., $\mathbb{E}\left(\exp\left( \frac{\|X\|^2}{2 q_G d \sigma^2}\right)\right) \leq 2$.  

\end{lemma}

We remark that $\mathcal{G}_\sigma(\mathbb{R}^d)$ is closed with respect to the weak topology on probability measures:

\begin{definition}
We say that $\{\mu_n\}_{n=1}^\infty\subset \mathcal{P}(\mathbb{R}^d)$ \emph{converges weakly} (or converges in the weak topology) to $\mu\in \mathcal{P}(\mathbb{R}^d)$ if $\int \phi d\mu_n\rightarrow \int \phi d\mu$ as $n\rightarrow \infty$ for all bounded, continuous functions $\phi.$
\end{definition}

\begin{lemma}\label{lemma:weakly closed}
Let $\{\mu_n\}_{n=1}^\infty\subset \mathcal{G}_\sigma(\mathbb{R}^d)$ be a sequence weakly converging to $\mu\in \mathcal{P}(\mathbb{R}^d)$ as $n\rightarrow \infty$. Then:
\begin{enumerate}
    \item $\mu\in \mathcal{G}_\sigma(\mathbb{R}^d)$;
    \item $\int \|x\|^pd\mu_n\rightarrow \int\|x\|^p d\mu$ as $n\rightarrow \infty$ for all $p\geq 1.$
\end{enumerate}
\end{lemma}
\begin{proof}
Fix $v\in S^{d-1}$ and let $Z_n^v:=|\langle v,X_n\rangle|$ with $X_n\sim \mu_n$ and $Z^v:=|\langle v, X\rangle|$ with $X\sim \mu$. Since $\mu_n$ converges to $\mu$ weakly, $X_n$ converges to $X$ in distribution by the Portmanteau Lemma [Theorem 2.1 in \citep{billingsley2013convergence}], and by the continuous mapping theorem [Theorem 2.7 in \citep{billingsley2013convergence}], $Z_n^v$ converges to $Z^v$ in distribution. By the layer-cake decomposition, we may write for any $p\ge 1$:
\begin{align*}
\mathbb{E}\left(|Z_n^v|^p\right)& =\int_0^\infty \mathbb{P}(|Z_n^v|^p\geq t) dt\\& =\int_0^\infty pt^{p-1}\mathbb{P}(|Z_n^v|\geq t) dt.
\end{align*}
Hence we have:
\begin{align}
\lim_{n\rightarrow\infty}\mathbb{E}\left(|Z_n^v|^p\right) & = \lim_{n\rightarrow\infty}\int_0^\infty pt^{p-1}\mathbb{P}(|Z_n^v|\geq t) dt. \label{eqn:limexc}
\end{align}
To interchange the limit and integral in 
 (\ref{eqn:limexc}), we apply the concentration bound on $\mathbb{P}(|Z_n^v|\ge t)$ (i.e., characterization 2. in Definition \ref{def:subgauss}) to get that for all $t\in [0,\infty)$ and all $n$,
\begin{align*} pt^{p-1}\mathbb{P}(|Z_{n}^{v}|\ge t)\le2pt^{p-1}\exp\left(-\frac{c_{G}t^{2}}{\sigma^{2}}\right),\end{align*}which is integrable for $p\ge 1$. 
 Therefore
\begin{align}
(\mathbb{E}(|Z^{v}|^{p}))^{\frac{1}{p}} =&\left(\int_0^\infty  p t^{p-1}\mathbb{P}(|Z^v|\ge t)dt\right)^{\frac{1}{p}}\notag\\
=&\left(\int_0^\infty \lim_{n\rightarrow \infty} p t^{p-1}\mathbb{P}(|Z_n^v|\ge t)dt\right)^{\frac{1}{p}}\label{eqn:ConvDist}\\
=&\left(\lim_{n\rightarrow \infty} \int_0^\infty p t^{p-1}\mathbb{P}(|Z_n^v|\ge t)dt\right)^{\frac{1}{p}}\label{eqn:LDCT}\\
=&\left(\lim_{n\rightarrow \infty}\mathbb{E}(|Z_n^v|^p)\right)^{\frac{1}{p}}\notag\\
\le &C_{G}\sigma \sqrt{p}\notag,
\end{align}
where (\ref{eqn:ConvDist}) follows by convergence in distribution of $|Z_{n}^{v}|$ to $|Z^{v}|$ and (\ref{eqn:LDCT}) follows by the Dominated Convergence Theorem.  Since this holds for all $v\in S^{d-1}$, we conclude $\mu$ is $\sigma$-subgaussian. 

The proof of the second statement is similar.  Note that since $\mu$ and $\{\mu_n\}_{n=1}^{\infty}$ are $\sigma$-subgaussian, Lemma \ref{lemma:normsubgauss} establishes that both $\|X\|$ and $\{\|X_n\|\}_{n=1}^{\infty}$ are $q_G\sqrt{d}\sigma$-subgaussian. Hence we may apply characterization 2 in Definition \ref{def:subgauss} to conclude for all $t\in [0,\infty)$, and $n$:
\begin{equation*}
pt^{p}\mathbb{P}(\|X_n\|\geq t)\leq 2pt^{p-1}\exp\left(-\frac{c_Gt^2}{q_G^2 d\sigma^2}\right)
\end{equation*}
which is integrable for all $p\geq 1$. Hence we may once again apply the layer-cake decomposition, the dominated convergence theorem, the convergence in distribution of $X_n $ to $X$ (and therefore by the continuous mapping theorem of $\|X_{n}\|$ to $\|X\|$) to conclude that:
\begin{align*}
\lim_{n\rightarrow \infty}\mathbb{E}(\|X_n\|^p)& =\lim_{n\rightarrow \infty}\int_0^\infty pt^{p-1}\mathbb{P}(\|X_n\|\geq t)dt\\&  =\int_0^\infty pt^{p-1}\lim_{n\rightarrow \infty}\mathbb{P}(\|X_n\|\geq t)dt \\&  =\int_0^\infty pt^{p-1}\mathbb{P}(\|X\|\geq t)dt = \mathbb{E}(\|X\|^p).
\end{align*}
\end{proof}

The following variance bound holds uniformly for $\sigma$-subgaussian measures:

\begin{lemma}\label{lemma:subgauss_var_bound}
Let $\mu\in \mathcal{G}_\sigma(\mathbb{R}^d)$. Then $\texttt{Var}(\mu)\leq \tilde{q}_Gd\sigma^{2}$, where $\tilde{q}_{G}$ is an absolute constant.
\end{lemma}
\begin{proof}
For $X\sim \mu$, we have $\texttt{Var}(\mu)\le M_2(\mu)=\mathbb{E}_{X\sim\mu}\|X\|^2$. By Lemma \ref{lemma:normsubgauss}, $\|X\|$ is $q_G\sqrt{d}\sigma$-subgaussian. Applying characterization 3. in Definition \ref{def:subgauss} gives $\mathbb{E}_{X\sim\mu}\|X\|^2\leq 2q_G^2 d\sigma^2$ as desired.
\end{proof}

Finally, we record a bound showing that centering a random variable can only increase its subgaussian norm by an absolute constant factor:

\begin{lemma}\label{lemma:centering}
(Lemma 2.6.8 in \citep{vershynin2018high}) Let $\mu$ be $\sigma$-subgaussian, and let $X\sim \mu.$ Let $\tilde{\mu}$ be the distribution of $X-\mathbb{E}_{X\sim \mu}(X)$. Then $\tilde{\mu}$ is $\tilde{C}_G\sigma$-subgaussian, where $\tilde{C}_G$ is an absolute constant.
\end{lemma}

\subsection{Background on Optimization of Functionals on $\mathcal{P}_{2}(\mathbb{R}^{d})$}

\begin{definition}\label{def:convex}
Let $\mathcal{U}\subseteq \mathcal{P}_2(\mathbb{R}^d)$ be a convex subset (i.e., for all $\mu,\mu'\in \mathcal{U}$ and $t\in [0,1]$, $t\mu+(1-t)\mu'\in \mathcal{U}$). Then a functional $\mathcal{F}:\mathcal{U}\rightarrow\mathbb{R}$ is \emph{convex} if for all $\mu,\mu'\in \mathcal{U}$ and $t\in [0,1]$ we have: 
\begin{equation*}
\mathcal{F}((1-t)\mu+t\mu')\leq (1-t)\mathcal{F}(\mu)+t\mathcal{F}(\mu').
\end{equation*}
It is \emph{strictly convex} if the inequality is strict for $t\in (0,1).$
\end{definition}
If $\mathcal{F}$ is differentiable on a convex subset $\mathcal{U}\subseteq \mathcal{P}_{2}(\mathbb{R}^d)$, one has an alternative characterization of convexity:

\begin{proposition}\label{prop:diffcvx}
(Proposition 6 in \citep{janati2020debiased}) A differentiable functional $\mathcal{F}$ is convex over a convex subset $\mathcal{U}\subseteq\mathcal{P}_2(\mathbb{R}^d)$ iff for all $\mu,\mu'\in \mathcal{U}$ we have:
\begin{equation}\label{eqn:diffcvx_ineq}
\mathcal{F}(\mu)\geq \mathcal{F}(\mu')+\langle \delta \mathcal{F}(\mu'),\mu-\mu'\rangle.
\end{equation}
It is strictly convex if the inequality is strict for all $\mu\neq \mu'$.
\end{proposition}

We note that Proposition 6 in \citep{janati2020debiased} establishes the above for $\mathcal{U}=\mathcal{P}_{2}(\mathbb{R}^d)$ but the proof is the same for a general convex $\mathcal{U}\subset\mathcal{P}_{2}(\mathbb{R}^{d})$.

We note the following equivalent characterization of optimizers of a convex functional:

\begin{proposition}\label{prop:optimality} 
(Proposition 7 in \citep{janati2020debiased}) Let $\mathcal{F}$ be a convex, differentiable functional on a convex subset $\mathcal{U}\subseteq\mathcal{P}_2(\mathbb{R}^d)$. Then $\mu^*$ is minimal for $\mathcal{F}$ on $\mathcal{U}$ iff $\langle \delta \mathcal{F}(\mu^*),\mu-\mu^*\rangle\geq 0$ for all $\mu\in \mathcal{U}.$ 
\end{proposition}

Next, we show if $\mathcal{F}$ is differentiable and admits a minimizer $\mu^*$, then the derivative at $\mu^*$ is constant $\mu^*$-almost everywhere. The following is essentially Proposition 7.20 in \citep{santambrogio_2015}:

\begin{lemma}\label{lemma:mins_are_constant}
Let $\mathcal{U}$ be either $\mathcal{P}_2(\Omega)$ or $\mathcal{G}(\Omega)$, and let $\mathcal{F}:\mathcal{U}\rightarrow \mathbb{R}$ be differentiable on $\mathcal{U}$. Suppose that $\mathcal{F}$ admits a minimizer $\mu^*\in \mathcal{U}$.   Let $c^*=\inf_{x\in \Omega}\delta\mathcal{F}(\mu^*)(x)$.   Then:
\begin{itemize}[leftmargin=*]
\item $c^*>-\infty$;
\item $\delta\mathcal{F}(\mu^*)(x)=c^*$ for all $x\in \texttt{supp}(\mu^*)$.
\end{itemize}
\end{lemma}
\begin{proof}
Let $\mathcal{U}=\mathcal{P}_2(\Omega)$ (the proof for $\mathcal{G}(\Omega)$ is identical). Since $\mu^*\in \mathcal{P}_2(\Omega)$ is minimal, we have for any $\rho\in \mathcal{P}_2(\Omega)$ and $t\in (0,1]:$
\begin{equation*}
\frac{\mathcal{F}(\mu^*+t(\rho-\mu^*))-\mathcal{F}(\mu^*)}{t}\geq 0.
\end{equation*}
Taking the limit as $t\rightarrow 0^{+}$, we have that 
\begin{equation}
0\leq \lim_{t\rightarrow 0^+}\frac{\mathcal{F}(\mu^*+t(\rho-\mu^*))-\mathcal{F}(\mu^*)}{t}=\int \delta \mathcal{F}(\mu^*)d(\rho-\mu^*)=\langle \delta\mathcal{F}(\mu^*),\rho-\mu^*\rangle\label{eqn:mins_constant_0}
\end{equation}
by definition. Assume that $c^*=\inf \delta\mathcal{F}(\mu^*)(x)=-\infty$.  For any $\epsilon\in\mathbb{R}$, define $V^\epsilon:=\{x\in \Omega \ | \ \delta\mathcal{F}(\mu^*)(x)>\epsilon\}$. These sets are measurable by continuity of $\delta\mathcal{F}(\mu^*)$,
and by definition of the infimum, $(V^\epsilon)^{c}$ is nonempty for any $\epsilon.$  By definition $\{x\in \Omega \ | \ \delta\mathcal{F}(\mu^*)(x)=-\infty\}=\emptyset$. 
 Towards a contradiction, assume that $\mu^*(V^{\bar{\epsilon}})>0$ for some fixed $\bar{\epsilon}$, and define $A:=\texttt{supp}(\mu^*)\cap V^{\bar{\epsilon}}$, and denote $\Omega\setminus A$ by $A^c\neq \emptyset.$ We define $\mu^*_A(B):=\mu^*(A\cap B)$ and $\mu^*_{A^c}(B)=\mu^*(A^c\cap B)$. Then for any $\nu\in \mathcal{P}_2(A^c)$ and $t\in (0,1)$, we may define $\mu_t:=(1-t)\mu^*_A+\mu^*_{A^c}+t\mu^*(A)\nu$, which lies in $\mathcal{P}_{2}(\Omega)$.  Using the fact that $\mu^*=\mu^*_A+\mu^*_{A^c}$, 
\begin{align}
\langle \delta \mathcal{F}(\mu^*),\mu_t-\mu^*\rangle &= \langle\delta\mathcal{F}(\mu^*),(1-t)\mu^*_{A}+\mu^*_{A^c}+t\mu^*(A)\nu-\mu^*_A-\mu^*_{A^c}\rangle\notag \\ &= \langle \delta\mathcal{F}(\mu^*),t\mu^*(A)\nu-t\mu^*_A\rangle\notag \\& =t\left(\langle \delta\mathcal{F}(\mu^*),\mu^*(A)\nu\rangle -\langle \delta\mathcal{F}(\mu^*),\mu^*_A\rangle\right) \label{eqn:mins_constant_1}
\end{align}
By construction, \begin{equation*} \langle \delta\mathcal{F}(\mu^*),\mu^*_A\rangle >\bar{\epsilon}\mu^*(A),\end{equation*} whereas \begin{equation*} \langle \delta\mathcal{F}(\mu^*),\mu^*(A)\nu\rangle =\mu^*(A)\langle \delta\mathcal{F}(\mu^*),\nu\rangle\leq\bar{\epsilon} \mu^*(A), \end{equation*} and hence (\ref{eqn:mins_constant_1}) is less than 0, contradicting (\ref{eqn:mins_constant_0}). Hence we must conclude that $\mu^*(V^{\bar{\epsilon}})=0$ for any fixed $\bar{\epsilon}$. But this implies that $\mu^*$ must be supported on $\Omega\setminus\cup_{\epsilon\in \mathbb{R}} V^{\epsilon}$, which is empty, giving us a contradiction. We conclude that $c^*=\inf \delta\mathcal{F}(\mu^*)(x)>-\infty.$

We prove the second claim similarly.  We define for any $\epsilon>0$,
\begin{equation*}
U^\epsilon:=\{x\in \Omega \ | \  \delta\mathcal{F}(\mu^*)(x)>c^*+\epsilon\}.
\end{equation*}
which is measurable by continuity of $\delta\mathcal{F}(\mu^*)$. 
By definition of $c^*$, the set $\Omega\setminus U^\epsilon$ is nonempty for all $\epsilon>0$. By way of contradiction, suppose there exists an $\bar{\epsilon}>0$ such that $\mu^*( U^{\bar{\epsilon}})>0$.  Let $A:=\texttt{supp}(\mu^*)\cap U^{\bar{\epsilon}}$ and let $\mu_{A}^{*}$ and $\mu_{A^c}^{*}$ be as above. Let $\nu \in \mathcal{P}_2(A^c)$, and for any $t\in (0,1)$, define $\mu_t=(1-t)\mu^*_A+\mu^*_{A^c}+t\mu^*(A)\nu$
, which lies in $\mathcal{P}_{2}(\Omega)$.  Then:
\begin{align}
\langle \delta \mathcal{F}(\mu^*),\mu_t-\mu^*\rangle &= \langle\delta\mathcal{F}(\mu^*),(1-t)\mu^*_{A}+\mu^*_{A^c}+t\mu^*(A)\nu-\mu^*_A-\mu^*_{A^c}\rangle\notag \\ &= \langle \delta\mathcal{F}(\mu^*),t\mu^*(A)\nu-t\mu^*_A\rangle\notag \\& =t\left(\langle \delta\mathcal{F}(\mu^*),\mu^*(A)\nu\rangle -\langle \delta\mathcal{F}(\mu^*),\mu^*_A\rangle\right).\label{eqn:mins_constant_2}
\end{align}
By construction, \begin{equation*} \langle \delta\mathcal{F}(\mu^*),\mu^*_A\rangle > (c^*+\bar{\epsilon})\mu^*(A),\end{equation*} whereas \begin{equation*} \langle \delta\mathcal{F}(\mu^*),\mu^*(A)\nu\rangle =\mu^*(A)\langle \delta\mathcal{F}(\mu^*),\nu\rangle\leq(c^*+\bar{\epsilon}) \mu^*(A), \end{equation*} and hence (\ref{eqn:mins_constant_2}) is less than 0, contradicting (\ref{eqn:mins_constant_0}), hence it must be the case that $\mu^*(\texttt{supp}(\mu^*)\cap U^\epsilon)=0$ for all $\epsilon>0$, i.e., that $\mu^*\in \mathcal{P}(\Omega\setminus U^\epsilon)$ for all $\epsilon>0$ . Hence for $\mu^*$-almost every $x$ and all $\epsilon>0$, $\delta\mathcal{F}(\mu^*)(x)< c^*+\epsilon$. Hence for $\mu^*$-almost every $x$, \begin{equation*}c^*\leq  \delta\mathcal{F}(\mu^*)(x)<\liminf_{\epsilon\rightarrow 0^+}c^*+\epsilon\end{equation*}
which implies that $\delta\mathcal{F}(\mu^*)(x)=c^*$. We may immediately promote this to $\delta\mathcal{F}(\mu^*)(x)=c^*$ everywhere on $\texttt{supp}(\mu^*)$ by continuity of $\delta\mathcal{F}(\mu^*)$. 
\end{proof}

\begin{corollary}\label{cor:mins_are_fixed}
Let $\mathcal{U}$ be either $\mathcal{P}_2(\mathbb{R}^d)$ or $\mathcal{G}(\mathbb{R}^d)$. Let $\mathcal{F}:\mathcal{U}\rightarrow \mathbb{R}$ be a differentiable, convex functional which admits a minimizer $\mu^*\in \mathcal{U}$. Furthermore, suppose that $\delta\mathcal{F}(\mu^*)$ is differentiable. Then $\mu^*$ is a critical point of $\mathcal{F}$ on $\mathcal{U}$.
\end{corollary}
\begin{proof}
From Lemma \ref{lemma:mins_are_constant}, we know that $\delta\mathcal{F}(\mu^*)(x)=c^*$ for all $x\in \texttt{supp}(\mu^*)$, where $c^*:=\min_{x\in \mathbb{R}^d}\delta\mathcal{F}(\mu^*)(x)$. Hence any $x\in \texttt{supp}(\mu^*)$ minimizes a differentiable function $\delta\mathcal{F}(\mu^*)$, and hence $\nabla\delta\mathcal{F}(\mu^*)(x)=0$. Thus $\mu^*$ is a critical point of $\mathcal{F}$.
\end{proof}

\section{PROOFS FOR SECTION \ref{sec:analytic properties}}\label{section: proofs for sec 2}

\subsection{Proof of Proposition \ref{prop:existence}}\label{sec:proofofexist}

We begin by proving existence. We will require Prokhorov's Theorem, which provides a sufficient condition for convergence of an infimizing subsequence for our functionals:

\begin{definition}
Let $(\mathcal{X},d)$ be a separable metric space. A sequence of probability measures $\{\mu_n\}_{n=1}^\infty\subseteq \mathcal{P}(\mathcal{X})$ is \emph{tight} if for all $\epsilon>0$ there exists a compact set $K_\epsilon$ such that $\mu_n(\mathcal{X}\setminus K_\epsilon)<\epsilon$ for all $n$.
\end{definition}
\begin{theorem}\label{thm:prokhorov}
(Prokhorov's Theorem, Theorems 5.1 and 5.2 in \citep{billingsley2013convergence}) Let $(\mathcal{X},d)$ be a complete separable metric space. Then a sequence of probability measures $\{\mu_n\}_{n=1}^\infty\subseteq \mathcal{P}(\mathcal{X})$ is tight if and only if every subsequence of $\{\mu_n\}_{n=1}^\infty$ contains a further subsequence which weakly converges to a limit in $\mathcal{P}(\mathcal{X}).$
\end{theorem}

\begin{lemma}\label{lem:tight}
Let $\mathcal{V}\subset \mathcal{P}_2(\Omega)$. Any infimizing sequence for $F^\epsilon_{\lambda,\mathcal{V}}$ is tight. 
\end{lemma}
\begin{proof} Let $\{\mu_n\}_{n=1}^{\infty}$ be an infimizing sequence for $F^\epsilon_{\lambda,\mathcal{V}}$. We claim that $\|\mathbb{E}(\mu_n)\|$ is uniformly bounded over $n$. Let $\rho\in \mathcal{P}_2(\Omega)$ and define $L:=F^\epsilon_{\lambda,\mathcal{V}}(\rho)<\infty$. As $\{\mu_n\}_{n=1}^{\infty}$ is infimizing, there exists some $N$ such that for all $n>N$, $F_{\lambda,\mathcal{V}}^\epsilon(\mu_n)\le L$.  Since $OT_2(\mu,\nu)^2\leq OT_2^\epsilon(\mu,\nu)$, we have:
\begin{align}
L & \ge F^\epsilon_{\lambda,\mathcal{V}}(\mu_n) \notag\\
&  \geq \sum_{j=1}^m \lambda_j OT_2(\mu_n,\nu_j)^2 \notag \\
& =\sum_{j=1}^m \lambda_j \inf_{\zeta_j\in \Pi(\mu_{n},\nu_j)}\int \frac{1}{2}\|x-y\|^2d\zeta_j(x,y) \notag\\
& \geq \frac{1}{2}\sum_{j=1}^m \lambda_j \inf_{\zeta_{j}\in \Pi(\mu_n,\nu_j)} \left\|\int (x-y)d\zeta_j(x,y)\right\|^2 \label{eqn:LemmaB1_Jensen}\\
&= \frac{1}{2}\sum_{j=1}^m \lambda_j \left\| \int x d\mu_n(x)- \int y d\nu_j(y)\right\|^2 \notag \\
&=\frac{1}{2}\sum_{j=1}^m\lambda_j\|\mathbb{E}(\mu_n)-\mathbb{E}(\nu_j)\|^2 \notag \\
& \geq \frac{1}{2}\min_{1\leq j\leq m}\|\mathbb{E}(\mu_n)-\mathbb{E}(\nu_j)\|^2\label{eqn:firstmomentlowerbound}
\end{align}
where we applied Jensen's inequality at (\ref{eqn:LemmaB1_Jensen}), and hence $\min_{1\leq j\leq m}\|\mathbb{E}(\mu_n)-\mathbb{E}(\nu_j)\|\leq \sqrt{2L}$ for all large enough $n$. We have thus shown that $\min_{1\leq j\leq m}\|\mathbb{E}(\mu_n)-\mathbb{E}(\nu_j)\|$ is uniformly bounded over all $n$, which implies that $\|\mathbb{E}(\mu_n)\|$ is uniformly bounded over all $n$ as well by some constant $M$.

Next, we claim that $\texttt{Var}(\mu_n)$ is uniformly bounded for all $n$. To see this, note that for any Dirac mass $\delta_x$ and $\mu\in\mathcal{P}_2(\Omega)$, $OT^\epsilon_2(\delta_x,\mu)=OT_2^2(\delta_x,\mu)$, as the only coupling between $\delta_x$ and $\mu$ is $\delta_x\otimes \mu$, and hence the KL term vanishes. We lower bound: $OT_2^2(\delta_x,\mu)\geq \min_{z}OT_2^2(\delta_z,\mu)$, then observe that:
\begin{align*}
\min_zOT_2^2(\delta_z,\mu) & = \frac{1}{2}\min_{z}\int \|z-y\|^2d\mu(y) & \\ & = \frac{1}{2}\int \|\mathbb{E}(\mu)-y\|^2d\mu(y)& \\&  =\frac{1}{2}\texttt{Var}(\mu),
\end{align*}
which follows from the characterization of $\mathbb{E}(\mu)$ as the minimizer for the mean-squared error of $\mu$. Now, fix some $x_0\in \mathbb{R}^d$. Then by the triangle inequality for $OT_2$ and the bound $L\ge \sum_{j=1}^m\lambda_j OT_2(\mu_n,\nu_j)^2$, for all large enough $n$ there exists an index $1\leq j_n\leq m$ such that:
\begin{align*}
\sqrt{\texttt{Var}(\mu_n)} & \leq  \sqrt{2}OT_2(\delta_{x_0},\mu_n) \\ &\leq \sqrt{2}OT_2(\delta_{x_0},\delta_{\mathbb{E}(\nu_{j_n})})+\sqrt{2}OT_2(\delta_{\mathbb{E}(\nu_{j_n})},\nu_{j_n})+\sqrt{2}OT_2(\nu_{j_n},\mu_n) \\ &\leq \sqrt{2}\|x_0-\mathbb{E}(\nu_{j_n})\|+\sqrt{2\texttt{Var}(\nu_{j_n})}+\sqrt{2L}
\end{align*}
Hence $\texttt{Var}(\mu_n)\leq \max_{1\leq j\leq m}(\sqrt{2}\|x_0-\mathbb{E}(\nu_j)\|+\sqrt{2\texttt{Var}(\nu_j)}+\sqrt{2L})^2:=Z$ for all large enough $n$, which implies a uniform bound for all $n$. As we have uniformly bounded $\texttt{Var}(\mu_n)$ for all $n$, we may apply Chebyshev's inequality to conclude that:
\begin{equation}
\mathbb{P}_{X\sim \mu_n}(\|X-\mathbb{E}(\mu_n)\| > t) <\frac{Z}{t^2}\label{eqn:chebyshev}
\end{equation} 
for some $Z\geq 0$ and for all $n$ and $t>0$. As $\|\mathbb{E}(\mu_n)\|\leq M$ for all $n$, $\mathbb{E}(\mu_n)\in \overline{B_M(0)}$, and hence for any $t>0$ and $n\in \mathbb{N}$ we have \begin{equation*} A^t_n:=\{x\in \Omega \ | \; \|x-\mathbb{E}(\mu_n)\|\leq t\}\subseteq \overline{B_{t+M}(0)}\cap \Omega. \end{equation*}
By this and (\ref{eqn:chebyshev}) we have: 
\begin{align*}
\mathbb{P}_{X\sim \mu_n}(X\in \overline{B_{t+M}(0)}\cap \Omega)\geq \mathbb{P}_{X\sim \mu_n}(X\in A_n^t) >1-\frac{Z}{t^2} 
\end{align*}
 for all $n$ and $t>0.$  
 Since $\Omega$ is closed, $\overline{B_{t+M}(0)}\cap \Omega$ is compact, and hence $K(t):=\overline{B_{t+M}(0)} \cap \Omega$ forms a family of compact sets such that, for any $\epsilon>0$ and $n\in \mathbb{N}$, $\mu_n(K(t))>1-\epsilon$ for all $t>\sqrt{Z/\epsilon}$, establishing that $\{\mu_n\}_{n=1}^\infty$ is tight.  
\end{proof}

We now show that for fixed $\sigma>0$, a $\sigma$-subgaussian infimizing sequence for $S^\epsilon_{\lambda,\mathcal{V}}$ is tight. 

\begin{lemma}\label{lemma:sinktight}
Let $\sigma>0$. Suppose that $\{\mu_n\}_{n=1}^\infty\subset\mathcal{G}_\sigma(\Omega)$ is an infimizing sequence for $S^\epsilon_{\lambda,\mathcal{V}}$, where $\mathcal{V}\subset \mathcal{P}_2(\Omega).$  Then $\|\mathbb{E}(\mu_n)\|$ is uniformly bounded and $\{\mu_n\}_{n=1}^\infty$ is tight. 
\end{lemma}

\begin{proof}
Let $\displaystyle\lim_{n\rightarrow \infty} S^\epsilon_{\lambda,\mathcal{V}}(\mu_n) =\inf_{\mu\in \mathcal{G}_\sigma(\Omega)} S^\epsilon_{\lambda,\mathcal{V}}(\mu)$. We construct a lower bound on $S^\epsilon_{\lambda,\mathcal{V}}$ along our infimizing sequence as follows.  For all $n$, we have:
\begin{align}
S^\epsilon_{\lambda,\mathcal{V}}(\mu_n)\notag& =\sum_{j=1}^m \lambda_j (OT_2^\epsilon(\mu_n,\nu_j)-\frac{1}{2}OT_2^\epsilon(\nu_j,\nu_j)) -\frac{1}{2}OT_2^\epsilon(\mu_n,\mu_n) \\ & \geq \sum_{j=1}^m \lambda_j OT_2^\epsilon(\mu_n,\nu_j) -J-\frac{1}{2}OT_2^\epsilon(\mu_n,\mu_n)\notag \\ & \geq \sum_{j=1}^m \lambda_j OT_2(\mu_n,\nu_j)^2 -J-\frac{1}{2}OT_2^\epsilon(\mu_n,\mu_n)  \label{eqn:sink_exist_1}\\ & \geq \sum_{j=1}^m \lambda_j OT_2(\mu_n,\nu_j)^2 -J-\frac{1}{2}\texttt{Var}(\mu_n) \label{eqn:sink_exist_2} \\ & \geq \sum_{j=1}^m \lambda_j OT_2(\mu_n,\nu_j)^2 -J-\frac{1}{2}\tilde{q}_Gd\sigma^2\label{eqn:sink_exist_3}
\end{align}
where $J=\max_{1\leq j\leq m} OT_2^\epsilon(\nu_j,\nu_j)$, and $\tilde{q}_G$ is an absolute constant. In (\ref{eqn:sink_exist_1}) we used the fact that $OT_2^\epsilon(\rho_1,\rho_2)\geq OT_2(\rho_1,\rho_2)^2$ for any $\rho_1,\rho_2\in \mathcal{P}_2(\Omega)$ and in (\ref{eqn:sink_exist_2}), we used:
\begin{align*}
OT_2^\epsilon(\mu_n,\mu_n) & =\inf_{\zeta\in \Pi(\mu_n,\mu_n)} \int \frac{1}{2}\|x-y\|^2 d\zeta(x,y) +\epsilon KL(\zeta | \mu_n\otimes \mu_n) \\ & \leq  \int \frac{1}{2}\|x-y\|^2 d\mu_n\otimes \mu_n(x,y) +\epsilon KL(\mu_n\otimes \mu_n| \mu_n\otimes \mu_n) \\ & =  \int \frac{1}{2}\|x-y\|^2 d\mu_n\otimes \mu_n(x,y) \\ & = \int \|x-\mathbb{E}(\mu_n)\|^2 d\mu_n = \texttt{Var}(\mu_n).
\end{align*}
In (\ref{eqn:sink_exist_3}) we applied Lemma \ref{lemma:subgauss_var_bound}.

Using the same arguments to establish (\ref{eqn:firstmomentlowerbound}), we have
\begin{align*}
\sum_{j=1}^m \lambda_j OT_2(\mu_n,\nu_j)^2 \notag 
& =\sum_{j=1}^m \lambda_j \inf_{\zeta_j\in \Pi(\mu_{n},\nu_j)}\int \frac{1}{2}\|x-y\|^2d\zeta_j(x,y) \notag\\
& \geq \frac{1}{2}\sum_{j=1}^m \lambda_j \inf_{\zeta_{j}\in \Pi(\mu_n,\nu_j)} \left\|\int (x-y)d\zeta_j(x,y)\right\|^2 \\
&= \frac{1}{2}\sum_{j=1}^m \lambda_j \left\| \int x d\mu_n(x)- \int y d\nu_j(y)\right\|^2 \notag \\
&=\frac{1}{2}\sum_{j=1}^m\lambda_j\|\mathbb{E}(\mu_n)-\mathbb{E}(\nu_j)\|^2 \notag \\
& \geq \frac{1}{2}\min_{1\leq j\leq m}\|\mathbb{E}(\mu_n)-\mathbb{E}(\nu_j)\|^2
\end{align*}
and thus we may lower bound $S^\epsilon_{\lambda,\mathcal{V}}(\mu_n)$ by:
\begin{equation}
S^\epsilon_{\lambda,\mathcal{V}}(\mu_n)\geq  \frac{1}{2}\min_{1\leq j\leq m}\|\mathbb{E}(\mu_n)-\mathbb{E}(\nu_j)\|^2 - J -\frac{1}{2}\tilde{q}_Gd\sigma^2.\label{eqn:sink_coercive}
\end{equation}
Since $\mu_n$ is infimizing and $\inf_{\mu\in \mathcal{G}_\sigma(\Omega)}S^\epsilon_{\lambda,\mathcal{V}}(\mu_n)<\infty$, there exists $B>0$ such that $\sup_{n}S^\epsilon_{\lambda,\mathcal{V}}(\mu_n)<B$ . Hence $ \frac{1}{2}\min_{1\leq j\leq m}\|\mathbb{E}(\mu_n)-\mathbb{E}(\nu_j)\|^2 < B+J+\frac{1}{2}\tilde{q}_Gd\sigma^2$ for all $n$, and we can conclude that $\|\mathbb{E}(\mu_n)\|$ must be uniformly bounded in $n$. To conclude, we have established a uniform bound on $\|\mathbb{E}(\mu_n)\|$, and assumed a uniform bound on $\texttt{Var}(\mu_n)$ via Lemma \ref{lemma:subgauss_var_bound}. We may thus repeat the argument in Lemma \ref{lem:tight} to establish that $\{\mu_n\}_{n=1}^{\infty}$ is tight.  
\end{proof}

We will also need the following continuity properties of $OT_2^\epsilon:$

\begin{theorem}\label{thm:thm3.7_in_eckstein}
(Theorem 3.7 in \citep{eckstein2022quantitative}) Let $\mu,\nu\in \mathcal{P}_2(\mathbb{R}^d)$. Then:
\begin{align*}
&|OT^\epsilon_2(\mu,\nu)-OT^\epsilon_2(\mu',\nu')|\\
\leq&\sqrt{2}\left(\sqrt{M_2(\mu)}+\sqrt{M_2(\mu')}+\sqrt{M_2(\nu)}+\sqrt{M_2(\nu')}\right)\sqrt{OT_2(\mu,\mu')^2+OT_2(\nu,\nu')^2}.
\end{align*}
\end{theorem}

This immediately implies a continuity result for $\mathcal{F}^\epsilon_{\lambda,\mathcal{V}}$:

\begin{corollary}\label{cor:continuity}
Let $\mathcal{V}\subset \mathcal{G}(\mathbb{R}^d)$. For any $\mu,\mu'\in\mathcal{P}_2(\mathbb{R}^d)$ and $\epsilon>0$, $F^\epsilon_{\lambda,\mathcal{V}}(\mu)$ satisfies: \begin{equation*}|F^\epsilon_{\lambda,\mathcal{V}}(\mu')-F^\epsilon_{\lambda,\mathcal{V}}(\mu)|\leq \sqrt{2}\left(2\max_{1\le j\le m}\sqrt{M_2(\nu_j)}+\sqrt{M_2(\mu)}+\sqrt{M_2(\mu')}\right)OT_2(\mu,\mu').\end{equation*} Similarly, $S^\epsilon_{\lambda,\mathcal{V}}(\mu)$ satisfies: 
\begin{equation*}
|S^\epsilon_{\lambda,\mathcal{V}}(\mu')-S^\epsilon_{\lambda,\mathcal{V}}(\mu)|\leq 3\sqrt{2}\left(\max_{1\leq j \leq m}\sqrt{M_2(\nu_j)}+\sqrt{M_2(\mu)}+\sqrt{M_2(\mu')}\right)OT_2(\mu,\mu').
\end{equation*}
\end{corollary}
\begin{proof}Theorem \ref{thm:thm3.7_in_eckstein} establishes:
\begin{equation*}
|OT^\epsilon_2(\mu,\nu_j)-OT^\epsilon_2(\mu',\nu_j)|\leq \sqrt{2}\left(2\sqrt{M_2(\nu_j)}+\sqrt{M_2(\mu)}+\sqrt{M_2(\mu')}\right)OT_2(\mu,\mu')
\end{equation*}
for each $j$. We then bound:
\begin{align*}
& |F^\epsilon_{\lambda,\mathcal{V}}(\mu)-F^\epsilon_{\lambda,\mathcal{V}}(\mu')| \\  = &\left|\sum_{j=1}^m \lambda_j OT^\epsilon_2(\mu,\nu_j)-\sum_{j=1}^m \lambda_j OT^\epsilon_2(\mu',\nu_j)\right| \\
\leq &\sum_{j=1}^m \lambda_j|OT^\epsilon_2(\mu,\nu_j)- OT^\epsilon_2(\mu',\nu_j)| \\
\leq &\sum_{j=1}^m\lambda_j \sqrt{2}\left(\sqrt{M_2(\mu)}+\sqrt{M_2(\mu')}+2\sqrt{M_2(\nu_j)}\right)OT_2(\mu,\mu')  \\ 
 = &\sqrt{2}(\sqrt{M_2(\mu)}+\sqrt{M_2(\mu')})OT_2(\mu,\mu')+2\sqrt{2}\sum_{j=1}^m\lambda_j\sqrt{M_2(\nu_j)}OT_2(\mu,\mu') \\
 \leq  &\sqrt{2}(\sqrt{M_2(\mu)}+\sqrt{M_2(\mu')})OT_2(\mu,\mu')+2\sqrt{2}\max_{1\leq j\leq m}\sqrt{M_2(\nu_j)}OT_2(\mu,\mu') \\
  = &\sqrt{2}\left(2\max_{1\leq j \leq m}\sqrt{M_2(\nu_j)}+\sqrt{M_2(\mu)}+\sqrt{M_2(\mu)}\right)OT_2(\mu,\mu').
\end{align*}
The proof for $S^\epsilon_{\lambda,\mathcal{V}}$ is similar: 
\begin{align*}
&|S^\epsilon_{\lambda,\mathcal{V}}(\mu')-S^\epsilon_{\lambda,\mathcal{V}}(\mu)|\\ 
= &\bigg|\sum_{j=1}^m\lambda_j (OT_2^\epsilon(\mu,\nu_j)-OT_2^\epsilon(\mu',\nu_j)) + \frac{1}{2} (OT_2^\epsilon(\mu',\mu')-OT_2^\epsilon(\mu,\mu))\bigg| \\
\leq &\bigg|\sum_{j=1}^m\lambda_j OT_2^\epsilon(\mu,\nu_j)-OT_2^\epsilon(\mu',\nu_j)\bigg| + \frac{1}{2}\bigg| OT_2^\epsilon(\mu',\mu')-OT_2^\epsilon(\mu,\mu)\bigg|.
\end{align*}
By Theorem \ref{thm:thm3.7_in_eckstein}, we may bound the first term by \begin{equation*}\sqrt{2}\left(2\max_{1\le j\le m}\sqrt{M_2(\nu_j)}+\sqrt{M_2(\mu)}+\sqrt{M_2(\mu')}\right)OT_2(\mu,\mu'),
\end{equation*}and the second term by $2\sqrt{2}(\sqrt{M_2(\mu)}+\sqrt{M_2(\mu')})OT_2(\mu,\mu').$
\end{proof}

\vspace{10pt}
\noindent\textbf{Proof of Proposition \ref{prop:existence}:}  We first establish existence of the minimizers for $F^\epsilon_{\lambda,\mathcal{V}}$.  For all $j$, the functional $OT^\epsilon_2(\mu,\nu_j)$ is non-negative and lower semicontinuous w.r.t. the weak convergence of probability measures. This can be seen by noting (\ref{eqn:dualform}) implies that $OT^\epsilon_2(\mu,\nu_j)$ is the supremum of a family of continuous linear functionals, and hence is lower semicontinuous \citep{Bell_2014}. This implies that $F_{\lambda,\mathcal{V}}^\epsilon$ is lower semicontinuous and bounded below. As $F^\epsilon_{\lambda,\mathcal{V}}$ is bounded below, $\inf_{\mu\in\mathcal{P}_2(\Omega)}F^\epsilon_{\lambda,\mathcal{V}}(\mu)$ is finite.  Let $\{\mu_n\}_{n=1}^\infty$ be an infimizing sequence for $F_{\lambda,\mathcal{V}}^\epsilon$, which by Lemma \ref{lem:tight} is tight. As $\Omega$ is closed, it is a complete metric space w.r.t. $d(x,y)=\|x-y\|$. Hence by Theorem \ref{thm:prokhorov}, the sequence $\{\mu_n\}_{n=1}^\infty$ contains a subsequence $\{\mu_n'\}_{n=1}^\infty\subset \{\mu_n\}_{n=1}^\infty$ which converges weakly to a limit $\mu'\in\mathcal{P}_2(\Omega).$ By lower semicontinuity, $F_{\lambda,\mathcal{V}}^\epsilon(\mu')\leq \liminf_{n\rightarrow\infty} F_{\lambda,\mathcal{V}}^\epsilon(\mu_n')=\inf_{\mu\in \mathcal{P}_2(\Omega)}F^\epsilon_{\lambda,\mathcal{V}}(\mu)$ as $\{\mu_n'\}_{n=1}^{\infty}$ is a subsequence of an infimizing sequence, and since $\inf_{\mu\in\mathcal{P}_{2}(\Omega)}F_{\lambda,\mathcal{V}}^\epsilon(\mu)\leq F_{\lambda,\mathcal{V}}^\epsilon(\mu')$ it implies that $\mu'$ achieves the infimum.

We now establish existence for $S^\epsilon_{\lambda,\mathcal{V}}.$ By (\ref{eqn:sink_exist_3}), $S^\epsilon_{\lambda,\mathcal{V}}$ is uniformly lower bounded on $\mathcal{G}_\sigma(\Omega)$, and hence $\inf_{\mu\in G_\sigma(\Omega)} S^\epsilon_{\lambda,\mathcal{V}}(\mu)>-\infty$.  By definition of infimum, we may find a sequence $\{\mu_n\}_{n=1}^\infty\subseteq \mathcal{G}_\sigma(\Omega)$ so that $\lim_{n\rightarrow \infty}S^\epsilon_{\lambda,\mathcal{V}}(\mu_n)=\inf_{\mu\in \mathcal{G}_\sigma(\Omega)}S^\epsilon_{\lambda,\mathcal{V}}(\mu)$. By Lemma \ref{lemma:sinktight}, $\{\mu_n\}_{n=1}^\infty$ is tight, and hence by Prokhorov's theorem it has a subsequence $\{\mu_n'\}_{n=1}^\infty$ weakly converging to a limit $\mu'$. We now claim that:

\begin{equation}\label{eqn:sinklsc} S^\epsilon_{\lambda,\mathcal{V}}(\mu')= \lim_{n\rightarrow \infty}S^\epsilon_{\lambda,\mathcal{V}}(\mu_n').\end{equation}

By Lemma \ref{lemma:weakly closed}, $\mu'\in \mathcal{G}_\sigma(\mathbb{R}^d)$, and $\lim_{n\rightarrow \infty} M_2(\mu_n')=M_2(\mu')$. 
 Weak convergence and convergence of second moments implies that $OT_2(\mu_n',\mu')\rightarrow 0$ as $n\rightarrow \infty$ (Theorem 5.11 in \citep{santambrogio_2015}). We now appeal to Corollary \ref{cor:continuity}, which establishes that:
\begin{align*}
|S^\epsilon_{\lambda,\mathcal{V}}(\mu_n')-S^\epsilon_{\lambda,\mathcal{V}}(\mu')| & \leq 3\sqrt{2}\left(\max_{1\leq j \leq m}\sqrt{M_2(\nu_j)}+\sqrt{M_2(\mu_n')}+\sqrt{M_2(\mu')}\right)OT_2(\mu_n',\mu') \\ & \leq18\sqrt{d}q_G \sigma OT_2(\mu_n',\mu')\rightarrow 0, \;\;\;\; n\rightarrow \infty,
\end{align*}
where we applied the bound $\sqrt{M_2(\rho)}\leq \sqrt{2d}q_G \sigma$ for all $\rho\in \mathcal{G}_\sigma(\mathbb{R}^d)$ (item 3 in Definition \ref{def:subgauss} and Lemma \ref{lemma:normsubgauss}). As $\{S^\epsilon_{\lambda,\mathcal{V}}(\mu'_n)\}_{n=1}^\infty$ is a subsequence of a convergent sequence, we may conclude that

\begin{equation*}
S^\epsilon_{\lambda,\mathcal{V}}(\mu')=\lim_{n\rightarrow \infty} S^\epsilon_{\lambda,\mathcal{V}}(\mu'_n)=\lim_{n\rightarrow \infty} S^\epsilon_{\lambda,\mathcal{V}}(\mu_n)=\inf_{\mu\in \mathcal{G}_\sigma(\Omega)}S^\epsilon_{\lambda,\mathcal{V}}(\mu),
\end{equation*}
and thus $\mu'$ minimizes $S^\epsilon_{\lambda,\mathcal{V}}$ on $\mathcal{G}_\sigma(\Omega).$

We now establish uniqueness of the minimizer of $S^\epsilon_{\lambda,\mathcal{V}}$ when $\Omega$ is bounded.  By Proposition 4 in \citep{feydy2019interpolating}, the functional $-\frac{1}{2}OT_2^\epsilon(\mu,\mu):\mathcal{P}_2(\Omega)\rightarrow \mathbb{R}$ is strictly convex. As $\sum_{j=1}^m \lambda_j OT_2^\epsilon(\mu,\nu_j)$ is convex, $S^\epsilon_{\lambda,\mathcal{V}}$ is strictly convex on $\mathcal{P}_2(\Omega)$, from which we conclude uniqueness of the minimizer.
\qed

\subsection{Proof of Theorem \ref{thm:entot_differentiable}}\label{sec:proof_of_entot_diff}

First, we state general bounds for subgaussian random variables: 
\begin{lemma}\label{lem:lemma1.5rig}
(Lemma 1.5 in \citep{rigollet2023high}) Suppose $\mu$ is $\sigma$-subgaussian in $\mathbb{R}^d$ and $\mathbb{E}(\mu)=0$. Then, for any $s>0$ we have:
\begin{equation*}
\sup_{v\in S^{d-1}}\int \exp(s |\langle x,v\rangle|)d\mu(x)\leq \exp(4\sigma^2s^2).
\end{equation*}
\end{lemma}

\noindent We obtain an immediate corollary:

\begin{corollary}
\label{cor:cor_of_lem1.5_rig} Suppose $\mu$ is $\sigma$-subgaussian in $\mathbb{R}^d$ and $z\in \mathbb{R}^d$. Then:
\begin{equation*}
\int \exp(|\langle x,z\rangle|)d\mu(x)\leq \exp(4\tilde{C}_G^2\sigma^2 \|z\|^2+\|\mathbb{E}(\mu)\|\|z\|).
\end{equation*}
where $\tilde{C}_G$ is an absolute constant.
\end{corollary}
\begin{proof}For any $z$,
\begin{align}
\int \exp(|\langle x, z \rangle|)d\mu(x) 
= &\int \exp(|\langle x-\mathbb{E}(\mu)+\mathbb{E}(\mu), z \rangle|)d\mu(x)\notag \\ 
\leq & \int \exp(|\langle x-\mathbb{E}(\mu), z \rangle|+|\langle \mathbb{E}(\mu), z \rangle|)d\mu(x)\label{eqn:rigbound_1} \\ 
= & \exp(|\langle \mathbb{E}(\mu), z \rangle|)\int \exp(|\langle x-\mathbb{E}(\mu), z \rangle|)d\mu(x)\notag \\ 
= & \exp(|\langle \mathbb{E}(\mu), z \rangle|)\int \exp\left(\left|\left\langle x-\mathbb{E}(\mu), \frac{z}{\|z\|} \right\rangle\right| \|z\|\right)d\mu(x)\notag \\ 
\leq & \exp(\|\mathbb{E}(\mu)\|\|z\|)\sup_{v\in S^{d-1}}\int \exp(|\langle x-\mathbb{E}(\mu), v \rangle|\|z\| )d\mu(x)\label{eqn:rigbound_2} \\ 
\leq & \exp(4\|\tilde{\mu}\|^2_\mathcal{G}\;\|z\|^2 + \|\mathbb{E}(\mu)\|\|z\|).\label{eqn:rigbound_3}
\end{align}
where $\tilde{\mu}$ is the distribution of the mean zero random variable $X-\mathbb{E}(\mu)$. In (\ref{eqn:rigbound_1}) we applied the triangle inequality, in (\ref{eqn:rigbound_2}) we applied the Cauchy-Schwarz inequality, and (\ref{eqn:rigbound_3}) follows from Lemma \ref{lem:lemma1.5rig}. We conclude by applying Lemma \ref{lemma:centering} to $\tilde{\mu}$, giving us a final bound of $\exp(4\tilde{C}_G^2\sigma^2\|z\|^2+\|\mathbb{E}(\mu)\|\| z\| )$.
\end{proof}

\begin{lemma}\label{lem:4.9 from nutz}
(Lemma 4.9 from \citep{nutz2021introduction}) Let $\mu,\nu\in\mathcal{P}_2(\mathbb{R}^d)$, and choose entropic potentials that solve (\ref{eqn:dualform}) such that $\int f^\epsilon_{\mu\rightarrow\nu}d\mu\geq 0$ and $\int g^\epsilon_{\mu\rightarrow \nu}d\nu\geq 0$. Then:
\begin{equation*}
\inf_{y\in\mathbb{R}^d}\left\{ \frac{1}{2}\|x-y\|^2-g^\epsilon_{\mu\rightarrow \nu}(y)\right\}\leq f^\epsilon_{\mu\rightarrow \nu}(x) \leq \frac{1}{2}\int\|x-y\|^2d\nu(y),
\end{equation*}
\begin{equation*}
\inf_{x\in\mathbb{R}^d}\left\{ \frac{1}{2}\|x-y\|^2-f^\epsilon_{\mu\rightarrow \nu}(x)\right\}\leq g^\epsilon_{\mu\rightarrow \nu}(y) \leq \frac{1}{2}\int \|x-y\|^2d\mu(x),
\end{equation*}
for all $x,y\in \mathbb{R}^d.$
\end{lemma}
The lower bounds in Lemma \ref{lem:4.9 from nutz}, while general, can be difficult to apply. We derive more straightforward bounds in the special case when one of the measures is subgaussian. The lemma below can be extracted from Proposition A.1 in \citep{mena2019}.

\begin{lemma}\label{lemma:lowerbound}
Let $\mu\in \mathcal{P}_2(\mathbb{R}^d)$, $\nu\in \mathcal{G}_\sigma(\mathbb{R}^d)$, and choose entropic potentials that solve (\ref{eqn:dualform}) such that $\int f^\epsilon_{\mu\rightarrow \nu}d\mu\geq 0$ and $\int g^\epsilon_{\mu\rightarrow \nu}d\nu\geq 0$. Then there exists a constant $C_{\epsilon,\sigma,\|\mathbb{E}(\mu)\|,M_2(\mu),\|\mathbb{E}(\nu)\|}\geq 0$, depending quadratically on $\sigma$, $\|\mathbb{E}(\mu)\|$, and linearly on $ {M}_2(\mu),\frac{1}{\epsilon}$ and $\|\mathbb{E}(\nu)\|$ such that for all $x\in \mathbb{R}^d$:
\begin{equation*}
-C_{\epsilon,\sigma,\|\mathbb{E}(\mu)\|,M_2(\mu),\|\mathbb{E}(\nu)\|}\left(\|x\|^2+1\right)\leq f^\epsilon_{\mu\rightarrow \nu}(x).
\end{equation*}

\end{lemma}
\begin{proof}
By (\ref{eqn:fduality}), we may write $-f^\epsilon_{\mu\rightarrow \nu}(x)=\displaystyle\epsilon\log\int\exp\left(\frac{1}{\epsilon}(g^\epsilon_{\mu\rightarrow \nu}(y)-\frac{1}{2}\|x-y\|^2)\right)d\nu(y)$ for all $x\in\mathbb{R}^{d}$. We then bound:
\begin{align}
&-f^\epsilon_{\mu\rightarrow \nu}(x)\notag \\
= &\epsilon\log \int \exp\left(\frac{1}{\epsilon}\left(g^\epsilon_{\mu\rightarrow \nu}(y)-\frac{1}{2}\|x-y\|^2\right)\right)d\nu(y) \notag\\
 \leq &\epsilon\log \int \exp\left(\frac{1}{\epsilon}\left(\frac{1}{2}\int\|z-y\|^2d\mu(z)-\frac{1}{2}\|x-y\|^2\right)\right)d\nu(y)\label{eqn:lowerbound1} \\
  = &\epsilon\log \int \exp\left(\frac{1}{\epsilon}\left(\frac{1}{2}M_2(\mu)-\langle y,\mathbb{E}(\mu)\rangle+\frac{1}{2}\|y\|^2-\frac{1}{2}\|x\|^2+\langle x,y\rangle-\frac{1}{2}\|y\|^2\right)\right)d\nu(y)\label{eqn:lowerbound2} \\ 
   =&\epsilon\log \int \exp\left(\frac{1}{\epsilon}\left(\frac{1}{2}M_2(\mu)-\langle y,\mathbb{E}(\mu)-x\rangle-\frac{1}{2}\|x\|^2\right)\right)d\nu(y) \notag\\
   =&\frac{1}{2}M_2(\mu)-\frac{1}{2}\|x\|^2+\epsilon\log\int \exp\left(\frac{1}{\epsilon}\left(-\langle y,\mathbb{E}(\mu)-x\rangle\right)\right) d\nu(y)\notag \\ 
    \leq &\frac{1}{2}M_2(\mu)-\frac{1}{2}\|x\|^2+\epsilon\log \int \exp\left(\frac{1}{\epsilon}\left|\langle y, \mathbb{E}(\mu)-x\right\rangle|\right)d\nu(y) \notag\\ 
     \leq &\frac{1}{2}M_2(\mu)-\frac{1}{2}\|x\|^2+\epsilon\log \exp\left(\frac{4\tilde{C}_G^2\sigma^2}{\epsilon^2}\|x-\mathbb{E}(\mu)\|^2+\frac{1}{\epsilon}\|x-\mathbb{E}(\mu)\|\|\mathbb{E}(\nu)\|\right) \label{eqn:lowerbound3}\\ 
     =&\frac{1}{2}M_2(\mu)-\frac{1}{2}\|x\|^2+\frac{4\tilde{C}_G^2\sigma^2}{\epsilon}\|x-\mathbb{E}(\mu)\|^2+\|x-\mathbb{E}(\mu)\|\|\mathbb{E}(\nu)\| \notag\\  \leq &\frac{1}{2}M_2(\mu)-\frac{1}{2}\|x\|^2+\frac{8\tilde{C}_G^2\sigma^2}{\epsilon}\|x\|^{2}+\frac{8\tilde{C}_G^2\sigma^{2}}{\epsilon}\|\mathbb{E}(\mu)\|^2+\|\mathbb{E}(\nu)\|\left(\|x\|+ \|\mathbb{E}(\mu)\| \right)\label{eqn:lowerbound4}\\
     \leq &\frac{1}{2}M_2(\mu)-\frac{1}{2}\|x\|^2+\frac{8\tilde{C}_G^2\sigma^2}{\epsilon}\|x\|^{2}+\frac{8\tilde{C}_G^2\sigma^{2}}{\epsilon}\|\mathbb{E}(\mu)\|^2+\|\mathbb{E}(\nu)\|\left(1+\|\mathbb{E}(\mu)\|\right)+\|\mathbb{E}(\nu)\|\|x\|^2,\notag
\end{align}
where in (\ref{eqn:lowerbound1}) we applied Lemma \ref{lem:4.9 from nutz} to $g_{\mu\rightarrow\nu}^\epsilon$, in (\ref{eqn:lowerbound2}) we expanded $\|.\|^2$, in (\ref{eqn:lowerbound3}) we applied Corollary \ref{cor:cor_of_lem1.5_rig}, and in (\ref{eqn:lowerbound4}) we applied Young's inequality and the triangle inequality. Combining various constants and multiplying both sides by $-1$, we obtain our result.
\end{proof}
We combine Lemma \ref{lem:4.9 from nutz} and Lemma \ref{lemma:lowerbound} to give a simple bound on the absolute value of $f^\epsilon_{\mu\rightarrow \nu}$ with $\nu\in \mathcal{G}(\mathbb{R}^d)$:
\begin{corollary}\label{cor:absolute_bounds}
Let $\mu\in \mathcal{P}_2(\mathbb{R}^d)$, let $\nu\in \mathcal{G}_\sigma(\mathbb{R}^d)$, and choose entropic potentials that solve (\ref{eqn:dualform}) such that $\int f^\epsilon_{\mu\rightarrow \nu}d\mu\geq 0$ and $\int g^\epsilon_{\mu\rightarrow \nu}d\nu\geq 0$. Then there exists a constant $\tilde{C}_{\epsilon,\sigma,\|\mathbb{E}(\mu)\|,M_2(\mu),M_2(\nu)} >0$ only depending quadratically on $\sigma,\|\mathbb{E}(\mu)\|$ and linearly on $\frac{1}{\epsilon},\|\mathbb{E}(\nu)\|,M_2(\mu),M_2(\nu)$, such that:
\begin{equation*}
|f^\epsilon_{\mu\rightarrow \nu}(x)|\leq \tilde{C}_{\epsilon,\sigma,\|\mathbb{E}(\mu)\|,\|\mathbb{E}(\nu)\|,M_2(\mu),M_2(\nu)} (\|x\|^2+1)
\end{equation*}
\end{corollary}
\begin{proof}
From Lemma \ref{lem:4.9 from nutz}, we have \begin{align*}f^\epsilon_{\mu\rightarrow \nu}(x) & \leq \frac{1}{2}\int \|x-y\|^2d\nu(y) \\ & \leq \|x\|^2+\int \|y\|^2 d\nu(y) \\ & =\|x\|^2+M_2(\nu),\end{align*}
where we applied Young's inequality. From this and Lemma \ref{lemma:lowerbound}, we have 
\begin{equation*}
 |f^\epsilon_{\mu\rightarrow \nu}(x)|\leq \max\{C_{\epsilon,\sigma,\|\mathbb{E}(\mu)\|,M_2(\mu),\|\mathbb{E}(\nu)\|}\left(\|x\|^2+1\right),\; \|x\|^2+M_2(\nu)\},
\end{equation*}
where $C_{\epsilon,\sigma,\|\mathbb{E}(\mu)\|,M_2(\mu)}$ depends quadratically on $\sigma,\|\mathbb{E}(\mu)\|$ and linearly on $\frac{1}{\epsilon}, M_2(\mu)$ and $\|\mathbb{E}(\nu)\|$. We conclude by combining constants.
\end{proof}

We now prove the differentiability of $OT_2^\epsilon(\mu,\nu)$ as a function of $\mu$. A crucial step is establishing the continuity of the entropic potentials along perturbations of $\mu$:
\begin{proposition}\label{prop:continuous_subgauss_target}
Let $\mu,\rho\in \mathcal{P}_2(\mathbb{R}^d)$ and $\nu\in\mathcal{G}(\mathbb{R}^d)$, let $\chi:=\rho-\mu$, and let $\mu_t:=\mu+t\chi.$ For all $t\in(0,1]$, let $f^\epsilon_{\mu_t\rightarrow \nu}$ be the unique entropic potential such that $\int f^\epsilon_{\mu_t\rightarrow \nu} d\mu_t=\frac{1}{2}OT^\epsilon_2(\mu_t,\nu)$. Then $f_{\mu_t\rightarrow \nu}^\epsilon$ converges pointwise to $f^\epsilon_{\mu\rightarrow \nu}$ with $\int f_{\mu\rightarrow\nu}^\epsilon d\mu=\frac{1}{2}OT_2^\epsilon(\mu,\nu)$ as $t\rightarrow 0^{+}.$
\end{proposition}

We defer its proof to Subsection \ref{subsec: proof of continuity}.

\vspace{10pt}
\noindent\textbf{Proof of Theorem \ref{thm:entot_differentiable}:} For convenience, set $\epsilon=1$.  Let $\rho\in\mathcal{P}_2(\mathbb{R}^d)$, $\chi$ and $\mu_t$ be as in the statement of Proposition \ref{prop:continuous_subgauss_target}.  We write $(f_t,g_t):=(f^\epsilon_{\mu_t\rightarrow \nu},g^\epsilon_{\mu_t\rightarrow\nu})$ and $(f,g):=(f^\epsilon_{\mu\rightarrow \nu},g^\epsilon_{\mu\rightarrow \nu})$. Let $\Delta_t:=\frac{1}{t}\left(OT^\epsilon_2(\mu_t,\nu)-OT^\epsilon_2(\mu,\nu)\right)$. By suboptimality of $(f,g)$ for the dual formulation of $OT^\epsilon_2(\mu_t,\nu)$ and of $(f_{t},g_{t})$ for $OT^\epsilon_2(\mu,\nu)$, we have the inequalities:

\begin{align}
OT^\epsilon_2(\mu_t,\nu)&\geq \int f d\mu_t +\int gd\nu - \int\int \left(\exp\left(f(x)+g(y)-\frac{1}{2}\|x-y\|^2\right)-1\right)d\mu_t(x)d\nu(y),\label{eqn:finaldiff_1}
\end{align}
\begin{align}
OT^\epsilon_2(\mu,\nu)&\geq \int f_t d\mu +\int g_{t}d\nu - \int\int \left(\exp\left(f_t(x)+g_t(y)-\frac{1}{2}\|x-y\|^2\right)-1\right)d\mu(x)d\nu(y).\label{eqn:finaldiff_2}
\end{align}
Replacing $OT_2^\epsilon(\mu_t,\nu)$ with the lower bound in (\ref{eqn:finaldiff_1}), we obtain a lower bound on $\Delta_t:$
\begin{align}
\Delta_t& =\frac{1}{t}(OT^\epsilon_2(\mu_t,\nu)-OT^\epsilon_2(\mu,\nu))\notag \\ & \geq \frac{1}{t}\bigg(\int f d\mu_t+\int g d\nu-\int \int \left(\exp\bigg(f(x)+g(y)-\frac{1}{2}\|x-y\|^2\right)-1\bigg)d\mu_t(x)d\nu(y)\notag \\& - \int f d\mu -\int g d\nu +\int \int \left(\exp\bigg(f(x)+g(y)-\frac{1}{2}\|x-y\|^2\right)-1\bigg)d\mu(x)d\nu(y)\bigg)\notag \\& = \int f d\chi - \int \int \left(\exp\bigg(f(x)+g(y)-\frac{1}{2}\|x-y\|^2\right)-1\bigg)d\chi(x)d\nu(y).\label{eqn:finaldiff_3}
\end{align}
Similarly, replacing $OT_2^\epsilon(\mu,\nu)$ with the lower bound in (\ref{eqn:finaldiff_2}), we obtain:
\begin{align}
\Delta_t& =\frac{1}{t}(OT^\epsilon_2(\mu_t,\nu)-OT^\epsilon_2(\mu,\nu))\notag \\ & \leq \frac{1}{t}\bigg(\int f_t d\mu_t+\int g_t d\nu-\int \int \left(\exp\bigg(f_t(x)+g_t(y)-\frac{1}{2}\|x-y\|^2\right)-1\bigg)d\mu_t(x)d\nu(y)\notag \\& - \int f_t d\mu -\int g_t d\nu +\int \int \left(\exp\bigg(f_t(x)+g_t(y)-\frac{1}{2}\|x-y\|^2\right)-1\bigg)d\mu(x) d\nu(y)\bigg)\notag \\& = \int f_t d\chi - \int \int \left(\exp\bigg(f_t(x)+g_t(y)-\frac{1}{2}\|x-y\|^2\right)-1\bigg)d\chi(x)d\nu(y).\label{eqn:finaldiff_4}
\end{align}
Subtracting (\ref{eqn:finaldiff_3}) from (\ref{eqn:finaldiff_4}), we obtain:
\begin{align}\label{eqn:finaldiff_5} 
\int (f_t-f)d\chi & -\int\int \exp\left(f_t(x)+g_t(y)-\frac{1}{2}\|x-y\|^2\right)d\chi(x)d\nu(y)\notag \\ & +\int\int \exp\left(f(x)+g(y)-\frac{1}{2}\|x-y\|^2\right)d\chi(x)d\nu(y),
\end{align}
which by construction is an upper bound on $\displaystyle\limsup_{t\rightarrow 0^{+}}\Delta_t-\displaystyle \liminf_{t\rightarrow 0^{+}}\Delta_t$, and hence if we show (\ref{eqn:finaldiff_5}) converges to 0 then we will have established existence of $\lim_{t\rightarrow 0^{+}}\Delta_t$. We now show that the first term converges to zero. For any $x\in \mathbb{R}^d$, we bound:
\begin{align*}
|f_t(x)-f(x)| \leq |f_t(x)|+|f(x)| \leq C_t^*(\|x\|^2+1),
\end{align*}
where we applied the triangle inequality followed by Corollary \ref{cor:absolute_bounds} and where $C^*_t$ is a constant quadratically depending on $\sigma,\|\mathbb{E}(\mu_t)\|,\|\mathbb{E}(\mu)\|$ and linearly depending on $\frac{1}{\epsilon},\|\mathbb{E}(\nu)\|, M_2(\mu_t),M_2(\mu), M_2(\nu)$. As we have $\|\mathbb{E}(\mu_t)\|\leq \|\mathbb{E}(\mu)\|+\|\mathbb{E}(\rho)\|$ and $M_2(\mu_t)\leq M_2(\mu)+M_2(\rho)$, we can further upper bound $\sup_{t\in [0,1]}C_t^*\leq C^*$, where $C^*$ depends only on $\frac{1}{\epsilon},\sigma,\|\mathbb{E}(\mu)\|,\|\mathbb{E}(\nu)\|,\|\mathbb{E}(\rho)\|,M_2(\mu),M_2(\rho)$, and $M_2(\nu)$, and is independent of $t$ and $x$. Since $\mu,\rho\in \mathcal{P}_2(\mathbb{R}^d)$, $C^*(\|x\|^2+1)$ is $\chi$-integrable as a function of $x$. Hence by the Dominated Convergence Theorem and Proposition \ref{prop:continuous_subgauss_target}, we have that $\lim_{t\rightarrow 0^+}\int |f_t -f |d\chi=\int \lim_{t\rightarrow 0^+}|f_t -f |d\chi= 0$. 

Next, we claim that \begin{equation*} \int\int \exp\left(f_t(x)+g_t(y)-\frac{1}{2}\|x-y\|^2\right)d\chi(x)d\nu(y) =0\end{equation*} for all $t\in (0,1]$.  Applying the Fubini-Tonelli Theorem (Proposition 5.2.1 in \citep{cohn2013measure}), which applies since the integrand is nonnegative and measurable, we have, for all $t\in(0,1]$:
 \begin{align}
 &\int\int \exp\left(f_t(x)+g_t(y)-\frac{1}{2}\|x-y\|^2\right)d\chi(x)d\nu(y)\notag \\
 = & \int \exp(f_t(x))\int \exp\left(g_t(y)-\frac{1}{2}\|x-y\|^2\right)d\nu(y)d\chi(x)\notag \\ 
 = &\int \exp\left(f_t(x)\right)\exp\left(-f_t(x)\right) d\chi(x)\notag\\ =&\int 1d\chi(x) = 0,\label{eqn:nogap_1}
 \end{align} 
 where we applied (\ref{eqn:fduality}) and (\ref{eqn:gduality}) and the fact that $\chi=\rho-\mu$ has total mass zero. By the same reasoning, 
 \begin{align}&\int\int \exp\left(f(x)+g(y)-\frac{1}{2}\|x-y\|^2\right)d\chi(x)d\nu(y)\notag \\ =&\int \exp(f(x))\int \exp\left(g(y)-\frac{1}{2}\|x-y\|^2\right)d\nu(y)\notag \\  = &\int \exp(f(x))\exp(-f(x))d\chi(x)\notag \\=&\int 1 d\chi(x)=0.\label{eqn:nogap_2}
 \end{align}

By (\ref{eqn:nogap_1}) and (\ref{eqn:nogap_2}), we have established that (\ref{eqn:finaldiff_5}) equals $\int (f_t-f)d\chi$ for all $t\in (0,1]$, and hence converges to zero as $t\rightarrow 0^+$ as desired. This establishes that $\lim_{t\rightarrow 0^+}\Delta_t$ exists. and we compute this limit using (\ref{eqn:finaldiff_3}) and (\ref{eqn:gduality}) which reduces to $\int f d\chi$, by (\ref{eqn:nogap_1}), and  we conclude that $\lim_{t\rightarrow 0^{+}}\Delta_t= \int f d\chi$.  Thus we have established that $f$ is a derivative for $OT_2^\epsilon(\mu,\nu)$ as a function of $\mu$. \qed 

\begin{remark}
The above proof hinges on the fact that the Sinkhorn relation (\ref{eqn:fduality}) holds for $f^\epsilon_{\mu\rightarrow \nu}$ for all $x\in \mathbb{R}^d$ (and the same for $f^\epsilon_{\mu_{t}\rightarrow \nu}$). In particular, this is crucial in establishing (\ref{eqn:nogap_2}). Without (\ref{eqn:fduality}), one would need to argue that 
\begin{align*}
\int \int &\exp\left(f_t(x)+g_t(y)-\frac{1}{2}\|x-y\|^2\right)d\chi(x)d\nu(y)\\ \rightarrow \int \int&\exp\left(f(x)+g(y)-\frac{1}{2}\|x-y\|^2\right)d\chi(x)d\nu(y)
\end{align*}
as $t\rightarrow 0^+$ directly, which is challenging as $\exp(f_t(x)-\frac{1}{2}\|x-y\|^2)$ is not a priori uniformly bounded over $t>0$ by a $\chi$-integrable function. Hence it is not obvious whether a naive application of the Dominated Convergence Theorem will allow for the verification of this limit.
\end{remark}

\subsection{Proof of Proposition \ref{prop:continuous_subgauss_target}}\label{subsec: proof of continuity}

To establish Proposition \ref{prop:continuous_subgauss_target}, we will need to show that the sequence of potentials converges in probability:

\begin{proposition}\label{prop:convergence_in_prob}
 Let $\mu,\rho,\nu\in\mathcal{P}_2(\mathbb{R}^d)$, let $\chi:=\rho-\mu$ and let $\mu_t=\mu+t\chi$. For all $t\in [0,1)$, let $(f^\epsilon_{\mu_t\rightarrow \nu},g^\epsilon_{\mu_t\rightarrow\nu})$ be the unique entropic potentials such that $\int f^\epsilon_{\mu_t\rightarrow \nu} d\mu_t=\int g^\epsilon_{\mu_t\rightarrow\nu}d\nu=\frac{1}{2}OT^\epsilon_2(\mu_t,\nu)$. Then:
\begin{itemize}[leftmargin=*]
\item $f^\epsilon_{\mu_t\rightarrow \nu}$ converges in $\mu$-probability as $t\rightarrow 0^{+}$ to the entropic potential $f^\epsilon_{\mu\rightarrow \nu}$ with $\int f^\epsilon_{\mu\rightarrow \nu}d\mu=\frac{1}{2} OT^\epsilon_2(\mu,\nu)$;
\item $g^\epsilon_{\mu_t\rightarrow\nu}$ converges in $\nu$-probability as $t\rightarrow 0^{+}$ to the entropic potential $g^\epsilon_{\mu\rightarrow\nu}$ with $\int g^\epsilon_{\mu\rightarrow\nu}d\nu=\frac{1}{2}OT_2^\epsilon(\mu,\nu)$.
\end{itemize}
\end{proposition}

\begin{remark}
We note that the normalization $\int f^\epsilon_{\mu_t\rightarrow \nu} d\mu_t=\int g^\epsilon_{\mu_t\rightarrow\nu}d\nu=\frac{1}{2}OT^\epsilon_2(\mu_t,\nu)\geq 0$ is convenient as it allows us to apply Lemma \ref{lem:4.9 from nutz}, and also varies smoothly as a function of $t$ (see e.g., Lemma \ref{lem:theorem2.1nutz}). 
\end{remark}
We postpone the proof of Proposition \ref{prop:convergence_in_prob} to Subsection \ref{subsec:proof of converge in prob}, and continue with the proof of Proposition \ref{prop:continuous_subgauss_target}.
\bigskip

\noindent\textbf{Proof of Proposition \ref{prop:continuous_subgauss_target}:}  For convenience, we will assume that $\epsilon=1$. Let $f_t:=f^\epsilon_{\mu_t\rightarrow \nu}$ and $g_t:=g^\epsilon_{\mu_t\rightarrow \nu}$ for any $t\in (0,1]$ and let $f:=f_{\mu\rightarrow \nu}^\epsilon$ and $g:=g_{\mu\rightarrow \nu}^\epsilon$ be the unique entropic potentials such that $\frac{1}{2}OT_2^\epsilon(\mu,\nu)=\int g d\nu$. Let $T:=\{t_i\}_{i\in \mathbb{N}}\subset [0,1)$, with $t_i\rightarrow 0$ as $i\rightarrow \infty$. By Proposition \ref{prop:convergence_in_prob}, the sequence $\{g_{t_i}\}_{t_i\in T}$ converges in $\nu$-probability to $g$. Since $\{g_{t_i}\}_{t_i\in T}$ converges in probability, there is a subsequence $T'\subset T$ such that $\{g_{t_i'}\}_{t_i'\in T'}$ converges $\nu$-almost-surely to $g$. We now show that this, together with subgaussian assumption on $\nu$, is enough to prove pointwise convergence everywhere of $\{f_{t_i'}\}_{t_i'\in T'}$ to $f$.  Indeed, by (\ref{eqn:fduality}), for any $x\in \mathbb{R}^d$ we may write: \begin{equation*}\lim_{i\rightarrow\infty} \exp\left(-f_{t_i'}(x)\right) =\lim_{i\rightarrow \infty}\int \exp\left(-\frac{1}{2}\|x-y\|^2+g_{t_i'}(y)\right)d\nu(y).
\end{equation*}
By Lemma \ref{lem:4.9 from nutz}, we may bound uniformly: 
\begin{align}&\exp\left(-\frac{1}{2}\|x-y\|^2+g_{t_i'}(y)\right)\notag\\ \leq & \exp\left(-\frac{1}{2}\|x\|^2+\langle x,y\rangle -\frac{1}{2}\|y\|^2+\int \frac{1}{2}\|z-y\|^2 d\mu_{t_i'}(z)\right)\notag \\ = & \exp\left(-\frac{1}{2}\|x\|^2+\langle x,y\rangle -\frac{1}{2}\|y\|^2+\int \left(\frac{1}{2}\|z\|^2 -\langle y,z\rangle +\frac{1}{2}\|y\|^2\right) d\mu_{t_i'}(z) \right)\notag \\ =& 
\exp\left(-\frac{1}{2}\|x\|^2+\frac{1}{2}M_2(\mu_{t_i'})+\langle x-\mathbb{E}(\mu_{t_i'}),y\rangle\right)\notag \\  
\leq &\exp\left(-\frac{1}{2}\|x\|^2+\frac{1}{2}M_2((1-t_i')\mu+t_i'\rho)+|\langle x,y\rangle|+|\langle \mathbb{E}((1-t_i')\mu+t_i'\rho),y\rangle | \right)\notag \\ 
\leq &\exp\left(-\frac{1}{2}\|x\|^2+\frac{1}{2}M_2(\mu)+\frac{1}{2}M_2(\rho)+|\langle x,y\rangle|+|\langle \mathbb{E}(\mu),y\rangle |+|\langle \mathbb{E}(\rho),y\rangle| \right)\label{eqn:integupperbound}
\end{align}
where we applied the triangle inequality twice to obtain $|\langle x-\mathbb{E}(\mu_t),y\rangle|\leq |\langle x,y\rangle|+|\langle \mathbb{E}(\mu_t),y\rangle|$ and $|\langle \mathbb{E}(\mu_t),y\rangle|\leq |\langle \mathbb{E}(\mu),y\rangle|+|\langle \mathbb{E}(\rho),y\rangle|$, and also used that $0< t_i'\le 1$. We now show that the function in (\ref{eqn:integupperbound}) is $\nu$-integrable in $y$.  By two applications of the Cauchy-Schwarz inequality, we bound:
\begin{align*}
& \int \exp\left(|\langle x,y\rangle|+|\langle \mathbb{E}(\mu),y\rangle|+|\langle \mathbb{E}(\rho),y\rangle|\right)d\nu(y) \\ 
 = & \int \exp\left(|\langle x,y\rangle|\right)\exp\left(|\langle \mathbb{E}(\mu),y\rangle|\right)\exp\left(|\langle \mathbb{E}(\rho),y\rangle|\right)d\nu(y)\\ 
\leq & \left(\int \exp\left(2|\langle x,y\rangle|\right)d\nu(y)\right)^{1/2} \left(\int \exp\left(4|\langle \mathbb{E}(\mu),y\rangle|\right)d\nu(y) \int \exp\left(4|\langle \mathbb{E}(\rho),y\rangle|\right)d\nu(y)\right)^{1/4}
\end{align*}
and by Corollary \ref{cor:cor_of_lem1.5_rig} this last expression is finite. Hence (\ref{eqn:integupperbound}) is $\nu$-integrable, and we conclude that for each $x\in \mathbb{R}^d$, $\exp\left(-\frac{1}{2}\|x-y\|^2+g_{t_i'}(y)\right)$ may be uniformly (in $t_i'$) upper bounded by a $\nu$-integrable function of $y$. By the Dominated Convergence Theorem, we have for all $x\in \mathbb{R}^d$:
\begin{align*}
\lim_{i\rightarrow\infty}\exp\left(-f_{t_i'}(x)\right) & = \int \lim_{i\rightarrow \infty}\exp\left(-\frac{1}{2}\|x-y\|^2 +g_{t_i'}(y)\right)d\nu(y) & \\ & = \int \exp\left(-\frac{1}{2}\|x-y\|^2 +g(y)\right)d\nu(y)\\
&= \exp(-f(x)
)\end{align*}
by $\nu$-almost sure convergence of $g_{t_i'}$ to $g$, which, along with continuity of $\log$, establishes pointwise everywhere convergence of $\{f_{t_i'}\}_{t_i'\in T'}.$

We claim this is enough to conclude pointwise everywhere convergence of $\{f_{t_i}\}_{t_i\in T}$ to $f$. Indeed, suppose $\{f_{t_i}\}_{t_i\in T}$ does not converge pointwise everywhere to $f$. Then there exists an $x_0\in \mathbb{R}^d$, a $\delta>0$, and a subsequence $\tilde{T}\subset T$ with $\tilde{t}_i\rightarrow 0$ as $i\rightarrow \infty$, such that $\|f(x_0)-f_{\tilde{t}_i}(x_0)\|>\delta$ for all $\tilde{t}_{i}\in \tilde{T}$.  But $\{g_{\tilde{t}_{i}}\}_{\tilde{t}_{i}\in\tilde{T}}$ converges in $\nu$-probability to $g$, and thus the sequence $\tilde{T}$ contains a subsequence $\tilde{T}'\subset\tilde{T}$ along which $\{g_{\tilde{t}_{i}'}\}_{\tilde{t}_{i}'\in\tilde{T}'}$ converges $\nu$-almost surely to $g$.  Hence we may apply the above argument to conclude that $f_{\tilde{t}_i'}$ converges pointwise to $f$, contradicting our assumption. We conclude that $\{f_{t_{i}}\}_{t_{i}\in T}$ converges pointwise to $f$ as $i\rightarrow\infty$. 
\qed

\subsection{Proof of Proposition \ref{prop:convergence_in_prob}}\label{subsec:proof of converge in prob}

To prove Proposition \ref{prop:convergence_in_prob}, we will first need to verify that our sequence $\mu_t$ contains a subsequence satisfying the conditions in the lemma below:
\begin{lemma}\label{lem:corollary2.4nutz}
(Corollary 2.4 in \citep{nutz2023stability}) Let $T:=\{t_i\}_{i\in\mathbb{N}}\subset [0,1)$ be any sequence such that $t_i\rightarrow 0^+$ as $i\rightarrow \infty.$ Let $\{\mu_{t_i}\}_{t_i\in T}\subset \mathcal{P}_2(\mathbb{R}^d)$ be a sequence of probability measures such that:
\begin{enumerate}[leftmargin=*]
    \item $\mu_{t_i}$ converges weakly to $\mu$ as $i\rightarrow \infty$; 
    \item $\mu\ll\mu_{t_i}$ for all $t_i\in [0,1)$;
    \item $OT_2^\epsilon(\mu_{t_i},\nu)<\infty$  for all $t_i\in [0,1]$;
    \item $\sup_i \int \max\{f_{\mu_{t_i}\rightarrow \nu}^\epsilon ,0\}d\mu_{t_i}<\infty$, $\sup_i \int \max\{g_{\mu_{t_i}\rightarrow \nu}^\epsilon,0\} d\nu<\infty$;
    \item $\displaystyle\lim_{i\rightarrow \infty}\int \arctan(f^\epsilon_{\mu_{t_i}\rightarrow \nu})d\mu_{t_i}=L$ for some $L\in  [-\pi/2,\pi/2]$;
    \item $\displaystyle\limsup_{C\rightarrow\infty}\limsup_{i\rightarrow \infty}\int \mathbbm{1}_{\frac{d\mu}{d\mu_{t_i}}(x)\geq C}d\mu(x) =0.$
\end{enumerate} 
Then:
\begin{itemize}[leftmargin=*]
\item $f_{\mu_{t_i}\rightarrow\nu}^\epsilon$ converges in $\mu$-probability to $f^{\epsilon,L}_{\mu\rightarrow\nu}$, the unique entropic potential with $\int \arctan(f^{\epsilon,L}_{\mu\rightarrow\nu})d\mu=L$;
\item$g^\epsilon_{\mu_{t_i}\rightarrow\nu}$ converges in $\nu$-probability to $g^{\epsilon,L}_{\mu\rightarrow\nu}$, such that $g^{\epsilon,L}_{\mu\rightarrow\nu}$ is the unique entropic potential paired with $f^{\epsilon,L}_{\mu\rightarrow \nu}$.
\end{itemize}
\end{lemma}

We remark that, for any $f\in L^1(\mu)$, $\int \arctan(f+a)d\mu$ is a strictly increasing function of $a\in \mathbb{R}$. Hence if there are two entropic potentials $f^{\epsilon}_{\mu\rightarrow\nu},\tilde{f}^\epsilon_{\mu\rightarrow\nu}$ such that:
\begin{itemize}[leftmargin=*]
\item $f^\epsilon_{\mu\rightarrow\nu}+c=\tilde{f}^\epsilon_{\mu\rightarrow \nu}$ for some $c\in \mathbb{R}$;
\item $\int \arctan(f^\epsilon_{\mu\rightarrow\nu})d\mu=\int\arctan(\tilde{f}^\epsilon_{\mu\rightarrow\nu})d\mu$;
\end{itemize}
then $c=0$, and consequently $f^{\epsilon}_{\mu\rightarrow\nu}=\tilde{f}^\epsilon_{\mu\rightarrow\nu}$. This justifies the uniqueness claimed for $f^{\epsilon,L}_{\mu\rightarrow\nu}$ (and hence for $g^{\epsilon,L}_{\mu\rightarrow\nu}$) in Lemma \ref{lem:corollary2.4nutz}.

To verify item 6 we will apply the following Lemma:

\begin{lemma}\label{lem:lemma2.5nutz}
(Lemma 2.5 in \citep{nutz2023stability}) Let $T:=\{t_i\}_{i\in\mathbb{N}}\subset [0,1)$ be any sequence such that $t_i\rightarrow 0^+$ as $i\rightarrow \infty$, and let $\{\mu_{t_i}\}_{t_i\in T}\subset \mathcal{P}_2(\mathbb{R}^d)$ be a sequence of probability measures such that $\mu \ll \mu_{t_i}$ for all $t_{i}\in T$. Suppose that $\mu_{t_i}\rightarrow \mu$ in total variation as $i\rightarrow\infty$. Then:
\begin{equation*}\label{eqn:2.13innutz}
\limsup_{C\rightarrow\infty}\limsup_{i\rightarrow \infty}\int \mathbbm{1}_{\frac{d\mu}{d\mu_{t_i}}(x)\geq C}d\mu(x) =0.
\end{equation*}
\end{lemma}

After verifying the conditions in Lemma \ref{lem:corollary2.4nutz}, we will need to establish that the limiting potentials satisfy $\int f^{\epsilon, L}_{\mu\rightarrow \nu}d\mu=\frac{1}{2}OT_2^\epsilon(\mu,\nu)$, $\int g^{\epsilon, L}_{\mu\rightarrow \nu}d\mu=\frac{1}{2}OT_2^\epsilon(\mu,\nu)$. The result immediately follows from the Lemma below under the stated assumption in Proposition \ref{prop:convergence_in_prob} that $\int f_t d\mu_t = \int g_t d\nu = \frac{1}{2} OT_2^{\epsilon}(\mu_t, \nu)$.
\begin{lemma}
\label{lem:theorem2.1nutz}
(Theorem 2.1 (iii) in \citep{nutz2023stability}) Let $T:=\{t_i\}_{i\in\mathbb{N}}\subset [0,1)$ be any sequence such that $t_i\rightarrow 0^+$ as $i\rightarrow \infty$, and let $\{\mu_{t_i}\}_{t_i\in T}\subset \mathcal{P}_2(\mathbb{R}^d)$ be a sequence of probability measures, converging weakly to $\mu\in \mathcal{P}_2(\mathbb{R}^d)$, and let $\nu\in \mathcal{P}_2(\mathbb{R}^d).$ 
 Let $(f_{t_i},g_{t_i}):=(f^\epsilon_{\mu_{t_i}\rightarrow \nu},g^\epsilon_{\mu_{t_i}\rightarrow \nu})$ be solutions to (\ref{eqn:dualform}), and suppose that: 

\begin{enumerate}[leftmargin=*]
    \item $\sup_{i}\{\int f_{t_i}^\epsilon d\mu_{t_i},\int g_{t_i}^\epsilon d\mu_{t_i}\}<\infty$.
    \item $\{(f_{t_i},g_{t_i})\}_{t_i\in T}$ are uniformly integrable with respect to $(\mu_{t_{i}},\nu)$: \begin{equation*}
\lim_{C\rightarrow \infty}\sup_{i\in \mathbb{N}}\int f_{t_i} (x)\mathbbm{1}_{f_{t_i}>C}(x)d\mu_{t_i}(x)= 0,\end{equation*}
\begin{equation*} \lim_{C\rightarrow \infty}\sup_{i\in \mathbb{N}}\int g_{t_i}(y) \mathbbm{1}_{g_{t_i}>C}(y)d\nu(y) =0.\end{equation*}
\item $\displaystyle\lim_{i\rightarrow \infty}\int \arctan(f_{t_i})d\mu_{t_i}=\int \arctan(f) d\mu$.
\end{enumerate} 

Then:
\begin{itemize}[leftmargin=*]
    \item $OT^\epsilon_2(\mu_{t_i},\nu)\rightarrow OT^\epsilon_{2}(\mu,\nu)$ as $i\rightarrow \infty.$
    \item $\int f_{t_i}d\mu_{t_i}\rightarrow \int f d\mu$ as $i\rightarrow \infty$.
    \item $\int g_{t_i}d\nu\rightarrow \int g d\nu$ as $i\rightarrow \infty.$
\end{itemize}
\end{lemma}

To verify item 2 in Lemma \ref{lem:theorem2.1nutz}, we will require the following result:

\begin{lemma}\label{lem:uniform_integrable}
Let $\mu,\rho,\nu\in \mathcal{P}_2(\mathbb{R}^d)$. We let $\chi=\rho-\mu$, and define a sequence of probability measures $\mu_{t_i}:=\mu+t_i\chi$ where $\{t_i\}_{i=1}^\infty=:T\subset [0,1]$ is a sequence of nonnegative real numbers such that $t_i\rightarrow 0^{+}$ as $i\rightarrow \infty$. Then the sequence of potentials $\{(f^\epsilon_{\mu_{t_i}\rightarrow \nu},g^\epsilon_{\mu_{t_i}\rightarrow \nu})\}_{i=1}^\infty$ satisfies:

\begin{equation}\label{eqn:uniform_integ_f}
\lim_{C\rightarrow \infty}\sup_{i \in \mathbb{N}}\int f_{\mu_{t_i}\rightarrow \nu}^\epsilon (x)\mathbbm{1}_{f^\epsilon_{\mu_{t_i}\rightarrow \nu}>C}(x)d\mu_{t_i}(x)= 0,\end{equation}
\begin{equation}\label{eqn:uniform_integ_g} \lim_{C\rightarrow \infty}\sup_{i\in \mathbb{N}}\int g_{\mu_{t_i}\rightarrow \nu}^\epsilon(y) \mathbbm{1}_{g^\epsilon_{\mu_{t_i}\rightarrow \nu}>C}(y)d\nu(y) =0.
\end{equation}
\end{lemma}

\begin{proof}
Without loss of generality, let $\epsilon=1$.  We bound:
\begin{align*}
& \int f_{\mu_{t_i}\rightarrow \nu}^\epsilon (x)\mathbbm{1}_{f^\epsilon_{\mu_{t_i}\rightarrow \nu}>C}(x)d\mu_{t_i}(x) \\
\leq &\int\left(\int \frac{1}{2}\|x-y\|^2d\nu(y)\right)\mathbbm{1}_{f^\epsilon_{\mu_{t_i}\rightarrow \nu}>C}(x)d\mu_{t_i}(x)\\\leq &\int\left(\int \|x\|^2+\|y\|^2d\nu(y)\right)\mathbbm{1}_{f^\epsilon_{\mu_{t_i}\rightarrow \nu}>C}(x)d\mu_{t_i}(x)\\=&  \int(\|x\|^2 + M_2(\nu) )\mathbbm{1}_{f^\epsilon_{\mu_{t_i}\rightarrow \nu}>C}(x)d\mu_{t_i}(x)\end{align*}
where we applied Lemma \ref{lem:4.9 from nutz}, the triangle inequality, and Young's inequality.  We also note that:
\begin{align*}
\mathbbm{1}_{f^\epsilon_{\mu_{t_i}\rightarrow \nu}> C}(x) &  \leq \mathbbm{1}_{\frac{1}{2}\int \|x-y\|^2 d\nu(y)>C}(x) & \\ & \leq \mathbbm{1}_{\|.\|^2+M_2(\nu)>C}(x) & \\ & =\mathbbm{1}_{\|.\|^2>C-M_2(\nu)}(x)
\end{align*}
where we applied Lemma \ref{lem:4.9 from nutz}, the triangle inequality, and Young's inequality. As $2\|x\|^2\geq\|x\|^2+M_2(\nu)$ for all $\|x\|^2\geq M_2(\nu)$ we have:
\begin{align}
& \lim_{C\rightarrow \infty}\sup_{i\in \mathbb{N}}\int \left(\|x\|^2+M_2(\nu)\right)\mathbbm{1}_{\|.\|^2>C-M_2(\nu)}(x)d\mu_{t_i}(x) \notag\\
\leq&\lim_{D\rightarrow \infty}\sup_{i \in \mathbb{N}}2\int \|x\|^2\mathbbm{1}_{\|.\|^2>D}(x)d\mu_{t_i}(x)\notag \\  =& 2\lim_{D\rightarrow \infty} \sup_{i \in \mathbb{N}}  \left((1-t_i)\int ||x||^2 \mathbbm{1}_{||.||^2>D}(x)d\mu(x)+t_i\int ||x||^2 \mathbbm{1}_{||.||^2>D}(x)d\rho(x)\right)\notag \\  \leq &2\lim_{D\rightarrow \infty}  \left(\int ||x||^2 \mathbbm{1}_{||.||^2>D}(x)d\mu(x)+\int ||x||^2 \mathbbm{1}_{||.||^2>D}(x)d\rho(x)\right) \label{eqn:c9_1}
\end{align}
Since $\mu \in\mathcal{P}_2(\mathbb{R}^d)$,
\begin{align*}
M_2(\mu)=\int ||x||^2 \mathbbm{1}_{||.||^2 > D}(x) d\mu + \int ||x||^2 \mathbbm{1}_{||.||^2 \leq  D}(x) d\mu<\infty,
\end{align*}
and by the monotone convergence theorem $\lim_{D\rightarrow \infty} \int ||x||^2 \mathbbm{1}_{||.||^2\leq D}(x)d\mu = M_2(\mu)$, and hence\newline $\limsup_{D\rightarrow \infty} \int ||x||^2\mathbbm{1}_{||.||^2>D}(x)d\mu=0$. Similarly, since $\rho \in \mathcal{P}_2(\mathbb{R}^d)$, $\lim_{D\rightarrow \infty} \int ||x||^2\mathbbm{1}_{||.||^2>D}(x)d\rho=0$, and (\ref{eqn:c9_1}) is zero.

The proof for $g^\epsilon_{\mu_{t_{i}}\rightarrow \nu}$ follows similarly. 
\end{proof}

\noindent\textbf{Proof of Proposition \ref{prop:convergence_in_prob}:} For convenience, we will assume that $\epsilon=1$. Let $T:=\{t_i\}_{i\in\mathbb{N}}\subset [0,1)$ be any sequence such that $t_i\rightarrow 0^+$ as $i\rightarrow \infty.$ We begin by verifying that the conditions in Lemma \ref{lem:corollary2.4nutz} hold along a subsequence $T'\subset T$.  Items 1 and 2 clearly hold for any subsequence of $T$ by definition of $\mu_{t_i}=(1-t_i)\mu+t_i\rho$. 

We may establish item 4 for any subsequence of $T$ via Lemma \ref{lem:4.9 from nutz} and Young's inequality: \begin{align*}\int \max\{f_{\mu_{t_i}\rightarrow \nu}^\epsilon ,0\}d\mu_{t_i}(x) &\leq \int\int \frac{1}{2}\|x-y\|^2d\nu(y) d\mu_{t_i}(x) \\ & \leq \int (\|x\|^2+M_2(\nu)) d\mu_{t_i}(x) \\ & = M_2(\mu_{t_i})+M_2(\nu) \\ & \leq M_2(\mu)+M_2(\rho)+M_2(\nu) \end{align*} and hence $\sup_{i}\int \max\{f_{\mu_{t_i}\rightarrow \nu}^{\epsilon}(x),0\}d\mu_{t_i}(x)<\infty$, and the same holds for $\sup_{i}\int \max\{g_{\mu_{t_i}\rightarrow \nu}(y),0\}d\nu(y)$.  This also verifies item 3, since $\int f^\epsilon_{\mu_{t_i}\rightarrow \nu}d\mu_{t_i}+\int g^\epsilon_{\mu_{t_i}\rightarrow \nu}d\nu=OT_2^\epsilon(\mu_{t_i},\nu)$.  

To see item 5, since $\arctan$ is bounded, $\int \arctan(f^\epsilon_{\mu_{t_i}\rightarrow \nu})d\mu_{t_i}$ is uniformly bounded, and hence there exists a convergent subsequence $T'$ with $\displaystyle\lim_{i\rightarrow\infty}\int\arctan(f^\epsilon_{\mu_{t_i'}\rightarrow \nu})d\mu_{t_i'}=L$ for some $L\in [-\pi/2,\pi/2]$. To verify item 6, the sequence $\mu_{t_i}$ clearly converges to $\mu$ in total variation\footnote{The \emph{total variation} distance between $\mu$ and $\nu$ is
$d_{TV}(\mu,\nu)=\sup_{A}|\mu(A)-\nu(A)|$ where the supremum is over all measurable sets $A$.}, and hence we may apply Lemma \ref{lem:lemma2.5nutz} to conclude that:
\begin{equation*}
\limsup_{C\rightarrow\infty}\limsup_{i\rightarrow \infty}\int \mathbbm{1}_{\frac{d\mu}{d\mu_{t_i}}(x)\geq C}d\mu(x) =0.
\end{equation*}
We have thus established that Lemma \ref{lem:corollary2.4nutz} holds for the subsequence indexed by $T'$, we conclude that $\{f^\epsilon_{\mu_{t_i'}\rightarrow \nu}\}_{t_i'\in T'}$ converges in $\mu$-probability to $f^{\epsilon,L}_{\mu\rightarrow\nu}$, and similarly $\{g^\epsilon_{\mu_{t_i'}\rightarrow \nu}\}_{t_i'\in T'}$ converges in $\nu$-probability to $g^{\epsilon,L}_{\mu\rightarrow\nu}.$

Next, we verify the conditions in Lemma \ref{lem:theorem2.1nutz} for the subsequence $(f_{\mu_{t_i'}\rightarrow \nu}^\epsilon,g_{\mu_{t_i'}\rightarrow \nu}^\epsilon)$.  Indeed, item 1 has already been established in the verification of the conditions of \ref{lem:corollary2.4nutz} above, and by definition of the subsequence, item 3 holds with limit potentials $(f^{\epsilon,L}_{\mu\rightarrow \nu},g^{\epsilon,L}_{\mu\rightarrow \nu}$). By Lemma \ref{lem:uniform_integrable}, this subsequence also satisfies (\ref{eqn:uniform_integ_f}) and (\ref{eqn:uniform_integ_g}), establishing item 2.  Hence we have verified Lemma \ref{lem:theorem2.1nutz} for our subsequence, and can conclude that $\int f^\epsilon_{\mu_{t_i'}\rightarrow \nu}d\mu_{t_i'}\rightarrow \int f^{\epsilon,L}_{\mu\rightarrow \nu} d\mu$ and $\int g^\epsilon_{\mu_{t_i'}\rightarrow \nu}d\nu\rightarrow\int g^{\epsilon,L}_{\mu\rightarrow \nu}d\nu.$ By assumption, we have that $\int f^\epsilon_{\mu_{t_i'}\rightarrow \nu}d\mu_{t_i'}=\int g^\epsilon_{\mu_{t_i'}\rightarrow \nu}d\nu=\frac{1}{2}OT_2^\epsilon(\mu_{t_i'},\nu)$, and again by Lemma \ref{lem:theorem2.1nutz} we have $\lim_{i\rightarrow \infty}OT_2^\epsilon(\mu_{t_i'},\nu)=OT_2^\epsilon(\mu,\nu)$. We conclude that \begin{equation}\label{eqn:fixednormalization} \int f^{\epsilon,L}_{\mu\rightarrow \nu}d\mu=\int g^{\epsilon,L}_{\mu\rightarrow\nu}d\nu=\frac{1}{2}OT_2^\epsilon(\mu,\nu).\end{equation}Hence we have identified $(f^{\epsilon,L}_{\mu\rightarrow\nu},g^{\epsilon,L}_{\mu\rightarrow \nu})=(f^\epsilon_{\mu\rightarrow \nu},g^\epsilon_{\mu\rightarrow \nu})$, the unique pair of potentials satisfying (\ref{eqn:fixednormalization}). 

Since our sequence $T$ was arbitrary, we can conclude as follows.  Let $T=\{t_i\}_{i=1}^\infty$ be a sequence with $t_i\rightarrow 0^{+}$ as $i\rightarrow \infty$, and suppose that $f_{\mu_{t_i}\rightarrow \nu}$ does not converge in $\mu$-probability to $f_{\mu\rightarrow \nu}^\epsilon$. Since convergence in probability is metrizable, there exists a metric $d$ and a $\kappa>0$ such that there exists a subsequence $S$ of $T$ with the property that $d(f_{\mu_{s_i}\rightarrow \nu}^\epsilon,f^\epsilon_{\mu\rightarrow \nu})>\kappa$ for all $s_i\in S$. The above proof also applies to $S$, and hence we may find a subsequence $S'\subset S$ so that $f_{\mu_{s_i'}\rightarrow \nu}$ converges to $f_{\mu\rightarrow \nu}^\epsilon$ in $\mu$-probability, and hence $d(f_{\mu_{s_i'}\rightarrow \nu}^\epsilon,f^\epsilon_{\mu\rightarrow \nu})\rightarrow 0$ as $i\rightarrow \infty$ which is a contradiction. Thus we may conclude that $f_{\mu_{t_i}\rightarrow \nu}^\epsilon$ converges to $f_{\mu\rightarrow \nu}^\epsilon$ in $\mu$-probability. 
The same argument allow us to conclude for $g_{\mu_{t_i}\rightarrow \nu}^\epsilon.$ 
\qed

\subsection{Proof of Proposition \ref{prop:sinkhorn_differentiable}}

For convenience, we set $\epsilon=1$.  Recall that we may write 
\begin{equation*}
f^\epsilon_{\mu\rightarrow \nu}(x)=-\log\left(\int \exp\left(-\frac{1}{2}\|x-y\|^2+g^\epsilon_{\mu\rightarrow \nu}(y)\right)d\nu(y)\right).
\end{equation*}
We define $h(x,y):= \exp\left(-\frac{1}{2}\|x-y\|^2+g^\epsilon_{\mu\rightarrow \nu}(y)\right)$, and will show that $\nabla_x\displaystyle\int h(x,y)d\nu(y)$ exists.  Define
\begin{equation*}
h_{t,i}(x,y):=\frac{1}{t}\left(\exp\left(-\frac{1}{2}\|x+te_i-y\|^2+g^\epsilon_{\mu\rightarrow \nu}(y)\right)-\exp\left(-\frac{1}{2}\|x-y\|^2+g^\epsilon_{\mu\rightarrow \nu}(y)\right)\right)
\end{equation*}
where $e_i$ is a standard basis vector in $\mathbb{R}^d$. We then define $h_t=(h_{t,1},...,h_{t,d})$, and observe that:
\begin{align*}
\lim_{t\rightarrow 0^+}h_t(x,y)& =\nabla_x h(x,y) & \\ & = -(x-y)\exp\left(-\frac{1}{2}\|x-y\|^2+g^\epsilon_{\mu\rightarrow \nu}(y)\right).
\end{align*}
At each $x$, we apply the mean value theorem to write $h_t(x,y)$ as $\|\nabla_x h(\theta_x^t,y)\|$ for some $\theta^t_x\in B_t(x):=\{y\in\mathbb{R}^{d} \ | \ \|x-y\|<t\}$. We upper bound: 
\begin{align}
|h_t(x,y)|& = \|\nabla_x h(\theta^t_x,y)\| \notag \\ 
&  = \left\|(\theta^t_x-y)\exp\left(-\frac{1}{2}\|\theta^t_x-y\|^2+g^\epsilon_{\mu\rightarrow \nu}(y)\right)\right\| \notag \\
& \leq \|\theta^t_x-y\| \; \left|\exp\left(-\frac{1}{2}\|\theta^t_x-y\|^2+\int \frac{1}{2}\|z-y\|^2d\mu(z)\right)\right| \label{eqn:nutzbound} \\ 
& = \|\theta^t_x-y\| \; \left|\exp\left(-\frac{1}{2}\|\theta^t_x\|^2+\frac{1}{2}M_2(\mu)+\langle \theta^t_x-\mathbb{E}(\mu),y\rangle\right)\right|\notag
\end{align}
where in (\ref{eqn:nutzbound}) we applied the bound on $g^\epsilon_{\mu\rightarrow \nu}(y)$ from Lemma \ref{lem:4.9 from nutz}.  Since $\theta^t_x\in B_t(x)$, we may assume WLOG that $\|\theta_x^t-x\|\leq 1$ for all small enough $t\leq 1$. Hence we may uniformly (in $y$ and $t$) bound:
\begin{align*}
|h_t(x,y)| & \leq \sup_{z\in \overline{B_1}(x)}\|z-y\| \; \exp\left(-\frac{1}{2}\|z\|^2+\frac{1}{2}M_2(\mu)+\langle z,y\rangle -\langle \mathbb{E}(\mu),y\rangle\right).
\end{align*}
 We may then bound:
\begin{align}
&\int |h_t(x,y)|d\nu(y)\notag\\ 
\leq &\int \sup_{z\in \overline{B_1}(x)}\|z-y\| \exp\left(-\frac{1}{2}\|z\|^2+\frac{1}{2}M_2(\mu)+\langle z-\mathbb{E}(\mu),y\rangle\right) d\nu(y)\label{eqn:differentiable 1} \\
\leq  &\left(\int \sup_{z\in \overline{B_1}(x)}\|z-y\|^2\exp\left(-\|z\|^2+M_2(\mu)\right) d\nu(y)\right)^{1/2} \left(\int \sup_{z\in \overline{B_1}(x)} \exp(2\langle z-\mathbb{E}(\mu),y\rangle)d\nu(y)\right)^{1/2}\label{eqn:differentiable 2}\\  < &\infty\notag,
\end{align}
where we applied the Cauchy-Schwarz inequality in (\ref{eqn:differentiable 1}) and Corollary \ref{cor:cor_of_lem1.5_rig} to conclude that (\ref{eqn:differentiable 2}) is finite. Thus for each $x\in \mathbb{R}^d$ and all $t\in [0,1)$, $h_t(x,y)$ is uniformly bounded by an integrable function of $y$ in a neighborhood of $x$, and hence we may apply the Dominated Convergence Theorem to conclude that $\displaystyle\lim_{t\rightarrow 0^+}\int h_t(x,y)d\nu(y)= \int \lim_{t\rightarrow 0^+}h_t(x,y)d\nu(y)$.
Equivalently, we have shown:\begin{align*}
\nabla_x\int \exp\left(-\frac{1}{2}\|x-y\|^2+g^\epsilon_{\mu\rightarrow \nu}(y)\right)d\nu(y) & =\int \nabla_x \exp\left(-\frac{1}{2}\|x-y\|^2+g^\epsilon_{\mu\rightarrow \nu}(y)\right)d\nu(y) & \\ & = \int (y-x)\exp\left(-\frac{1}{2}\|x-y\|^2+g^\epsilon_{\mu\rightarrow \nu}(y)\right) d\nu(y)
\end{align*} for all $x\in \mathbb{R}^d$. Since $f^\epsilon_{\mu\rightarrow \nu}(x)=-\displaystyle\log \left(\int h(x,y)d\nu(y)\right)$,  $f^\epsilon_{\mu\rightarrow \nu}(x)$ is a composition of two differentiable functions and is hence differentiable by the chain rule. One then easily verifies (\ref{eqn:gradformula}).  The proof that the higher derivatives of $f^\epsilon_{\mu\rightarrow\nu}$ exist follows similarly. 
 
\qed 

\subsection{Proof of Corollary \ref{cor:first_var for_functionals}}
\begin{proof}
The result immediately follows from Theorem \ref{thm:entot_differentiable} followed by Proposition \ref{prop:sinkhorn_differentiable} and linearity of derivatives.
\end{proof}

\subsection{Proof of Corollary \ref{cor:optimality}}

Let $\mu^\epsilon$ be a critical point for $F^\epsilon_{\lambda,\mathcal{V}}$ with $\texttt{supp}(\mu^\epsilon)=\mathbb{R}^d$. By Theorem \ref{thm:entot_differentiable}, $\delta F^\epsilon_{\lambda,\mathcal{V}}$ is smooth, and hence $\nabla\delta F^\epsilon_{\lambda,\mathcal{V}}(x)=0$ for $\mu^\epsilon$-almost every $x$ implies that $\nabla\delta F^\epsilon_{\lambda,\mathcal{V}}(x)=0$ for all $x\in\mathbb{R}^d$. Since $\mathbb{R}^d$ is connected, this implies that there exists a constant $c$ such that $\delta F^\epsilon_{\lambda,\mathcal{V}}(x)=c$ for all $x\in \mathbb{R}^d$. Thus $\langle \delta F^\epsilon_{\lambda,\mathcal{V}}(\mu^\epsilon),\mu-\mu^\epsilon\rangle=0$ for any $\mu\in \mathcal{P}_2(\mathbb{R}^d)$ satisfying the optimality criteria in Proposition \ref{prop:optimality}. The same proof applies to $S^\epsilon_{\lambda,\mathcal{V}}$. \qed 

\subsection{Proof of Lemma \ref{lemma:diracfixed}}

For each $1\leq j\leq m$ there is a unique coupling between $\delta_{x_\mathcal{V}}$ and $\nu_j$, given by $\delta_{x_{\mathcal{V}}}\otimes \nu_j$. We then compute the entropic map $T^\epsilon_{\delta_{x_\mathcal{V}}\rightarrow\nu_j}(x_\mathcal{V})=\mathbb{E}_{(X,Y)\sim \delta_{x_\mathcal{V}}\otimes\nu_j}(Y|X=x_\mathcal{V})=\mathbb{E}(\nu_{j})$. This holds for each $j$, and hence we may compute $\sum_{j=1}^m\lambda_j T^\epsilon_{\delta_{x_\mathcal{V}}\rightarrow\nu_j}(x_\mathcal{V})=x_\mathcal{V}$. 

Hence $\delta_{x_\mathcal{V}}$ satisfies the fixed point condition in Corollary \ref{cor:first_var for_functionals}, and thus is a critical point.
\qed

\subsection{Proof of Proposition \ref{prop:subgauss}}
The proof is essentially an immediate consequence of Proposition \ref{prop:domination} and Lemma \ref{lemma:subgauss dom} stated below.  In order to state the results, we first recall the following definition of convex ordering and a useful lemma.

\begin{definition}
Two probability measures $\mu,\nu\in\mathcal{P}_2(\mathbb{R}^d)$ are in \emph{convex order}, denoted $\mu\leq_c \nu$, if $\displaystyle\int \phi d\mu \leq \displaystyle\int \phi d\nu$ for all convex $\phi:\mathbb{R}^d\rightarrow \mathbb{R}$.
\end{definition}

\begin{lemma}\label{lemma:subgauss dom}
Let $\nu$ be $\sigma$-subgaussian and $\mu\leq_c \nu$. Then $\mu$ is $\sigma$-subgaussian. 
\end{lemma}
\begin{proof}Let $X\sim \mu$ and $Y\sim\nu$, and let $v\in S^{d-1}$. Then we define the function $f_v(x):=\exp\left(\frac{|\langle x,v\rangle|^2}{\sigma^2}\right)$, which is convex in $x$. Since $\mu\leq_c \nu$, we have $\mathbb{E}_{X\sim \mu}(f_v(X))\leq \mathbb{E}_{Y\sim\nu}(f_v(Y))$ for any $v$, and since $\nu$ is subgaussian, we have that $\sup_{v\in S^{d-1}}\mathbb{E}_{Y\sim\nu}(f_v(Y))\leq 2$. Hence we may conclude that $\sup_{v\in S^{d-1}}\mathbb{E}_{X\sim\mu}(f_v(X))\leq 2$, so that $\mu$ is $\sigma$-subgaussian.
\end{proof}
We note the following result from \citep{yang2024estimating}.
\begin{lemma}\label{lemma:10.5yang}
(Lemma 9.2 in \citep{yang2024estimating}) Let $\mu,\nu\in\mathcal{P}_2(\mathbb{R}^d)$. Then $(T_{\mu\rightarrow \nu}^\epsilon)_{\#}(\mu)\leq_c \nu$.
\end{lemma}

We apply the above lemma to obtain the following result for critical points of $F^\epsilon_{\lambda,\mathcal{V}}$:

\begin{proposition}\label{prop:domination}
Suppose $\{\nu_1,...,\nu_m\}=\mathcal{V}\subset \mathcal{G}_\sigma(\mathbb{R}^d),$ and let $\mu^\epsilon$ be a critical point for $F^\epsilon_{\lambda,\mathcal{V}}$. Then $\mu^\epsilon \leq_c \sum_{j=1}^m\lambda_j\nu_j$.
\end{proposition}
\begin{proof}
Let $\mu^\epsilon$ be a critical point of $F_{\lambda,\mathcal{V}}^\epsilon.$ By Lemma \ref{lemma:10.5yang}, $\tilde{\mu}_{j}:=(T^\epsilon_{\mu^\epsilon\rightarrow\nu_j})_{\#}(\mu^\epsilon)\leq_c\nu_j$ for all $1\leq j\leq m$. As $\mu^\epsilon$ is a critical point, we have: $(\sum_{j=1}^m\lambda_j T^\epsilon_{\mu^\epsilon\rightarrow\nu_j})_{\#}(\mu^\epsilon)=\mu^\epsilon$ by Corollary \ref{lem:SufficientFixed}. Therefore, for any convex function $\phi(x)$:
\begin{align}
\int \phi(x) d\mu^\epsilon(x) & =\int \phi(x) d\left(\sum_{j=1}^m \lambda_j T^\epsilon_{\mu^\epsilon\rightarrow\nu_j}\right)_{\#}(\mu^\epsilon)(x) \notag\\ 
& =\int \phi\left(\sum_{j=1}^m \lambda_j T^\epsilon_{\mu^\epsilon\rightarrow\nu_j}(x)\right) d\mu^\epsilon(x) \notag\\ 
& \leq \int \sum_{j=1}^m\lambda_j\phi(T_{\mu^\epsilon\rightarrow \nu_j}^\epsilon(x))d\mu^\epsilon(x) \label{eqn:Jensens} \\ 
& = \sum_{j=1}^m\lambda_j\int \phi(x) d(T_{\mu^\epsilon\rightarrow \nu_j}^\epsilon)_{\#}(\mu^\epsilon)(x) \notag
\\ & = \sum_{j=1}^m\lambda_j\int \phi(x) d\tilde{\mu}_j(x) \notag & \\ & \leq \sum_{j=1}^m \lambda_j \int \phi(x) d\nu_j(x)\label{eqn:cvx order} &\\
& = \int \phi(x)d\left(\sum_{j=1}^m \lambda_j \nu_j\right)(x), \notag
\end{align}
where we applied Jensen's inequality in (\ref{eqn:Jensens}) and Lemma \ref{lemma:10.5yang} to assert that $\tilde{\mu}_j\leq_c\nu_j$ in (\ref{eqn:cvx order}). Thus $\mu^\epsilon\leq_c \sum_{j=1}^m\lambda_j\nu_j$.

\end{proof}

We conclude the proof of Proposition \ref{prop:subgauss} by observing that\begin{equation*} \left\|\sum_{j=1}^m\lambda_j \nu_j\right\|_\mathcal{G}\leq \max_{1\leq j\leq m}\|\nu_j\|_\mathcal{G}\leq \sigma.\end{equation*}\qed

\section{PROOF OF THEOREM \ref{thm:synthsampcomplex} }\label{subsec:proofofapprox_min}

The proof of Theorem 3.1 rests on the following propositions.

\begin{proposition}\label{prop:approx_min}
Let $\mathcal{V}=\{\nu_1,...,\nu_m\}\subset\mathcal{G}_{\sigma}(\mathbb{R}^d)$ and let $\hat{\mathcal{V}}^n=\{\hat{\nu}^n_1,...,\hat{\nu}^n_m\}.$ Then:
\begin{align}
&\mathbb{E}\left(\bigg|\min_{\mu\in\mathcal{P}_2(\mathbb{R}^d)}F_{\lambda,\mathcal{V}}^\epsilon(\mu)-\min_{\mu\in\mathcal{P}_2(\mathbb{R}^d)}F_{\lambda,\hat{\mathcal{V}}^n}^\epsilon(\mu)\bigg|\right) \leq \frac{m^{1/2}R^*_{d,\epsilon,\sigma}}{\sqrt{n}},\label{eqn:synthsampF_exp}
\\& \bigg|\min_{\mu\in\mathcal{P}_2(\mathbb{R}^d)}F_{\lambda,\mathcal{V}}^\epsilon(\mu)-\min_{\mu\in\mathcal{P}^n(\mathbb{R}^d)}F_{\lambda,\mathcal{V}}^\epsilon(\mu)\bigg| \leq \frac{R^*_{d,\epsilon,\sigma}}{\sqrt{n}}\label{eqn:synthsampF_min},
\\ & \mathbb{E}\left(\bigg|\min_{\mu\in\mathcal{P}_2(\mathbb{R}^d)}F_{\lambda,\mathcal{V}}^\epsilon(\mu)-\min_{\mu\in\mathcal{P}^{n_1}(\mathbb{R}^d)}F_{\lambda,\hat{\mathcal{V}}^{n_2}}^\epsilon(\mu)\bigg|\right) \leq 
 mR^*_{d,\epsilon,\sigma}\left(\frac{1}{\sqrt{n_1}}+\frac{1}{\sqrt{n_2}}\right),\label{eqn:synthsampF_double}
\end{align}
where $R^*_{d,\epsilon,\sigma}$ is a constant depending only on $d,\epsilon$ and $\sigma$. 
\end{proposition}
We note the different dependencies on $m$ appearing in (\ref{eqn:synthsampF_exp}), (\ref{eqn:synthsampF_min}) and (\ref{eqn:synthsampF_double}), which is possibly due to our method of proof. We leave the question of whether the dependence on $m$ can be improved to future work. We give an analogous proposition for the Sinkhorn barycenter functional.

\begin{proposition}\label{prop:alt}
Suppose that, for any $\mathcal{V}\subset \mathcal{G}(\mathbb{R}^d)$,  $U_n:=\argmin_{\mu\in \mathcal{P}^n(\mathbb{R}^d)}S^\epsilon_{\lambda,\mathcal{V}}(\mu)$ is nonempty and furthermore $U_n\subset \mathcal{G}_{\tilde{\sigma}}(\mathbb{R}^d)$, where $\tilde{\sigma}=\max_{1\leq j \leq m}\{\|\nu_j\|_\mathcal{G}\}$.  Let $\mathcal{V}=\{\nu_1,...,\nu_m\}\subset\mathcal{G}_{\sigma}(\mathbb{R}^d)$ and let $\hat{\mathcal{V}}^n=\{\hat{\nu}_1^n,...,\hat{\nu}_m^n\}$. Then:
\begin{align} 
& \mathbb{E}\left(\bigg|\min_{\mu\in\mathcal{G}_\sigma(\mathbb{R}^d)}S_{\lambda,\mathcal{V}}^\epsilon(\mu)-\min_{\mu\in\mathcal{G}_\sigma(\mathbb{R}^d)}S_{\lambda,\hat{\mathcal{V}}^n}^\epsilon(\mu)\bigg|\right) \leq \frac{m^{1/2}R^*_{d,\epsilon,\sigma}}{\sqrt{n}},\label{eqn:alt_synthsampS_exp} \\ & \bigg|\min_{\mu\in\mathcal{G}_\sigma(\mathbb{R}^d)}S_{\lambda,\mathcal{V}}^\epsilon(\mu)-\min_{\mu\in \mathcal{P}^n(\mathbb{R}^d)}S_{\lambda,\mathcal{V}}^\epsilon(\mu)\bigg| \leq \frac{R^*_{d,\epsilon,\sigma}}{\sqrt{n}}\label{eqn:alt_synthsampS_min},
\\& \mathbb{E}\left(\bigg|\min_{\mu\in\mathcal{G}_\sigma(\mathbb{R}^d)}S_{\lambda,\mathcal{V}}^\epsilon(\mu)-\min_{\mu\in \mathcal{P}^{n_1}(\mathbb{R}^d)}S_{\lambda,\hat{\mathcal{V}}^{n_2}}^\epsilon(\mu)\bigg|\right)\leq mR^*_{d,\epsilon,\sigma}\left(\frac{1}{\sqrt{n_1}}+\frac{1}{\sqrt{n_2}}\right),\label{eqn:alt_synthsampS_double}
\end{align}
\end{proposition}

We remark that the additional assumptions present in Proposition \ref{prop:alt} are to guarantee that the minimizer of the Sinkhorn barycenter functional has subgaussian constant bounded above by the maximal subgaussian constant over the reference measures. This is not needed in Proposition \ref{prop:approx_min}, as this is implied by Proposition \ref{prop:domination}. The proofs of Propositions \ref{prop:approx_min} and Proposition \ref{prop:alt} rely on three technical results established in \citep{mena2019}.

\begin{proposition}\label{prop:menaweed_1}
(Proposition 2 in \citep{mena2019}) Let $\mu,\nu$ and $\hat{\nu}^n$ be $\sigma$-subgaussian measures, where $\sigma\in [0,\infty)$ is (possibly) random. Then: \begin{equation*}
|OT^\epsilon_2(\mu,\hat{\nu}^n)-OT^\epsilon_2(\mu,\nu)|\leq 2\sup_{u\in \mathcal{F}_{\epsilon,\sigma}}\bigg|\int u d\nu -\int u d\hat{\nu}^n\bigg|,
\end{equation*}
almost-surely, where $\mathcal{F}_{\epsilon,\sigma}$ is defined as the set of functions $u$ such that, for any multi-index $\alpha$ with $|\alpha|=k$: 

\begin{equation*} |D^\alpha(u-\frac{1}{2}\|.\|^2)(x)| \leq C_{\epsilon,k,d}\begin{cases}
1+\sigma^4 & \text{if}\; k=0 \\
\sigma^k(\sigma+\sigma^2)^k & \text{otherwise}\\
\end{cases}
\end{equation*}
if $\|x\|\leq \sqrt{d}\sigma$, and 
\begin{equation*}
|D^\alpha(u-\frac{1}{2}\|.\|^2)(x)| \leq C_{\epsilon,k,d}\begin{cases}
1+(1+\sigma^2)\|x\|^2 & \text{if}\; k=0 \\
\sigma^k(\sqrt{\sigma\|x\|}+\sigma\|x\|)^k & \text{otherwise}\\
\end{cases}
\end{equation*}
if $\|x\|>\sqrt{d}\sigma$, where $C_{\epsilon,k,d}$ is a constant depending only on $\epsilon, k, d$. 
\end{proposition}

\begin{proposition}\label{prop:menaweed_2} (Proof of Theorem 2, \citep{mena2019})
Let $\nu$ be $\sigma$-subgaussian and let $\tilde{\sigma}$ be a random variable such that both $\nu$ and $\hat{\nu}^n$ are $\tilde{\sigma}$-subgaussian. Then: 
\begin{equation}\label{eqn:prop:menaweed_2}
\mathbb{E}_{\hat{\nu}^n,\tilde{\sigma}}\left(\sup_{u\in \mathcal{F}_{\epsilon,\tilde{\sigma}}}\int u d\nu -\int ud\hat{\nu}^n\right)\leq \frac{C_{d,\epsilon}}{\sqrt{n}}(\mathbb{E}_{\tilde{\sigma}}((1+\tilde{\sigma}^{3s})^2))^{1/2} (1+\sigma^{d+2}),
\end{equation}
where $s=\lceil \frac{d}{2}\rceil+1$ and $C_{d,\epsilon}$ is a constant depending on $d,\epsilon$.
\end{proposition}
As Proposition \ref{prop:menaweed_2} is not stated exactly as a result in \citep{mena2019}, we give a sketch of the proof:
\begin{proof}
For any $q\geq 2$, let $\mathcal{F}^{\epsilon,q}$ to be the set of all functions with $|f(x)| \leq C_{d,\epsilon}(1+\|x\|^2)$ and $|D^\alpha f(x)|\leq C_{d,\epsilon} (1+\|x\|^q)$ for $|\alpha|\leq  q$, where $C_{d,\epsilon}$ is an absolute constant that can depend on $d$ and $\epsilon$. Then it can be shown that if $u\in \mathcal{F}_{\epsilon,\tilde{\sigma}}$ (the function class from Proposition \ref{prop:menaweed_1}) then $\frac{u}{1+\tilde{\sigma}^{3s}}\in \mathcal{F}^{\epsilon,s}$, almost-surely. This implies the left hand side of (\ref{eqn:prop:menaweed_2}) can be upper bound by \begin{align*} &\mathbb{E}_{\tilde{\sigma},\hat{\mathcal{\nu}}^n}\left((1+\tilde{\sigma}^{3s}) \sup_{u\in \mathcal{F}^{\epsilon,s}}\left|\int u d\nu -\int u d\hat{\nu}^n\right|\right)\\
\leq &\left(\mathbb{E}_{\tilde{\sigma}}((1+\tilde{\sigma}^{3s})^2)\right)^{1/2} \left(\mathbb{E}_{\hat{\nu}^n}\left(\sup_{u\in \mathcal{F}^{\epsilon,s}}\left|\int u d\nu -\int u d\hat{\nu}^n\right|\right)^2\right)^{1/2},\end{align*} 
where we applied the Cauchy-Schwarz inequality. Following the empirical process theory argument as in the proof of Theorem 2 in \citep{mena2019}, one is then able to bound 

\begin{equation*}\mathbb{E}_{\hat{\nu}^n}\left(\left(\sup_{u\in \mathcal{F}^{\epsilon,s}}\left|\int u d\nu -\int u d\hat{\nu}^n\right|\right)^2\right)
\end{equation*}by $\frac{C_{d,\epsilon}}{n}(1+\sigma^{2d+4})$.  We conclude by applying the bound $(1+\sigma^{2d+4})^{1/2}\leq 1+\sigma^{d+2}$, which holds since $\sigma\geq 0$.
\end{proof}

\begin{lemma}\label{prop:menaweed_3}
(Lemma A.4 in \citep{mena2019}) Let $\nu$ be $\sigma$-subgaussian. Then for all integers $k\geq 1$,
\begin{equation*} \mathbb{E}(\|\hat{\nu}^n\|_\mathcal{G}^{2k})\leq L_k\sigma^{2k},\end{equation*}
\begin{equation*}
\mathbb{E}(\|\hat{\nu}^n\|^{2k-1}_\mathcal{G}) \leq \max\{L_k\sigma^{2k},1\}
\end{equation*}
where $L_{k}$ is a constant depending on $k$.
\end{lemma}

We remark that Lemma A.4 in \citep{mena2019} establishes $\mathbb{E}(\|\hat{\nu}^n\|_\mathcal{G}^{2k})\leq L_k\sigma^{2k}$. To extend this to odd powers, note that $\mathbb{E}(\|X\|^k)\leq \mathbb{E}(\|X\|^{k+1})^{\frac{k}{k+1}}$ for any random variable $X$ by Jensen's inequality. If $\mathbb{E}(\|X\|^{k+1})\geq 1$ then $\mathbb{E}(\|X\|^{k+1})^{\frac{k}{k+1}}\leq \mathbb{E}(\|X\|^{k+1})$, and hence $\mathbb{E}(\|X\|^{k})\leq \mathbb{E}(\|X\|^{k+1})$, and if  $\mathbb{E}(\|X\|^{k+1})\leq 1$ then $\mathbb{E}(\|X\|^{k})\leq 1$, from which we establish Lemma \ref{prop:menaweed_3}.  Also note that from [Lemma A.2, \citep{mena2019}], the random variable $\|\hat{\nu}^n\|_\mathcal{G}$ is finite almost surely. 

\begin{lemma}\label{lemma:mena_weed_4}
For all $1\leq j \leq m$, let $\|\nu_j\|_\mathcal{G}\leq \sigma$, and let $\tilde{\sigma}:=\max\{\sigma,\|\hat{\nu}_1^n\|_\mathcal{G},\|\hat{\nu}_2^n\|_\mathcal{G},...,\|\hat{\nu}_m^n\|_\mathcal{G}\}$. Then if $\|\mu\|_\mathcal{G}\leq \tilde{\sigma}$,
\begin{equation}\label{eqn:mena weed 4 eqn 1}
\mathbb{E}(|OT_2^\epsilon(\mu,\nu_j)-OT^\epsilon_2(\mu,\hat{\nu}^n_j)|)\leq \frac{m^{1/2}P_{d,\epsilon}(\sigma)}{\sqrt{n}},\;\;\forall 1\leq j\leq m,
\end{equation}
for some polynomial $P_{d,\epsilon}(\sigma)$ with coefficients depending on $d,\epsilon$ and degree depending on $d$.  
\end{lemma}
\begin{proof}
As $\mu,\nu_j$ and $\hat{\nu}_j^n$ are all $\tilde{\sigma}$ subgaussian, by Proposition \ref{prop:menaweed_1} we have:
\begin{equation*} \mathbb{E}(|OT_2^\epsilon(\mu,\nu_j)-OT^\epsilon_2(\mu,\hat{\nu}^n_j)|) \leq  \mathbb{E}\left(2\sup_{u\in \mathcal{F}_{\epsilon,\tilde{\sigma}}}\bigg|\int u d\nu_j -\int u d\hat{\nu}_j^n\bigg|\right),\end{equation*} 

where the expectation is with respect to the empirical measure $\hat{\nu}_j^n$ and $\tilde{\sigma}$.  By Proposition \ref{prop:menaweed_2} and Young's inequality we may upper bound:
\begin{align}
\mathbb{E}\left(2\sup_{u\in \mathcal{F}_{\epsilon,\tilde{\sigma}}}\bigg|\int u d\nu_j -\int u d\hat{\nu}_j^n\bigg|\right) & \leq \frac{2C_{d,\epsilon}}{\sqrt{n}}(\mathbb{E}_{\tilde{\sigma}}((1+\tilde{\sigma}^{3s})^2))^{1/2}(1+\sigma^{d+2})\notag \\& \leq \frac{2C_{d,\epsilon}}{\sqrt{n}}(2+ 2 \mathbb{E}_{\tilde{\sigma}}(\tilde{\sigma}^{6s}))^{1/2}(1+\sigma^{d+2}) \label{eqn:sigma tilde k}
\end{align}
where $s=\lceil\frac{d}{2}\rceil+1$. It remains to bound $\mathbb{E}_{\tilde{\sigma}}(\tilde{\sigma}^{6s})$.
\begin{align*} \mathbb{E}_{\tilde{\sigma}}(\tilde{\sigma}^{6s})&=\mathbb{E}(\max\{\sigma,\|\hat{\nu}_1^n\|_\mathcal{G},\|\hat{\nu}_2^n\|_\mathcal{G},...,\|\hat{\nu}_m^n\|_\mathcal{G}\}^{6s})\\
& \leq \mathbb{E}\left(\sigma^{6s}+\sum_{j=1}^m\|\hat{\nu}_j^n\|_\mathcal{G}^{6s}\right)\\&=\sigma^{6s}+\sum_{j=1}^m\mathbb{E}\left(\|\hat{\nu}_j^n\|_\mathcal{G}^{6s}\right)\\& \leq \sigma^{6s}+ mL_{6s} \sigma^{6s}\end{align*}
by Lemma \ref{prop:menaweed_3}. Therefore:
 \begin{align*}
 \frac{2C_{d,\epsilon}}{\sqrt{n}}(2+ 2 \mathbb{E}_{\tilde{\sigma}}(\tilde{\sigma}^{6s}))^{1/2}(1+\sigma^{d+2}) & \leq   \frac{2C_{d,\epsilon}}{\sqrt{n}}(2+ 2 (mL_{6s}+1)\sigma^{6s})^{1/2}(1+\sigma^{d+2})\\& \leq \frac{2m^{1/2}C_{d,\epsilon}}{\sqrt{n}}(2+ 2 (L_{6s}+1)\sigma^{6s})^{1/2}(1+\sigma^{d+2})\\& \leq \frac{2m^{1/2}C_{d,\epsilon}}{\sqrt{n}}(\sqrt{2}+\sqrt{2(L_{6s}+1)}\sigma^{3s})(1+\sigma^{d+2}).
 \end{align*}
We conclude that $\mathbb{E}(|OT_2^\epsilon(\mu,\nu_j)-OT^\epsilon_2(\mu,\hat{\nu}_j^{n})|) \leq \frac{m^{1/2}P_{d,\epsilon}(\sigma)}{\sqrt{n}}$ for some polynomial $P_{d,\epsilon}(\sigma)$ with coefficients depending on $d,\epsilon$ and degree depending on $d$. 
\end{proof}

\begin{lemma}\label{lemma:mena_weed_5} Let $\mu\in \mathcal{G}_\sigma(\mathbb{R}^d)$. Then:
\begin{equation*}
\mathbb{E}\left(| OT^\epsilon_2(\mu,\mu)-OT^\epsilon_2(\hat{\mu}^n,\hat{\mu}^n)|\right)\leq \frac{2P_{d,\epsilon}(\sigma)}{\sqrt{n}},
\end{equation*}
for some polynomial $P_{d,\epsilon}(\sigma)$ with coefficients depending on $d,\epsilon$ and degree depending on $d$, and the expectation is over the empirical measure $\hat{\mu}^n.$
\end{lemma}
\begin{proof}
We apply the triangle inequality to bound:
\begin{align}
& \mathbb{E}\left(|OT_2^\epsilon(\mu,\mu)-OT_2^\epsilon(\hat{\mu}^n,\hat{\mu}^n)|\right)\notag\\
= &\mathbb{E}\left(|OT_2^\epsilon(\mu,\mu)- OT_2^\epsilon(\hat{\mu}^n,\mu)+OT_2^\epsilon(\hat{\mu}^n,\mu)-OT_2^\epsilon(\hat{\mu}^n,\hat{\mu}^n)|\right)\notag \\ 
 \leq &\mathbb{E}\left(|OT_2^\epsilon(\mu,\mu)- OT_2^\epsilon(\hat{\mu}^n,\mu)|\right)+\mathbb{E}\left(OT_2^\epsilon(\hat{\mu}^n,\mu)-OT_2^\epsilon(\hat{\mu}^n,\hat{\mu}^n)|\right)\label{eqn:synthexp bound 1}.
\end{align}
The first term is bounded by $\frac{P_{d,\epsilon}(\sigma)}{\sqrt{n}}$ following a similar argument as used in the proof of Lemma \ref{lemma:mena_weed_4} with $m=1$. We bound the second term by conditioning:
\begin{align}
&\mathbb{E}\left(|OT_2^\epsilon(\hat{\mu}^n,\hat{\mu}^n)-OT_2^\epsilon(\hat{\mu}^n,\mu)|\right)\notag \\ & =\mathbb{E}\left(\mathbb{E}\left(|OT_2^\epsilon(\hat{\mu}^n,\hat{\mu}^n)-OT_2^\epsilon(\hat{\mu}^n,\mu)|\; \bigg| \; \hat{\mu}^n\right)\right)\label{eqn:sink condit 0}\\& \leq \mathbb{E}\left(\mathbb{E}\left(2\sup_{u\in \mathcal{F}_{\epsilon,\tilde{\sigma}}}\bigg|\int u d\mu-\int u d \hat{\mu}^n\; \bigg| \; \bigg| \;\hat{\mu}^n\right) \right)\notag\\&  =\mathbb{E}\left(2\sup_{u\in \mathcal{F}_{\epsilon,\tilde{\sigma}}}\bigg|\int u d\mu -\int u d \hat{\mu}^n\; \bigg|  \right)\notag
\end{align}
where we applied Proposition \ref{prop:menaweed_1} to the inner expectation in (\ref{eqn:sink condit 0}) with $\tilde{\sigma} = \max \{\sigma, \|\hat{\mu}^n \|_{\mathcal{G}}\}$. We then apply Proposition \ref{prop:menaweed_2} followed by Lemma \ref{prop:menaweed_3} to obtain the bound:
\begin{equation*}
\mathbb{E}\left(|OT_2^\epsilon(\hat{\mu}^n,\hat{\mu}^n)-OT_2^\epsilon(\hat{\mu}^n,\mu)|\right)\leq \frac{P_{d,\epsilon}(\sigma)}{\sqrt{n}}.
\end{equation*}
\end{proof}

\noindent\textbf{Proof of Proposition \ref{prop:approx_min}:}
We begin by proving (\ref{eqn:synthsampF_exp}).  Let $\mu^*$ be a minimizer for $F^\epsilon_{\lambda,\mathcal{V}}$, and $\mu^{*,n}$ be a random minimizer for $F^\epsilon_{\lambda,\hat{\mathcal{V}}^n}$, so that
\begin{equation}
\mathbb{E}\left(\bigg|\min_{\mu\in\mathcal{P}_2(\mathbb{R}^d)}F_{\lambda,\mathcal{V}}^\epsilon(\mu)-\min_{\mu\in\mathcal{P}_2(\mathbb{R}^d)}F_{\lambda,\hat{\mathcal{V}}^n}^\epsilon(\mu)\bigg|\right)= \mathbb{E}\left(\bigg|F_{\lambda,{\mathcal{V}}}^\epsilon(\mu^{*})-F_{\lambda,\hat{\mathcal{V}}^n}^\epsilon(\mu^{*,n})\bigg|\right).\notag\end{equation}
Suppose that $F^\epsilon_{\lambda,\hat{\mathcal{V}}^n}(\mu^{*,n})\geq F^\epsilon_{\lambda,\mathcal{V}}(\mu^*)$. Then:
\begin{align}\bigg|F_{\lambda,{\mathcal{V}}}^\epsilon(\mu^{*})-F_{\lambda,\hat{\mathcal{V}}^n}^\epsilon(\mu^{*,n})\bigg|& =F_{\lambda,\hat{\mathcal{V}}^n}^\epsilon(\mu^{*,n})-F_{\lambda,{\mathcal{V}}}^\epsilon(\mu^{*})\notag & \\ & \leq F_{\lambda,\hat{\mathcal{V}}^n}^\epsilon(\mu^*)-F_{\lambda,{\mathcal{V}}}^\epsilon(\mu^{*})\notag.\end{align}
Similarly, if $F^\epsilon_{\lambda,\hat{\mathcal{V}}^n}(\mu^{*,n})\leq F^\epsilon_{\lambda,\mathcal{V}}(\mu^*)$, we have \begin{equation}\bigg|F_{\lambda,{\mathcal{V}}}^\epsilon(\mu^{*})-F_{\lambda,\hat{\mathcal{V}}^n}^\epsilon(\mu^{*,n})\bigg|\leq F_{\lambda,\mathcal{V}}^\epsilon(\mu^{*,n})-F_{\lambda,{\hat{\mathcal{V}}^n}}^\epsilon(\mu^{*,n})\notag\end{equation} and hence we may conclude that \begin{equation}
\mathbb{E}\left(\bigg|F_{\lambda,{\mathcal{V}}}^\epsilon(\mu^{*})-F_{\lambda,\hat{\mathcal{V}}^n}^\epsilon(\mu^{*,n})\bigg|\right) \leq \mathbb{E}\left(\max_{\mu\in \{\mu^*,\mu^{*,n}\}} \bigg|F_{\lambda,\mathcal{V}}^\epsilon(\mu)-F_{\lambda,\hat{\mathcal{V}}^n}^\epsilon(\mu)\bigg|\right)\notag.
\end{equation}
We then bound:
\begin{align} & \mathbb{E}\left(\max_{\mu\in \{\mu^*,\mu^{*,n}\}} \bigg|F_{\lambda,\mathcal{V}}^\epsilon(\mu)-F_{\lambda,\hat{\mathcal{V}}^n}^\epsilon(\mu)\bigg|\right)\notag   \\ & = \mathbb{E}\left(\max_{\mu\in \{\mu^*,\mu^{*,n}\}} \bigg|\sum_{j=1}^m\lambda_jOT^\epsilon_2(\mu,\nu_j)-\sum_{j=1}^m\lambda_jOT^\epsilon_2(\mu,\hat{\nu}_j^n)\bigg|\right)\notag &\\&
 \leq \mathbb{E}\left(\max_{\mu\in \{\mu^*,\mu^{*,n}\}} \sum_{j=1}^m\lambda_j \bigg|OT^\epsilon_2(\mu,\nu_j)- OT^\epsilon_2(\mu,\hat{\nu}_j^n)\bigg|\right)\notag &\\&
\leq \mathbb{E}\left( \sum_{j=1}^m\lambda_j \max_{\mu\in \{\mu^*,\mu^{*,n}\}} \bigg|OT^\epsilon_2(\mu,\nu_j)- OT^\epsilon_2(\mu,\hat{\nu}_j^n)\bigg|\right)\notag \\ & = \sum_{j=1}^m\lambda_j \mathbb{E}\left(\max_{\mu\in \{\mu^*,\mu^{*,n}\}} \bigg|OT^\epsilon_2(\mu,\nu_j)- OT^\epsilon_2(\mu,\hat{\nu}_j^n)\bigg|\right) \\
&  \leq \sum_{j=1}^m\lambda_j \mathbb{E}\left( \bigg|OT^\epsilon_2(\mu^*,\nu_j)- OT^\epsilon_2(\mu^*,\hat{\nu}_j^n)\bigg|\right) + \sum_{j=1}^m\lambda_j \mathbb{E}\left( \bigg|OT^\epsilon_2(\mu^{*,n},\nu_j)- OT^\epsilon_2(\mu^{*,n},\hat{\nu}_j^n)\bigg|\right)  \label{eqn:max entot}
\end{align}
To bound $\mathbb{E}\left( |OT_2^\epsilon(\mu^*,\nu_j)-OT_2^\epsilon(\mu^*,\hat{\nu}_j^n)|\right)$, note by Proposition \ref{prop:subgauss}, $\mu^*$ is $\sigma$-subgaussian, and $\mu^{*,n}$ is $\max_{1\leq j\leq m}\{\|\hat{\nu}_j^n\|_\mathcal{G}\}$-subgaussian,  both $\mu^*$ and $\mu^{*,n}$ are $\max\{\sigma,\|\hat{\nu}^n_1\|_\mathcal{G},...,\|\hat{\nu}_m^n\|_\mathcal{G}\}$-subgaussian, and we may apply Lemma \ref{lemma:mena_weed_4} to conclude: 

\begin{equation}\label{eqn:entot polybound}
\mathbb{E}\left( |OT_2^\epsilon(\mu^*,\nu_j)-OT_2^\epsilon(\mu^*,\hat{\nu}_j^n)|\right)\leq \frac{m^{1/2}P_{d,\epsilon}(\sigma)}{\sqrt{n}}, \,\, \mathbb{E}\left( |OT_2^\epsilon(\mu^{*,n},\nu_j)-OT_2^\epsilon(\mu^{*,n},\hat{\nu}_j^n)|\right)\leq \frac{m^{1/2}P_{d,\epsilon}(\sigma)}{\sqrt{n}},
\end{equation}
where $P_{d,\epsilon}(\sigma)$ is a polynomial in $\sigma$, depending on $d$ and $\epsilon$. Summing over $1\leq j\leq m$, we establish (\ref{eqn:synthsampF_exp}) with a bound of $\frac{2m^{1/2}P_{d,\epsilon}(\sigma)}{\sqrt{n}}.$

Next we prove (\ref{eqn:synthsampF_min}).  Let $\mu^*$ minimize $F^\epsilon_{\lambda,\mathcal{V}}$. Note that the empirical measure $\hat{\mu^*}^n$ associated to $\mu^{*}$ is an element of $\mathcal{P}^n(\mathbb{R}^d)$, and hence
\begin{align}
\left|\min_{\mu\in\mathcal{P}_2(\mathbb{R}^d)}F^\epsilon_{\lambda,\mathcal{V}}(\mu)-\min_{\mu\in \mathcal{P}^n(\mathbb{R}^d)}F^\epsilon_{\lambda,\mathcal{V}}(\mu)\right| & = \min_{\mu\in \mathcal{P}^n(\mathbb{R}^d)}F^\epsilon_{\lambda,\mathcal{V}}(\mu)-\min_{\mu\in\mathcal{P}_2(\mathbb{R}^d)}F^\epsilon_{\lambda,\mathcal{V}}(\mu)\notag \\&\leq \mathbb{E}\left(F^\epsilon_{\lambda,\mathcal{V}}(\hat{\mu^*}^n)\right)-F^\epsilon_{\lambda,\mathcal{V}}(\mu^*) \notag \\& = \mathbb{E}\left(|F^\epsilon_{\lambda,\mathcal{V}}(\hat{\mu^*}^n)-F^\epsilon_{\lambda,\mathcal{V}}(\mu^*)|\right)\notag \\& \leq \sum_{j=1}^m\lambda_j\mathbb{E}\left(|OT_2^\epsilon(\hat{\mu^*}^n,\nu_j)-OT_2^\epsilon(\mu^*,\nu_j)|\right) \label{eqn:synth f min 1}
\end{align}
by the triangle inequality. We bound each term \begin{equation*} \mathbb{E}\left(|OT_2^\epsilon(\hat{\mu^*}^n,\nu_j)-OT_2^\epsilon(\mu^*,\nu_j)|\right) \leq\frac{P_{\epsilon,d}(\sigma)}{\sqrt{n}} \end{equation*} by applying Lemma \ref{lemma:mena_weed_4} with $m=1$.  This applies because $\|\mu^*\|_\mathcal{G}\leq \sigma$, by Proposition \ref{prop:subgauss}.  This gives (\ref{eqn:synthsampF_min}).

We now prove (\ref{eqn:synthsampF_double}): 
\begin{align*}
& \mathbb{E}\left(\bigg|\min_{\mu\in\mathcal{P}_2(\mathbb{R}^d)}F_{\lambda,\mathcal{V}}^\epsilon(\mu)-\min_{\mu\in\mathcal{P}^{n_1}(\mathbb{R}^d)}F_{\lambda,\hat{\mathcal{V}}^{n_2}}^\epsilon(\mu)\bigg|\right)\\
\leq &\mathbb{E}\left(\bigg|\min_{\mu\in\mathcal{P}_2(\mathbb{R}^d)}F_{\lambda,\mathcal{V}}^\epsilon(\mu)-\min_{\mu\in\mathcal{P}_2(\mathbb{R}^d)}F_{\lambda,\hat{\mathcal{V}}^{n_2}}^\epsilon(\mu)\bigg|\right)+\mathbb{E}\left(\bigg|\min_{\mu\in\mathcal{P}_2(\mathbb{R}^d)}F_{\lambda,\hat{\mathcal{V}}^{n_2}}^\epsilon(\mu)-\min_{\mu\in\mathcal{P}^{n_1}(\mathbb{R}^d)}F_{\lambda,\hat{\mathcal{V}}^{n_2}}^\epsilon(\mu)\bigg|\right)
\end{align*}
We apply (\ref{eqn:synthsampF_exp}) to the first term to produce a bound of $\frac{2m^{1/2}P_{d,\epsilon}(\sigma)}{\sqrt{n_2}}$.  Let 
\begin{equation*} \tilde{\sigma}=\max_{1\leq j\leq m}\|\hat{\nu}^{n_2}_j\|_\mathcal{G}.
\end{equation*} 
Then, by conditioning on the samples $\hat{\mathcal{V}}^{n_2}$, we have:
\begin{align}
&\mathbb{E}\left(\bigg|\min_{\mu\in\mathcal{P}_2(\mathbb{R}^d)}F_{\lambda,\hat{\mathcal{V}}^{n_2}}^\epsilon(\mu)-\min_{\mu\in\mathcal{P}^{n_1}(\mathbb{R}^d)}F_{\lambda,\hat{\mathcal{V}}^{n_2}}^\epsilon(\mu)\bigg|\right)\notag \\
 = &\mathbb{E}\left(\mathbb{E}\left(\bigg|\min_{\mu\in\mathcal{P}_2(\mathbb{R}^d)}F_{\lambda,\hat{\mathcal{V}}^{n_2}}^\epsilon(\mu)-\min_{\mu\in\mathcal{P}^{n_1}(\mathbb{R}^d)}F_{\lambda,\hat{\mathcal{V}}^{n_2}}^\epsilon(\mu)\bigg| \; \bigg| \; \hat{\mathcal{V}}^{n_2}\right)\right)\notag\\
  \leq &\mathbb{E}\left(\mathbb{E}\left( \frac{2P_{d,\epsilon}(\tilde{\sigma})}{\sqrt{n_1}} \; \bigg| \; \hat{\mathcal{V}}^{n_2}\right)\right)\\
=&\frac{2}{\sqrt{n_1}}\mathbb{E}(P_{\sigma,\epsilon}(\tilde{\sigma})).\label{eqn:apply51}
\end{align}where we applied (\ref{eqn:synthsampF_min}) conditionally on the samples $\hat{\mathcal{V}}^{n_2}$. By linearity of expectation:\begin{align} \mathbb{E}\left( P_{d,\epsilon}(\tilde{\sigma})\right)& = \mathbb{E}\left(\sum_{i=1}^{r_d}c_{d,\epsilon,i}\tilde{\sigma}^i\right)\\& =\sum_{i=1}^{r_d}c_{d,\epsilon,i}\mathbb{E}(\tilde{\sigma}^i)\label{eqn:expected_poly}
\end{align}
for some maximum degree $r_d$ depending on $d$, and (deterministic) coefficients $c_{d,\epsilon,i}$, depending on $d$ and $\epsilon$. Since $\tilde{\sigma}^i\leq \sum_{j=1}^m\|\hat{\nu}_j\|_\mathcal{G}^i$, we have:

\begin{align*} 
\mathbb{E}(P_{d,\epsilon}(\tilde{\sigma})) & \leq \sum_{i=1}^{r_d}c_{d,\epsilon,i}\sum_{j=1}^m\mathbb{E}(\|\hat{\nu}_j^n\|_{\mathcal{G}}^i)\\&\leq \sum_{i=1}^{r_d}c_{d,\epsilon,i}\sum_{j=1}^m \max\{L_{\lceil \frac{i}{2}\rceil}\sigma^{2\lceil{\frac{i}{2}\rceil}},1\}\\&\leq m\sum_{i=1}^{r_d}c_{d,\epsilon,i}(L_{\lceil\frac{i}{2}\rceil}\sigma^{2\lceil\frac{i}{2}\rceil}+1)\\& =:m\tilde{P}_{d,\epsilon}(\sigma),\end{align*}
 where we applied Lemma \ref{prop:menaweed_3}. Thus we may conclude that:
\begin{equation*}
\mathbb{E}\left(\bigg|\min_{\mu\in\mathcal{P}_2(\mathbb{R}^d)}F_{\lambda,\mathcal{V}}^\epsilon(\mu)-\min_{\mu\in\mathcal{P}^{n_1}(\mathbb{R}^d)}F_{\lambda,\hat{\mathcal{V}}^{n_2}}^\epsilon(\mu)\bigg|\right)\leq (2m^{1/2}P_{d,\epsilon}(\sigma)+m\tilde{P}_{d,\epsilon}(\sigma)) \left(\frac{1}{\sqrt{n_2}}+\frac{1}{\sqrt{n_1}}\right),
\end{equation*}
which implies  (\ref{eqn:synthsampF_double}).

Labeling the maximum over all constants in the three statements in Proposition \ref{prop:approx_min} as $R^*_{d,\epsilon,\sigma}$, we obtain our result. \qed 
\vspace{10pt}

\noindent\textbf{Proof of Proposition \ref{prop:alt}:}  We first show (\ref{eqn:alt_synthsampS_exp}). Let $\mu^*\in\argmin_{\mu\in \mathcal{G}_\sigma(\mathbb{R}^d)}S^\epsilon_{\lambda,\mathcal{V}}$ and $\mu^{*,n}\in\argmin_{\mu\in \mathcal{G}_\sigma(\mathbb{R}^d)}S^\epsilon_{\lambda,\hat{\mathcal{V}}^n}$. Then we bound:

\begin{align}
&\mathbb{E}\left(\bigg|\min_{\mu\in\mathcal{G}_\sigma(\mathbb{R}^d)}S_{\lambda,\mathcal{V}}^\epsilon(\mu)-\min_{\mu\in\mathcal{G}_\sigma(\mathbb{R}^d)}S_{\lambda,\hat{\mathcal{V}}^n}^\epsilon(\mu)\bigg|\right)\notag \\  \leq &\mathbb{E}\left(\sum_{j=1}^m\lambda_j\max_{\mu\in \{\mu^*,\mu^{*,n}\}} \bigg|OT^\epsilon_2(\mu,\nu_j)-OT^\epsilon_2(\mu,\hat{\nu}^n_j) - \frac{1}{2}OT^\epsilon_2(\nu_j,\nu_j)+\frac{1}{2}OT^\epsilon_2(\hat{\nu}_j^n,\hat{\nu}_j^n)\bigg|\right)\notag \\
\leq &\mathbb{E}\left(\sum_{j=1}^m\lambda_j\max_{\mu\in \{\mu^*,\mu^{*,n}\}} \left(\bigg|OT^\epsilon_2(\mu,\nu_j)-OT^\epsilon_2(\mu,\hat{\nu}^n_j)\bigg| + \bigg| \frac{1}{2}OT^\epsilon_2(\nu_j,\nu_j)-\frac{1}{2}OT^\epsilon_2(\hat{\nu}_j^n,\hat{\nu}_j^n)\bigg|\right)\right)\notag  \\
= &\sum_{j=1}^m\lambda_j\left( \mathbb{E}\left(\max_{\mu\in \{\mu^*,\mu^{*,n}\}}|OT^\epsilon_2(\mu,\nu_j)-OT^\epsilon_2(\mu,\hat{\nu}^n_j)|\right)+\mathbb{E}\left(| \frac{1}{2}OT^\epsilon_2(\nu_j,\nu_j)-\frac{1}{2}OT^\epsilon_2(\hat{\nu}_j^n,\hat{\nu}_j^n)|\right)\right)\notag\end{align}
where we have applied the triangle inequality. Using the assumptions  $\mu^*,\mu^{*,n}\in \mathcal{G}_\sigma(\mathbb{R}^d)$, we bound:
\begin{equation}\sum_{j=1}^m \lambda_j\mathbb{E}\left(\max_{\mu\in \{\mu^*,\mu^{*,n}\}}|OT^\epsilon_2(\mu,\nu_j)-OT_2^\epsilon(\mu,\hat{\nu}_j^n)| \right)\leq \frac{2m^{1/2}P_{d,\epsilon}(\sigma)}{\sqrt{n}}\label{eqn:synthexp bound 0}
\end{equation} 
by applying Lemma \ref{lemma:mena_weed_4}. To bound $\frac{1}{2}\sum_{j=1}^m\lambda_j \mathbb{E}\left(|OT^\epsilon_2(\nu_j,\nu_j)-OT^\epsilon_2(\hat{\nu}_j^n,\hat{\nu}_j^n)|\right)$, we apply Lemma \ref{lemma:mena_weed_5}, giving us a bound of $\frac{P_{d,\epsilon}(\sigma)}{\sqrt{n}}.$ We conclude:
\begin{equation*}
\mathbb{E}\left(\bigg|\min_{\mu\in\mathcal{G}_\sigma(\mathbb{R}^d)}S_{\lambda,\mathcal{V}}^\epsilon(\mu)-\min_{\mu\in\mathcal{G}_\sigma(\mathbb{R}^d)}S_{\lambda,\hat{\mathcal{V}}^n}^\epsilon(\mu)\bigg|\right)\leq \frac{(2m^{1/2}+1) P_{d,\epsilon}(\sigma)}{\sqrt{n}},
\end{equation*}
which establishes (\ref{eqn:alt_synthsampS_exp}). 

Similarly, to establish (\ref{eqn:alt_synthsampS_min}), we let $\mu^*$ again be the minimizer of $S^\epsilon_{\lambda,\mathcal{V}}$ on $\mathcal{G}_\sigma(\mathbb{R}^d)$ and let $\hat{\mu^*}^n$ be the empirical measure for $\mu^*$. By assumption, we have for any $\rho\in \argmin_{\mu\in \mathcal{P}^n(\mathbb{R}^d)}S^\epsilon_{\lambda,\mathcal{V}}(\mu)$, $\rho\in \mathcal{G}_\sigma(\mathbb{R}^d)$, and hence:
\begin{align} &\bigg|\min_{\mu\in \mathcal{G}_\sigma(\mathbb{R}^d) }S^\epsilon_{\lambda,\mathcal{V}}(\mu)-\min_{\mu\in \mathcal{P}^n(\mathbb{R}^d)}S^\epsilon_{\lambda,\mathcal{V}}(\mu)\bigg|\notag\\ & = \min_{\mu\in \mathcal{P}^n(\mathbb{R}^d)}S^\epsilon_{\lambda,\mathcal{V}}(\mu)- S^\epsilon_{\lambda,\mathcal{V}}(\mu^*).\notag\end{align}
Since $\hat{\mu^*}^n\in \mathcal{P}^n(\mathbb{R}^d)$, we may upper bound this by:
\begin{align*}
&\mathbb{E}\left(S_{\lambda,{\mathcal{V}}}^\epsilon(\hat{\mu^*}^n)\right)-S_{\lambda,\mathcal{V}}^\epsilon(\mu^*)\notag \\ =&\mathbb{E}\left(\bigg| S_{\lambda,{\mathcal{V}}}^\epsilon(\hat{\mu^*}^n)-S_{\lambda,\mathcal{V}}^\epsilon(\mu^*)\bigg|\right)  \\ \leq &\sum_{j=1}^m\lambda_j \mathbb{E}\left(|OT_2^\epsilon(\hat{\mu^*}^n,\nu_j)-OT_2^\epsilon(\mu^*,\nu_j)|\right)+\frac{1}{2}\mathbb{E}\left(|OT_2^\epsilon(\hat{\mu^*}^n,\hat{\mu^*}^n)-OT_2^\epsilon(\mu^*,\mu^*)|\right)
\end{align*}
by the triangle inequality. By the assumption that $\mu^*\in \mathcal{G}_\sigma(\mathbb{R}^d)$, we may apply Lemma \ref{lemma:mena_weed_4} to bound each summand in the first term by $\frac{P_{d,\epsilon}(\sigma)}{\sqrt{n}}$ and the second term can be bounded by $\frac{P_{d,\epsilon}(\sigma)}{\sqrt{n}}$ via Lemma \ref{lemma:mena_weed_5}. Combining these two bounds, we have:
\begin{equation}\label{eqn:finite approx sink}
\bigg|\min_{\mu\in \mathcal{G}_\sigma(\mathbb{R}^d) }S^\epsilon_{\lambda,\mathcal{V}}(\mu)-\min_{\mu\in \mathcal{P}^n(\mathbb{R}^d)}S^\epsilon_{\lambda,\mathcal{V}}(\mu)\bigg|\leq \frac{2 P_{d,\epsilon}(\sigma)}{\sqrt{n}}.
\end{equation}

We conclude with the proof of (\ref{eqn:alt_synthsampS_double}):
\begin{align*}
& \mathbb{E}\left(\bigg|\min_{\mu\in\mathcal{G}_\sigma(\mathbb{R}^d)}S_{\lambda,\mathcal{V}}^\epsilon(\mu)-\min_{\mu\in\mathcal{P}^{n_1}(\mathbb{R}^d)}S_{\lambda,\hat{\mathcal{V}}^{n_2}}^\epsilon(\mu)\bigg|\right)\\
\leq &\bigg|\min_{\mu\in\mathcal{G}_\sigma(\mathbb{R}^d)}S_{\lambda,\mathcal{V}}^\epsilon(\mu)-\min_{\mu\in\mathcal{P}^{n_1}(\mathbb{R}^d)}S_{\lambda,\mathcal{V}}^\epsilon(\mu)\bigg|+\mathbb{E}\left(\bigg|\min_{\mu\in\mathcal{P}^{n_1}(\mathbb{R}^d)}S_{\lambda,\mathcal{V}}^\epsilon(\mu)-\min_{\mu\in\mathcal{P}^{n_1}(\mathbb{R}^d)}S_{\lambda,\hat{\mathcal{V}}^{n_2}}^\epsilon(\mu)\bigg|\right).
\end{align*}
by the triangle inequality. We may apply (\ref{eqn:finite approx sink}) to bound the first term by $\frac{2P_{d,\epsilon}(\sigma)}{\sqrt{n_1}}$.  To bound the second term, we apply a similar argument as used to bound (\ref{eqn:alt_synthsampS_exp}). Let $\mu^{n_1}\in\argmin_{\mu\in \mathcal{P}^{n_1}(\mathbb{R}^d)} S^\epsilon_{\lambda,\mathcal{V}}(\mu)$ and $\mu^{n_1}_{n_2}\in\argmin_{\mu\in \mathcal{P}^{n_1}(\mathbb{R}^d)}S^\epsilon_{\lambda,\hat{\mathcal{V}}^{n_2}}(\mu)$. Then: 
\begin{align*}
&\mathbb{E}\left(\bigg|\min_{\mu\in\mathcal{P}^{n_1}(\mathbb{R}^d)}S^\epsilon_{\lambda,\mathcal{V}}(\mu)-\min_{\mu\in \mathcal{P}^{n_1}(\mathbb{R}^d)}S^\epsilon_{\lambda,\hat{\mathcal{V}}^{n_2}}(\mu)\right)\\& \leq \mathbb{E}\left(\max_{\mu^*\in \{\mu^{n_1},\mu^{n_1}_{n_2}\}} \bigg|S^\epsilon_{\lambda,\mathcal{V}}(\mu^*)-S^\epsilon_{\lambda,\hat{\mathcal{V}}^{n_2}}(\mu^*)\bigg|\right)\\&\leq \mathbb{E}\left(\max_{\mu^*\in \{\mu^{n_1},\mu^{n_1}_{n_2}\}} \sum_{j=1}^m\lambda_j( |OT_2^\epsilon(\mu^*,\nu_j)-OT_2^\epsilon(\mu^*,\hat{\nu}_j^{n_2})|\right)+\sum_{j=1}^m\lambda_j\mathbb{E}\left(|OT_2^\epsilon(\nu_j,\nu_j)-OT_2^\epsilon(\hat{\nu}_j^{n_2},\hat{\nu}_j^{n_2})|\right)
\end{align*}
by the triangle inequality. The second term can be bounded by $\frac{P_{d,\epsilon}(\sigma)}{\sqrt{n}}$ by Lemma \ref{lemma:mena_weed_5}. To deal with the first term, note that by assumption, $\mu^{n_1}\in \mathcal{G}_\sigma(\mathbb{R}^d)$, and $\mu^{n_1}_{n_2}\in \mathcal{G}_{\tilde{\sigma}}(\mathbb{R}^d)$ with $\tilde{\sigma}=\max_{1\leq j\leq m}\{\|\hat{\nu}_j^{n_2}\|\}$, and hence $\{\mu^{n_1},\mu^{n_1}_{n_2}\}\subset \mathcal{G}_{\sigma^*}(\mathbb{R}^d)$ where $\sigma^*=\max\{\sigma,\tilde{\sigma}\}$. Hence we can bound:
\begin{align*}
& \mathbb{E}\left(\max_{\mu^*\in \{\mu^{n_1},\mu^{n_1}_{n_2}\}} \sum_{j=1}^m\lambda_j |OT_2^\epsilon(\mu^*,\nu_j)-OT_2^\epsilon(\mu^*,\hat{\nu}_j^{n_2})|\right)\\& \leq \sum_{j=1}^m\lambda_j \mathbb{E}\left(|OT_2^\epsilon(\mu^{n_1},\nu_j)-OT_2^\epsilon(\mu^{n_1},\hat{\nu}_j^{n_2})|\right)+\sum_{j=1}^m\lambda_j \mathbb{E}\left(|OT_2^\epsilon(\mu^{n_1}_{n_2},\nu_j)-OT_2^\epsilon(\mu^{n_1}_{n_2},\hat{\nu}_j^{n_2})|\right)\\&\leq \sum_{j=1}^m\lambda_j\frac{m^{1/2}P_{d,\epsilon}(\sigma)}{\sqrt{n_2}}+\sum_{j=1}^m\lambda_j \frac{m^{1/2}P_{d,\epsilon}(\sigma)}{\sqrt{n_2}} \\&=\frac{2m^{1/2}P_{d,\epsilon}(\sigma)}{\sqrt{n_2}},
\end{align*}
where we applied Lemma \ref{lemma:mena_weed_4} in the final inequality. Thus we have bounded 
\begin{equation*}
\mathbb{E}\left(\bigg|\min_{\mu\in\mathcal{G}_\sigma(\mathbb{R}^d)}S_{\lambda,\mathcal{V}}^\epsilon(\mu)-\min_{\mu\in\mathcal{P}^{n_1}(\mathbb{R}^d)}S_{\lambda,\hat{\mathcal{V}}^{n_2}}^\epsilon(\mu)\bigg|\right)\leq 2(m^{1/2}+1)P_{d,\epsilon}(\sigma)\left(\frac{1}{\sqrt{n_1}}+\frac{1}{\sqrt{n_2}}\right).
\end{equation*}
Taking the maximum over all constants as $R^*_{d,\epsilon,\sigma}$, we conclude.

\vspace{10pt}

\noindent\textbf{Proof of Theorem \ref{thm:synthsampcomplex}:} Let $\mu^*=\argmin_{\mu\in \mathcal{P}_2(\mathbb{R}^d)}F^\epsilon_{\lambda,\mathcal{V}}(\mu)$, and let $\mu^n$ denote a random minimizer of $F^\epsilon_{\lambda,\hat{\mathcal{V}^n}}$ over $\mathcal{P}^n(\mathbb{R}^d)$. We bound:
\begin{align*}
\mathbb{E}\left(\bigg|F^\epsilon_{\lambda,\mathcal{V}}(\mu^*)-F^\epsilon_{\lambda,\mathcal{V}}(\mu^n)\bigg| \right) & \leq\mathbb{E}\left(\bigg|F^\epsilon_{\lambda,\mathcal{V}}(\mu^*)-F^\epsilon_{\lambda,\hat{\mathcal{V}}^n}(\mu^n)\bigg| \right)+\mathbb{E}\left(\bigg|F^\epsilon_{\lambda,\hat{\mathcal{V}}^n}(\mu^n)-F^\epsilon_{\lambda,\mathcal{V}}(\mu^n)\bigg| \right).
\end{align*}

By Proposition \ref{prop:approx_min}, specifically (\ref{eqn:synthsampF_double}), the first term is upper bounded by $\frac{mR^*_{d,\epsilon,\sigma}}{\sqrt{n}}.$ We bound the second term:
\begin{align*}
\mathbb{E}\left(\bigg|F^\epsilon_{\lambda,\hat{\mathcal{V}}^n}(\mu^n)-F^\epsilon_{\lambda,\mathcal{V}}(\mu^n)\bigg| \right) & \leq \sum_{j=1}^m \lambda_j \mathbb{E}\left(|OT_2^\epsilon(\mu^n,\hat{\nu}^n_j)-OT^\epsilon_2(\mu^n,\nu_j)|\right)\\& \leq  \sum_{j=1}^m \lambda_j \frac{m^{1/2}P_{d,\epsilon}(\sigma)}{\sqrt{n}}\\& =\frac{m^{1/2}P_{d,\epsilon}(\sigma)}{\sqrt{n}}.
\end{align*}
Above, we applied the triangle inequality, followed by (\ref{eqn:mena weed 4 eqn 1}) from Lemma \ref{lemma:mena_weed_4} to each summand, which applies since $\|\mu^n\|_\mathcal{G}\leq \max_{1\leq j\leq m}\|\hat{\nu}_j^n\|_\mathcal{G}$ almost-surely. From this we conclude (\ref{eqn:thm3.1F}), with a final bound of $\frac{1}{\sqrt{n}}(m^{1/2}P_{d,\epsilon}(\sigma)+mR^*_{d,\epsilon,\sigma})$.

Similar arguments apply to $S^\epsilon_{\lambda,\mathcal{V}}$. By assumption, for $\mu^n\in\argmin_{\mu\in \mathcal{P}^n(\mathbb{R}^d)}S^\epsilon_{\lambda,\hat{\mathcal{V}}^n}(\mu)$ we have $\|\mu^n\|_\mathcal{G}\leq \max_{1\leq j\leq m}\{\|\hat{\nu}_{j}^n\|_\mathcal{G},\sigma\}=:\tilde{\sigma}.$  We bound:
\begin{align*}
\mathbb{E}\left(\bigg|S^\epsilon_{\lambda,\mathcal{V}}(\mu^*)-S^\epsilon_{\lambda,\mathcal{V}}(\mu^n)\bigg| \right) & \leq\mathbb{E}\left(\bigg|S^\epsilon_{\lambda,\mathcal{V}}(\mu^*)-S^\epsilon_{\lambda,\hat{\mathcal{V}}^n}(\mu^n)\bigg| \right)+\mathbb{E}\left(\bigg|S^\epsilon_{\lambda,\hat{\mathcal{V}}^n}(\mu^n)-S^\epsilon_{\lambda,\mathcal{V}}(\mu^n)\bigg| \right).
\end{align*}
We apply (\ref{eqn:alt_synthsampS_double}) to bound the first term by $\frac{mR^*_{d,\epsilon,\sigma}}{\sqrt{n}}$. We bound the second term:
\begin{align*}
&\mathbb{E}\left(\bigg|S^\epsilon_{\lambda,\hat{\mathcal{V}}^n}(\mu^n)-S^\epsilon_{\lambda,\mathcal{V}}(\mu^n)\bigg| \right) \\& \leq \sum_{j=1}^m \lambda_j \mathbb{E}\left(|OT_2^\epsilon(\mu^n,\hat{\nu}^n_j)-OT^\epsilon_2(\mu^n,\nu_j)|+\frac{1}{2}|OT_2^\epsilon(\nu_j,\nu_j)-OT_2^\epsilon(\hat{\nu}_j^n,\hat{\nu}_j^n)|\right) \\ & =\sum_{j=1}^m \lambda_j \mathbb{E}\left(|OT_2^\epsilon(\mu^n,\hat{\nu}^n_j)-OT^\epsilon_2(\mu^n,\nu_j)|\right)+\frac{1}{2}\sum_{j=1}^m\lambda_j\mathbb{E}\left(|OT_2^\epsilon(\nu_j,\nu_j)-OT_2^\epsilon(\hat{\nu}_j^n,\hat{\nu}_j^n)|\right)
\end{align*}
by the triangle inequality. To bound the first term, we apply (\ref{eqn:mena weed 4 eqn 1}) from Lemma \ref{lemma:mena_weed_4} to each summand, which applies since $\mu^n\in\mathcal{G}_{\tilde{\sigma}}(\mathbb{R}^d)$, giving us a bound of $\frac{m^{1/2}P_{d,\epsilon}(\sigma)}{\sqrt{n}}$ To bound the second term, we apply Lemma \ref{lemma:mena_weed_5}, giving a bound of $\frac{P_{d,\epsilon}(\sigma)}{\sqrt{n}}$.  Adding these bounds together, we have:
\begin{equation*}
\mathbb{E}\left(\bigg|S^\epsilon_{\lambda,\mathcal{V}}(\mu^*)-S^\epsilon_{\lambda,\mathcal{V}}(\mu^n)\bigg| \right)\leq \frac{1}{\sqrt{n}}(mR^*_{d,\epsilon,\sigma}+(m^{1/2}+1)P_{d,\epsilon}(\sigma)).
\end{equation*}
Taking the maximum over all constants depending on $d,\epsilon$ and $\sigma$ as $C^*_{d,\epsilon,\sigma},$ we conclude. 
\qed

\section{PROOFS FOR SECTION \ref{sec:analysis_problem}}\label{sec:proof_of_coeff_thm}

\textbf{Proof of Proposition \ref{prop:analysis_grad}:} The assumption that $\mathcal{V}\subset \mathcal{G}(\mathbb{R}^d)$ ensures that $\nabla \delta F_{\lambda,\mathcal{V}}^\epsilon(\mu)$ exists and equals $Id-\sum_{j=1}^m\lambda_j T^\epsilon_{\mu\rightarrow \nu_j}$. We compute:
\begin{align*}
\|\nabla\delta F^\epsilon_{\lambda,\mathcal{V}}(\mu)\|_{L^2(\mu)}^2 & = \| Id -\sum_{j=1}^m \lambda_jT^\epsilon_{\mu\rightarrow \nu_j}\|^2_{L^2(\mu)} & \\ & = \int \langle Id-\sum_{j=1}^m \lambda_j T^\epsilon_{\mu\rightarrow \nu_j}, Id - \sum_{k=1}^m\lambda_k T^\epsilon_{\mu\rightarrow \nu_k}\rangle d\mu & \\ & = \int \sum_{j,k=1}^m \lambda_j\lambda_k\langle Id-T^\epsilon_{\mu\rightarrow \nu_j},Id-T^\epsilon_{\mu\rightarrow \nu_k}\rangle d\mu &\\& =\sum_{j,k=1}^m \lambda_j\lambda_k [A_\mu^\epsilon]_{jk}=\lambda^\top A_{\mu}^\epsilon \lambda.
\end{align*}
Similarly:
\begin{align*}
\|\nabla\delta S^\epsilon_{\lambda,\mathcal{V}}(\tilde{\mu})\|_{L^2(\tilde{\mu})}^2 & = \| \entmap{\tilde{\mu}}{\tilde{\mu}} -\sum_{j=1}^m \lambda_jT^\epsilon_{\tilde{\mu}\rightarrow \nu_j}\|^2_{L^2(\tilde{\mu})} & \\ & = \int \langle \entmap{\tilde{\mu}}{\tilde{\mu}}-\sum_{j=1}^m \lambda_j T^\epsilon_{\tilde{\mu}\rightarrow \nu_j}, \entmap{\tilde{\mu}}{\tilde{\mu}} - \sum_{k=1}^m\lambda_k T^\epsilon_{\tilde{\mu}\rightarrow \nu_k}\rangle d\tilde{\mu} & \\ & = \int \sum_{j,k=1}^m \lambda_j\lambda_k\langle \entmap{\tilde{\mu}}{\tilde{\mu}}-T^\epsilon_{\tilde{\mu}\rightarrow \nu_j},\entmap{\tilde{\mu}}{\tilde{\mu}}-T^\epsilon_{\tilde{\mu}\rightarrow \nu_k}\rangle d\tilde{\mu} &\\& =\sum_{j,k=1}^m \lambda_j\lambda_k [S_{\tilde{\mu}}^\epsilon]_{jk}=\lambda^\top S_{\tilde{\mu}}^\epsilon \lambda.
\end{align*}
Hence we see that $\lambda^\top A^\epsilon_\mu\lambda\geq 0,$ and $\lambda^\top A_\mu^\epsilon\lambda=0$ is equivalent to $\|\nabla\delta F^\epsilon_{\lambda,\mathcal{V}}(\mu)\|^2_{L^2(\mu)}=0$, which is itself equivalent to $\nabla\delta F^\epsilon_{\lambda,\mathcal{V}}(\mu)=0$ $\mu$-a.e. The analogous result holds for $S^\epsilon_{\tilde{\mu}}$.  Finally, if $\texttt{supp}(\mu)=\mathbb{R}^d$ (resp. $\texttt{supp}(\tilde{\mu})=\mathbb{R}^d$), then we may apply the optimality criteria from Corollary \ref{cor:optimality}.

\qed 

\subsection{Proof of Theorem \ref{thm:coeff_theorem}} 

\begin{lemma}\label{lem:ent_mat_bound}
 Let $\mu\in \mathcal{P}_2(\Omega)$ and $\mathcal{V}\subset \mathcal{G}_\sigma(\Omega)$, and suppose that $A_\mu^\epsilon$ has an eigenvalue of 0 with unique eigenvector $\lambda_*\in \Delta^{m}$. Suppose there exists an estimator $\hat{T}(\hat{\mu}^n,\hat{\nu}^n_j)$ satisfying (\ref{eqn:thetabound1}) for $1\leq j\leq m$. Then: \begin{equation*} \mathbb{E}(|[A_\mu^\epsilon]_{ij}-[\hat{A}^\epsilon_\mu]_{ij}|)\leq J d\sigma^2\max\left\{\frac{1}{\sqrt{n}},\sqrt{\theta(n)}+\sqrt{\theta(n)+\theta(n)^2}\right\}\end{equation*}
where $\hat{A}_\mu$ is the matrix $\hat{M}_\mu$ in Algorithm \ref{alg:coefficient_recovery} with $\mathcal{F}^\epsilon_{\lambda,\mathcal{V}}=F^\epsilon_{\lambda,\mathcal{V}}$, and $J$ is an absolute constant.
\end{lemma}
\begin{proof}
 For ease of notation, we will write $T^\epsilon_j$ for $T^\epsilon_{\mu\rightarrow\nu_j}$ and $\hat{T_j^\epsilon}$ for $\hat{T}(\hat{\mu}^n,\hat{\nu}^n_j)$. We begin by bounding $\mathbb{E}(|[A_\mu^\epsilon]_{ij}-[\hat{A}^\epsilon_\mu]_{ij}|)$:
\begin{align}
&\mathbb{E}(|[A_\mu^\epsilon]_{ij}-[\hat{A}^\epsilon_\mu]_{ij}|)\notag \\ 
= &\mathbb{E}\left(\bigg|\int \langle T^\epsilon_j-Id,T^\epsilon_i-Id\rangle d\mu - \frac{1}{n}\sum_{k=n+1}^{2n}\langle \hat{T_j^\epsilon}(X_k)-X_k,\hat{T_i^\epsilon}(X_k)-X_k\rangle\bigg|\right)\notag  \\
\leq  &\mathbb{E}\left(\bigg|\int \langle T^\epsilon_j-Id,T^\epsilon_i-Id\rangle d\mu  - \frac{1}{n}\sum_{k=n+1}^{2n}\langle T^\epsilon_j(X_k)-X_k,T^\epsilon_i(X_k)-X_k\rangle   \bigg|\right)\label{eqn:analysis_split1} \\
 + &\mathbb{E}\left(\bigg|\frac{1}{n}\sum_{k=n+1}^{2n}\langle T^\epsilon_j(X_k)-X_k, T^\epsilon_i(X_k)-X_k\rangle - \frac{1}{n}\sum_{k=n+1}^{2n}\langle \hat{T^\epsilon_j}(X_k)-X_k, \hat{T_i^\epsilon}(X_k)-X_k\rangle \bigg|\right) \label{eqn:analysis_split2}
\end{align}
The term (\ref{eqn:analysis_split1}) may be bounded by observing that $\langle T^\epsilon_j(X_k)-X_k, T^\epsilon_i(X_k)-X_k\rangle:=Z_k$ are i.i.d. samples from the random variable $\langle T^\epsilon_j(X) - X, T^\epsilon_i(X)-X\rangle:=Z$ with $X\sim \mu$. Using Jensen's inequality and the independence of $Z_{n+1},Z_{n+2},...,Z_{2n}$, we bound: 
\begin{align*}
\mathbb{E}_{Z_{n+1},...,Z_{2n}}\left(\left|\mathbb{E}_Z(Z)-\frac{1}{n}\sum_{k=n+1}^{2n} Z_k\right|\right) & = \mathbb{E}_{Z_{n+1},...,Z_{2n}}\left(\left|\frac{1}{n}\sum_{k=n+1}^{2n}(\mathbb{E}_Z(Z)-Z_k)\right|\right)\notag  \\ & \leq \sqrt{\mathbb{E}_{Z_{n+1},...,Z_{2n}}\left(\left|\frac{1}{n}\sum_{k=n+1}^{2n}(\mathbb{E}_Z(Z)-Z_k)\right|^2\right)} \notag\\ & = \sqrt{\frac{1}{n^2}\mathbb{E}_{Z_{n+1},...,Z_{2n}}\left(\sum_{k=n+1}^{2n}\sum_{k'=n+1}^{2n} (\mathbb{E}_Z(Z)-Z_k) (\mathbb{E}_Z(Z)-Z_{k'})\right)} \\ & = \sqrt{\frac{1}{n^2}\mathbb{E}_{Z_{n+1},...,Z_{2n}}\left(\sum_{k=n+1}^{2n}(\mathbb{E}_Z(Z)-Z_k)^2\right)}  \\ &=\sqrt{\frac{\mathbb{E}_{Z_{n+1}}((\mathbb{E}_Z(Z)-Z_{n+1})^2)}{n}}\notag 
\end{align*}
where the outer expectation in the last line is over the sample $Z_{n+1}$. Observe that $\mathbb{E}_{Z_{n+1}}(\mathbb{E}_{Z}(Z)-Z_{n+1})^2=\mathbb{E}_{Z}(\mathbb{E}_{Z}(Z)-Z)^2$ is the variance of the random variable $Z$. Hence we may bound:
\begin{align}
&\mathbb{E}_Z((\mathbb{E}_Z(Z)-Z)^2) \\& \leq \mathbb{E}_Z(Z^2)\notag\\ &  = \mathbb{E}(\langle T_i^\epsilon(X)-X,T_j^\epsilon(X)-X\rangle^2) \notag\\& \leq \mathbb{E}(\|T_i^\epsilon(X)-X\|^2\|T_j^\epsilon(X)-X\|^2)\notag 
\\&\leq \left(\mathbb{E}(\|T_i^\epsilon(X)-X\|^4)\mathbb{E}(\|T_j^\epsilon(X)-X\|^4)\right)^{1/2}\notag\\& \leq (\mathbb{E}((\|T_i^\epsilon(X)\|+\|X\|)^4)\mathbb{E}((\|T_j^\epsilon(X)\|+\|X\|)^4))^{1/2}\notag\\ & = \left(2^4\mathbb{E}\left(\left(\frac{\|T_i^\epsilon(X)\|}{2}+\frac{\|X\|}{2}\right)^4\right)\cdot 2^4\mathbb{E}\left(\left(\frac{\|T_j^\epsilon(X)\|}{2}+\frac{\|X\|}{2}\right)^4\right)\right)^{1/2}\notag\\& \leq \left(2^8\mathbb{E}\left(\frac{1}{2}(\|T_i^\epsilon(X)\|^4+\|X\|^4)\right)\mathbb{E}\left(\frac{1}{2}(\|T_j^\epsilon(X)\|^4+\|X\|^4)\right)\right)^{1/2} \notag\\& =\left(2^6(\mathbb{E}(\|T_i^\epsilon(X)\|^4)+\mathbb{E}(\|X\|^4))( \mathbb{E}(\|T_j^\epsilon(X)\|^4]+\mathbb{E}(\|X\|^4))\right)^{1/2}\notag\\& = 8\left((\mathbb{E}(\|T_i^\epsilon(X)\|^4)+\mathbb{E}(\|X\|^4))( \mathbb{E}(\|T_j^\epsilon(X)\|^4)+\mathbb{E}(\|X\|^4))\right)^{1/2}\label{eqn:fourth moments}
\end{align}
where we applied Jensen's inequality, two applications of the Cauchy-Schwarz inequality, the triangle inequality, and Jensen's inequality. By Lemma \ref{lemma:10.5yang}, we have that 
\begin{equation*}
\mathbb{E}_{X\sim\mu}(\|T_i^\epsilon(X)\|^4) \leq \mathbb{E}_{Y\sim \nu_i}(\|Y\|^4),\;\;\; \forall 1\leq i\leq m.
\end{equation*}
Furthermore, by Proposition \ref{prop:analysis_grad}, $\mu$ is a critical point, we may apply Proposition \ref{prop:domination} to bound
\begin{equation*}
\mathbb{E}_{X\sim\mu}(\|X\|^4) \leq \mathbb{E}_{Y\sim \sum_{j=1}^m\lambda_j\nu_j}(\|Y\|^4)
\end{equation*}
and combining these we have upper bounded (\ref{eqn:fourth moments}) by $16\max_k \mathbb{E}_{Y\sim \nu_k}(\|Y\|^4).$ As $\mathcal{V}\subset \mathcal{G}_\sigma(\Omega),$ by Lemma \ref{lemma:normsubgauss}, $\|Y\|$ is $q_G\sqrt{d}\sigma$-subgaussian. Hence by applying item 3 in Definition \ref{def:subgauss}, we have the upper bound 
\begin{align*}
16\max_k \mathbb{E}_{Y\sim \nu_k}(\|Y\|^4)&\leq 16 (2C_Gq_G\sqrt{d}\sigma)^4 \\& =2^{8} C_G^4q_G^4 d^2\sigma^4
\end{align*}
 
Hence we have upper bounded \ref{eqn:analysis_split1}:
\begin{align} &\mathbb{E}\left(|\mathbb{E}(Z)-\frac{1}{n}\sum_{k=n+1}^{2n} Z_k|\right)\leq \frac{16C_G^2q_G^2d\sigma^2}{\sqrt{n}}\label{eqn:analysismoments}.
\end{align}
We now bound (\ref{eqn:analysis_split2}):
\begin{align*}
&\mathbb{E}\left(\bigg|\frac{1}{n}\sum_{k=n+1}^{2n}\langle T^\epsilon_i(X_k)-X_k, T^\epsilon_j(X_k)-X_k\rangle - \frac{1}{n}\sum_{k=n+1}^{2n}\langle \hat{T^\epsilon_i}(X_k)-X_k, \hat{T_j^\epsilon}(X_k)-X_k\rangle \bigg|\right) \\
\leq &\frac{1}{n}\sum_{k=n+1}^{2n}\mathbb{E}\left(\bigg|\langle T^\epsilon_i(X_k)-X_k,T^\epsilon_j(X_k)-X_k \rangle - \langle \hat{T^\epsilon_i}(X_k)-X_k,\hat{T_j^\epsilon}(X_k)-X_k\rangle \bigg|\right)\\ 
= &\mathbb{E}\left(\bigg|\langle T^\epsilon_i(X_k)-X_k,T^\epsilon_j(X_k)-X_k \rangle - \langle \hat{T^\epsilon_i}(X_k)-X_k,\hat{T^\epsilon_j}(X_k)-X_k\rangle \bigg|\right)\end{align*}
for an arbitrary $k\in \{n+1,n+2,...,2n\}$, where we have used the triangle inequality and the fact that $X_k$ are identically distributed.  Now, 
\begin{align}
&\mathbb{E}\left(\bigg|\langle T^\epsilon_i(X_k)-X_k,T^\epsilon_j(X_k)-X_k \rangle - \langle \hat{T^\epsilon_i}(X_k)-X_k,\hat{T_j^\epsilon}(X_k)-X_k\rangle \bigg|\right)\notag \\
\leq &\mathbb{E}\left(\bigg| \langle T^\epsilon_i(X_k)-X_k,T^\epsilon_j(X_k)-X_k \rangle - \langle \hat{T^\epsilon_i}(X_k)-X_k,T^\epsilon_j(X_k)-X_k\rangle\bigg|\right)\notag \\
+&\mathbb{E}\left(\bigg| \langle \hat{T^\epsilon_i}(X_k)-X_k,T^\epsilon_j(X_k)-X_k \rangle-\langle \hat{T^\epsilon_i}(X_k)-X_k,\hat{T_j^\epsilon}(X_k)-X_k\rangle\bigg|\right)\notag \\
=&\mathbb{E}\left(\bigg| \langle T^\epsilon_i(X_k)-\hat{T^\epsilon_i}(X_k),T^\epsilon_j(X_k)-X_k \rangle\bigg|\right) + \mathbb{E}\left(\bigg| \langle \hat{T^\epsilon_i}(X_k)-X_k,T^\epsilon_j(X_k)-\hat{T_j^\epsilon}(X_k) \rangle\bigg|\right)\notag \\
\leq &\mathbb{E}\left(\|T^\epsilon_i(X_k)-\hat{T^\epsilon_i}(X_k)\|\;\| T^\epsilon_j(X_k)-X_k  \|\right)+\mathbb{E}\left(\|\hat{T^\epsilon_i}(X_k)-X_k\|\;\|T^\epsilon_j(X_k)-\hat{T_j^\epsilon}(X_k)\| \right)\notag \\  \leq & \sqrt{\mathbb{E}\left(\|T^\epsilon_i(X_k)-\hat{T^\epsilon_i}(X_k)\|^2\right)\mathbb{E}\left(\| T^\epsilon_j(X_k)-X_k  \|^2\right)}+\sqrt{\mathbb{E}\left(\|\hat{T^\epsilon_i}(X_k)-X_k\|^2\right)\mathbb{E}\left(\|T^\epsilon_j(X_k)-\hat{T_j^\epsilon}(X_k)\|^2 \right)}\notag \\ \leq & \sqrt{2\mathbb{E}\left(\|T^\epsilon_i(X_k)-\hat{T^\epsilon_i}(X_k)\|^2\right)\mathbb{E}\left(\| T^\epsilon_j(X_k)\|^2+\|X_k  \|^2\right)} + \sqrt{\mathbb{E}\left(\|\hat{T^\epsilon_i}(X_k)-X_k\|^2\right)\mathbb{E}\left(\|T^\epsilon_j(X_k)-\hat{T_j^\epsilon}(X_k)\|^2 \right)}\notag \\ \leq & \sqrt{2\mathbb{E}\left(\|T^\epsilon_i(X_k)-\hat{T^\epsilon_i}(X_k)\|^2\right)\left(M_2(\nu_j)+M_2(\mu)\right)} + \sqrt{\mathbb{E}\left(\|\hat{T^\epsilon_i}(X_k)-X_k\|^2\right)\mathbb{E}\left(\|T^\epsilon_j(X_k)-\hat{T_j^\epsilon}(X_k)\|^2 \right)}\label{eqn:lastlineanalysis}
\end{align}

In the above, we applied the triangle inequality, two applications of the Cauchy-Schwarz inequality, the triangle inequality, Young's inequality and Lemma \ref{lemma:10.5yang} to bound $\mathbb{E}\left(\|T_j^\epsilon(X_k)\|^2\right)\leq M_2(\nu_j)$. We now control:
\begin{align*}
& \mathbb{E}\left(\|\hat{T}^\epsilon_i(X_k)-X_k\|^2\right)\notag \\  
\leq&\mathbb{E}\left((\|\hat{T}^\epsilon_i(X_k)-T_i^\epsilon(X_k)\|+\|T^\epsilon_i(X_k)-X_k\|)^2\right)\notag \\ 
\leq &2\mathbb{E}\left(\|\hat{T}^\epsilon_i(X_k)-T_i^\epsilon(X_k)\|^2+\|T^\epsilon_i(X_k)-X_k\|^2\right)\notag \\
 \leq &2\mathbb{E}\left(\|\hat{T}^\epsilon_i(X_k)-T_i^\epsilon(X_k)\|^2\right)+4\mathbb{E}\left(\|T^\epsilon_i(X_k)\|^2+\|X_k\|^2\right) \notag\\ 
  = &2\mathbb{E}\left(\|\hat{T}^\epsilon_i(X_k)-T_i^\epsilon(X_k)\|^2\right)+4M_2(\nu_i)+4M_2(\mu)\label{eqn:trivial}
\end{align*}
where we applied the triangle inequality and Young's inequality twice, and the fact that  $\mathbb{E}\left(\|T_i^\epsilon(X_k)\|^2\right)\leq M_2(\nu_i)$. Since $\mu$ is a critical point of $F^\epsilon_{\lambda,\mathcal{V}}$, we may apply Proposition \ref{prop:domination} to further bound $M_2(\mu)\leq \max_{1\leq k \leq m}M_2(\nu_k)$. Summarizing, we have bounded (\ref{eqn:lastlineanalysis}) by:
\begin{align*}
&\sqrt{2\mathbb{E}\left(\|T^\epsilon_i(X_k)-\hat{T^\epsilon_i}(X_k)\|^2\right)\left(M_2(\nu_j)+\max_{1\leq k \leq m}M_2(\nu_j)\right)} \\  + &\sqrt{\left(2\mathbb{E}\left(\|\hat{T}^\epsilon_i(X_k)-T_i^\epsilon(X_k)\|^2\right)+4M_2(\nu_i)+4\max_{1\leq k \leq m}M_2(\nu_k)\right)\mathbb{E}\left(\|T^\epsilon_j(X_k)-\hat{T_j^\epsilon}(X_k)\|^2 \right)}\\\leq &\sqrt{4C_G^2q_G^2d\sigma^2\mathbb{E}(\|T_i^\epsilon(X_k)-\hat{T}^\epsilon_i(X_k)\|^2)}\\+&\sqrt{2\left(\mathbb{E}\left(\|\hat{T}^\epsilon_i(X_k)-T_i^\epsilon(X_k)\|^2\right)+16C_G^2q_G^2d\sigma^2\right)\mathbb{E}(\|T_j^\epsilon(X_k)-\hat{T}^\epsilon_j(X_k)\|^2)}
\end{align*}
where we have applied Lemma \ref{lemma:normsubgauss} followed by item 3 in Definition \ref{def:subgauss} to upper bound the second moments.

Since $X_k\sim \mu$ for $k\geq n+1$ is independent of the samples $X_1,...,X_n$ used to compute the maps $\hat{T^\epsilon_i}$ and $\hat{T_j^\epsilon}$, we may apply (\ref{eqn:thetabound1}) to obtain the upper bound (\ref{eqn:lastlineanalysis}) by:
\begin{equation*}
2C_Gq_G\sqrt{d}\sigma\sqrt{\theta(n)}+\sqrt{2\theta(n)^2+16C_G^2q_G^2d\sigma^2\theta(n)}.
\end{equation*}
Combining this with the bound on (\ref{eqn:analysis_split1}), we have
\begin{align}\label{eqn:entrybound}
\mathbb{E}\left(\bigg|[A_\mu^\epsilon]_{ij}-[\hat{A}_\mu^\epsilon]_{ij}\bigg|\right)& \leq J d\sigma^2\max\left\{\frac{1}{\sqrt{n}},\sqrt{\theta(n)}+\sqrt{\theta(n)+\theta(n)^2}\right\},
\end{align}
where $J$ is an absolute constant.
\end{proof}
We prove a similar lemma for $S^\epsilon_{\mu}:$

\begin{lemma}\label{lemma:sinkmatbound}
 Let $\mu\in \mathcal{G}_\sigma(\Omega)$ and $\mathcal{V}\subset \mathcal{G}_\sigma(\Omega)$, and suppose that $S_\mu^\epsilon$ has an eigenvalue of 0 with unique eigenvector $\lambda_*\in \Delta^{m}$. Suppose there exists an estimator $\hat{T}(\hat{\mu}^n,\hat{\nu}^n_j)$ satisfying (\ref{eqn:thetabound1}) for all $1\leq j\leq m$, and an estimator $\hat{T}(\hat{\mu}^n)$ satisfying (\ref{eqn:thetabound2}). Then: \begin{equation*} \mathbb{E}(|[S_\mu^\epsilon]_{ij}-[\hat{S}^\epsilon_\mu]_{ij}|)\leq \tilde{J} d\sigma^2\max\left\{\frac{1}{\sqrt{n}},\sqrt{\theta(n)}+\sqrt{\theta(n)+\theta(n)^2}\right\}\end{equation*}
where $\hat{S}_\mu$ is the matrix $\hat{M}_\mu$ in Algorithm \ref{alg:coefficient_recovery} with $\mathcal{F}^\epsilon_{\lambda,\mathcal{V}}=S^\epsilon_{\lambda,\mathcal{V}}$, and $\tilde{J}$ is an absolute constant.
\end{lemma}
\begin{proof}
 For ease of notation, we write $T^\epsilon_j:=T^\epsilon_{\mu\rightarrow\nu_j}$,  $\hat{T_j^\epsilon}:= \hat{T}(\hat{\mu}^n,\hat{\nu}^n_j)$, $T^\epsilon:=T^\epsilon_{\mu\rightarrow \mu}$ and $\hat{T}^\epsilon:=\hat{T}(\hat{\mu}^n)$.  We bound:
\begin{align}
& \mathbb{E}\left(\bigg|\int \langle T_i^\epsilon - T^\epsilon, T_j^\epsilon-T^\epsilon\rangle d\mu - \frac{1}{n}\sum_{k=n+1}^{2n}\langle \hat{T}_i^\epsilon(X_k)-\hat{T}^\epsilon(X_k),\hat{T}_j^\epsilon(X_k)-\hat{T}^\epsilon(X_k)\rangle\bigg|\right)\notag\\ 
\leq &\mathbb{E}\left(\bigg|\int \langle T_i^\epsilon-T^\epsilon,T_j^\epsilon -T^\epsilon\rangle d\mu -\frac{1}{n}\sum_{k=n+1}^{2n}\langle T_i^\epsilon(X_k)-T^\epsilon(X_k),T_j^\epsilon(X_k)-T^\epsilon(X_k)\rangle \bigg|\right)\notag
\\
+&\mathbb{E}\left(\bigg|\frac{1}{n}\sum_{k=n+1}^{2n}\langle T_i^\epsilon(X_k)-T^\epsilon(X_k),T_j^\epsilon(X_k)-T^\epsilon(X_k)\rangle-\frac{1}{n}\sum_{k=n+1}^{2n}\langle \hat{T}_i^\epsilon(X_k)-T^\epsilon(X_k),\hat{T}_j^\epsilon(X_k)-T^\epsilon(X_k)\rangle \bigg|\right)\notag \\ + &\mathbb{E}\left(\bigg| \frac{1}{n}\sum_{k=n+1}^{2n}\langle \hat{T}_i^\epsilon(X_k)-T^\epsilon(X_k),\hat{T}_j^\epsilon(X_k)-T^\epsilon(X_k)\rangle  
 - \frac{1}{n}\sum_{k=n+1}^{2n}\langle \hat{T}_i^\epsilon(X_k)-\hat{T}^\epsilon(X_k),\hat{T}_j^\epsilon(X_k)-\hat{T}^\epsilon(X_k)\rangle \bigg|\right).\label{eqn:sinkanalysis_1}
\end{align}
By the same sequence of steps as used to bound (\ref{eqn:analysis_split1}) by (\ref{eqn:fourth moments}), we may upper bound the first term in (\ref{eqn:sinkanalysis_1}) by
\begin{equation}\frac{1}{\sqrt{n}}\sqrt{8\left((\mathbb{E}(\|T^\epsilon_i(X)\|^4)+\mathbb{E}(\|T^\epsilon(X)\|^4))( \mathbb{E}(\|T^\epsilon_j(X)\|^4)+\mathbb{E}(\|T^\epsilon(X)\|^4))\right)^{1/2}}.\label{eqn:fourthmoment_sink}\end{equation} %
By Lemma \ref{lemma:10.5yang}, $(T^\epsilon_i)_{\#}(\mu)\leq_c \nu_i\in \mathcal{G}_\sigma(\Omega)$ for all $1\leq i \leq m$, and $(T^\epsilon)_{\#}(\mu)\leq_c \mu\in \mathcal{G}_\sigma(\Omega)$. Hence we may apply Lemma \ref{lemma:subgauss dom}, Lemma \ref{lemma:normsubgauss} and item 3 in Definition \ref{def:subgauss} as in the proof of Lemma \ref{lem:ent_mat_bound} to conclude that (\ref{eqn:fourthmoment_sink}) is upper bounded by $\frac{16C_G^2q_G^2d\sigma^2}{\sqrt{n}}$.

The second term in (\ref{eqn:sinkanalysis_1}) may be bounded in the same manner as (\ref{eqn:analysis_split2}). By the triangle inequality and noting that $X_1, X_2,\dots, X_{2n}$ are all i.i.d. we have:
\begin{align*}
& \mathbb{E}\left(\bigg|\frac{1}{n}\sum_{k=n+1}^{2n}\langle T_i^\epsilon(X_k)-T^\epsilon(X_k),T_j^\epsilon(X_k)-T^\epsilon(X_k)\rangle-\frac{1}{n}\sum_{k=n+1}^{2n}\langle \hat{T}_i^\epsilon(X_k)-T^\epsilon(X_k),\hat{T}_j^\epsilon(X_k)-T^\epsilon(X_k)\rangle \bigg|\right)\\ \leq & \frac{1}{n}\sum_{k={n+1}}^{2n}\mathbb{E}\left(|\langle T_i^\epsilon(X_k)-T^\epsilon(X_k),T_j^\epsilon(X_k)-T^\epsilon(X_k)\rangle -\langle \hat{T}^\epsilon_i(X_k)-T^\epsilon(X_k),\hat{T}^\epsilon_j(X_k)-T^\epsilon(X_k)\rangle |\right)\\ \leq &\mathbb{E}\left(\bigg|\langle T^\epsilon_i(X)-T^\epsilon(X),T^\epsilon_j(X)-T^\epsilon(X) \rangle - \langle \hat{T^\epsilon_i}(X)-T^\epsilon(X),\hat{T^\epsilon_j}(X)-T^\epsilon(X)\rangle \bigg|\right)\notag\end{align*}
where $X\sim \mu$ is independent of $X_1,X_2,...,X_n$. By the triangle inequality, we may upper bound this by: 
\begin{align*}
&\mathbb{E}\left(\bigg| \langle T^\epsilon_i(X)-T^\epsilon(X),T^\epsilon_j(X)-T^\epsilon(X) \rangle - \langle \hat{T^\epsilon_i}(X)-T^\epsilon(X),T^\epsilon_j(X)-T^\epsilon(X)\rangle\bigg|\right)\notag \\
+ &\mathbb{E}\left(\bigg| \langle \hat{T^\epsilon_i}(X)-T^\epsilon(X),T^\epsilon_j(X)-T^\epsilon(X) \rangle-\langle \hat{T^\epsilon_i}(X)-T^\epsilon(X),\hat{T_j^\epsilon}(X)-T^\epsilon(X)\rangle\bigg|\right)\notag \\ \leq & \mathbb{E}\left(\bigg|\langle T_i^\epsilon(X)-\hat{T_i^\epsilon}(X), T_j^\epsilon(X)-T^\epsilon(X)\rangle\bigg|\right) + \mathbb{E}\left(\bigg|\langle \hat{T^\epsilon_i}(X)-T^\epsilon(X),T_j^\epsilon(X)-\hat{T_j^\epsilon}(X)\rangle\bigg|\right)\notag \\ \leq & \mathbb{E}\left(\| T_i^\epsilon(X)-\hat{T_i^\epsilon}(X)\|\| T_j^\epsilon(X)-T^\epsilon(X)\|\right) + \mathbb{E}\left(\| \hat{T^\epsilon_i}(X)-T^\epsilon(X)\| \|T_j^\epsilon(X)-\hat{T_j^\epsilon}(X)\|\right) \notag\\ \leq &\mathbb{E}\left(\| T_i^\epsilon(X)-\hat{T_i^\epsilon}(X)\|\left(\| T_j^\epsilon(X)\|+\|T^\epsilon(X)\|\right)\right) + \mathbb{E}\left(\| \hat{T^\epsilon_i}(X)-T^\epsilon(X)\| \|T_j^\epsilon(X)-\hat{T_j^\epsilon}(X)\|\right)\notag 
\\ \leq &\sqrt{\mathbb{E}\left(\|T^\epsilon_i(X)-\hat{T^\epsilon_i}(X)\|^2\right)\left(2\mathbb{E}(\|T_j^\epsilon(X)\|^2)+2\mathbb{E}(\|T^\epsilon(X)\|^2)\right)}\notag \\
& \;\;\;+\sqrt{\mathbb{E}\left(\|\hat{T^\epsilon_i}(X)-T^\epsilon(X)\|^2\right)\mathbb{E}\left(\|T^\epsilon_j(X)-\hat{T_j^\epsilon}(X)\|^2 \right)}\notag  \\\leq  &\sqrt{2\theta(n)\left(M_2(\nu_j)+\mathbb{E}(\|T^\epsilon(X)\|^2)\right)}\notag \\
& \;\;\;+\sqrt{\mathbb{E}\left(\|\hat{T^\epsilon_i}(X)-T^\epsilon(X)\|^2\right)\mathbb{E}\left(\|T^\epsilon_j(X)-\hat{T_j^\epsilon}(X)\|^2 \right)}\notag\\ \leq  &\sqrt{2\theta(n)\left(M_2(\nu_j)+\mathbb{E}(\|T^\epsilon(X)\|^2)\right)}\notag \\
& \;\;\;+\sqrt{2\mathbb{E}\left(\|\hat{T^\epsilon_i}(X)-T^\epsilon_i(X)\|^2+\|T^\epsilon_i(X)-T^\epsilon(X)\|^2\right)\mathbb{E}\left(\|T^\epsilon_j(X)-\hat{T_j^\epsilon}(X)\|^2 \right)}\notag\\ \leq  &\sqrt{2\theta(n)\left(M_2(\nu_j)+\mathbb{E}(\|T^\epsilon(X)\|^2)\right)}\notag \\
& \;\;\;+\sqrt{2\mathbb{E}\left(\|\hat{T^\epsilon_i}(X)-T^\epsilon_i(X)\|^2+2\|T^\epsilon_i(X)\|^2 +2\|T^\epsilon(X)\|^2\right)\mathbb{E}\left(\|T^\epsilon_j(X)-\hat{T_j^\epsilon}(X)\|^2 \right)}\notag\\  \leq  & \sqrt{2\theta(n)\left(M_2(\nu_j)+M_2(\mu)\right)}\notag \\
& \;\;\;+\sqrt{\left(2\mathbb{E}\left(\|\hat{T}^\epsilon_i(X)-T_i^\epsilon(X_n)\|^2\right)+4M_2(\nu_\ell)+4M_2(\mu)\right)\mathbb{E}\left(\|T^\epsilon_j(X)-\hat{T_j^\epsilon}(X)\|^2 \right)}\label{eqn:trivial_application} \\
\leq & \max_{1\leq \ell \leq m}\sqrt{2(M_2(\nu_j)+M_2(\mu))\theta(n)}+\sqrt{(2\theta(n)+4M_2(\nu_\ell)+4M_2(\mu))\theta(n)}.\notag
\end{align*}
By the triangle inequality and using the fact that $X_k, k \in \{1, \cdots, 2n \}$ are all i.i.d, the third term in (\ref{eqn:sinkanalysis_1}) is upper bounded by:
\begin{align*}
&\mathbb{E}\left(\left |\langle \hat{T}^\epsilon_i(X)-T^\epsilon(X),\hat{T}^\epsilon_j(X)-T^\epsilon(X)\rangle -\langle \hat{T}^\epsilon_i(X)-\hat{T}^\epsilon(X),\hat{T}^\epsilon_j(X)-\hat{T}^\epsilon(X)\rangle \right|\right) \\  =&\mathbb{E}( |\langle \hat{T}^\epsilon_i(X)-T^\epsilon(X),\hat{T}^\epsilon_j(X)-T^\epsilon(X)\rangle -\langle \hat{T}^\epsilon_i(X)-\hat{T}^\epsilon(X),\hat{T}^\epsilon_j(X)-T^\epsilon(X)\rangle  \\ & \; \; \; + \langle \hat{T}^\epsilon_i(X)-\hat{T}^\epsilon(X),\hat{T}^\epsilon_j(X)-T^\epsilon(X)\rangle -\langle \hat{T}^\epsilon_i(X)-\hat{T}^\epsilon(X),\hat{T}^\epsilon_j(X)-\hat{T}^\epsilon(X)\rangle|) \end{align*}
where $X\sim \mu$ is independent of $X_1,X_2,...,X_n$. By the triangle inequality, we may bound this as:
\begin{align*}& \mathbb{E}\left(\left| \langle \hat{T}^\epsilon(X)-T^\epsilon(X), \hat{T}^\epsilon_j(X)-T^\epsilon(X)\rangle\right|\right) + \mathbb{E}\left(\left|\langle \hat{T}^\epsilon_i(X)-\hat{T}^\epsilon(X),\hat{T}^\epsilon(X)-T^\epsilon(X)\right|\right)\\ & \leq \mathbb{E}\left(\|\hat{T}^\epsilon(X)-T^\epsilon(X)\| \|\hat{T}^\epsilon_j(X)-T^\epsilon(X)\|\right)+\mathbb{E}\left(\|\hat{T}^\epsilon_i(X)-\hat{T}^\epsilon(X)\| \|\hat{T}^\epsilon(X)-T^\epsilon(X)\|\right) \\ & = \mathbb{E}\left(\|\hat{T}^\epsilon(X)-T^\epsilon(X)\| \|\hat{T}^\epsilon_j(X)-T^\epsilon_j(X) + T^\epsilon_j(X)-T^\epsilon(X)\|\right) \\ & \; \; \; \; \;+\mathbb{E}\left(\|\hat{T}^\epsilon_i(X)-T^\epsilon_i(X)+T^\epsilon_i(X)-\hat{T}^\epsilon(X)\| \|\hat{T}^\epsilon(X)-T^\epsilon(X)\|\right) \\ &\leq \mathbb{E}\left(\|\hat{T}^\epsilon(X)-T^\epsilon(X)\| \left(\|\hat{T}^\epsilon_j(X)-T^\epsilon_j(X)\|+\|T^\epsilon_j(X)-T^\epsilon(X)\|\right)\right) \\& \; \; \; \; +\mathbb{E}\left(\left(\|\hat{T}^\epsilon_i(X)-T^\epsilon_i(X)\|+\|T^\epsilon_i(X)-\hat{T}^\epsilon(X)\|\right) \|\hat{T}^\epsilon(X)-T^\epsilon(X)\|\right) \\ & \leq \sqrt{\mathbb{E}\left(\|\hat{T}^\epsilon(X)-T^\epsilon(X)\|^2\right)\mathbb{E}\left(2\|\hat{T}^\epsilon_j(X)-T^\epsilon_j(X)\|^2+2\|T^\epsilon_j(X)-T^\epsilon(X)\|^2\right)} \\& \; \; \; \; \; + \sqrt{\mathbb{E}\left(2\|\hat{T}^\epsilon_i(X)-T^\epsilon_i(X)\|^2+2\|T_i^\epsilon(X)-\hat{T}^\epsilon(X)\|^2\right) \mathbb{E}\left(\|\hat{T}^\epsilon(X)-T^\epsilon(X)\|^2\right)} \\ & \leq \sqrt{\mathbb{E}\left(\|\hat{T}^\epsilon(X)-T^\epsilon(X)\|^2\right)\mathbb{E}\left(2\|\hat{T}^\epsilon_j(X)-T^\epsilon_j(X)\|^2+4\|T^\epsilon_j(X)\|^2 + 4\|T^\epsilon(X)\|^2\right)} \\& \; \; \; \; \; + \sqrt{\mathbb{E}\left(2\|\hat{T}^\epsilon_i(X)-T^\epsilon_i(X)\|^2+2\|T_i^\epsilon(X)-\hat{T}^\epsilon(X)\|^2\right) \mathbb{E}\left(\|\hat{T}^\epsilon(X)-T^\epsilon(X)\|^2\right)} \\ & \leq \sqrt{2\theta(n)\left(2\theta(n)+4M_2(\nu_j) + 4M_2(\mu)\right)} \\& \; \; \; \; \; + \sqrt{\mathbb{E}\left(\left(2\|\hat{T}^\epsilon_i(X)-T^\epsilon_i(X)\|^2+4\|T_i^\epsilon(X)-T^\epsilon(X)\|^2+4\|T^\epsilon(X)-\hat{T}^\epsilon(X)\|^2\right) \theta(n)\right)} \\ & \leq \sqrt{2\theta(n)\left(2\theta(n)+4M_2(\nu_j) + 4M_2(\mu)\right)} \\& \; \; \; \; \; + \sqrt{\mathbb{E}\left(\left(2\|\hat{T}^\epsilon_i(X)-T^\epsilon_i(X)\|^2+8\|T_i^\epsilon(X)\|^2 + 8\|T^\epsilon(X)\|^2+4\|T^\epsilon(X)-\hat{T}^\epsilon(X)\|^2\right) \theta(n)\right)}  \\ & \leq \sqrt{2\theta(n)\left(2\theta(n)+4M_2(\nu_j) + 4M_2(\mu)\right)} + \sqrt{\left(2\theta(n)+8M_2(\nu_i)+8M_2(\mu)+4\theta(n)\right) \theta(n)}.
\end{align*}
Applying the bound $M_2(\mu)\leq 2C_G^2q_G^2\sigma^2d$, which applies since $\mu$ is $\sigma$ subgaussian, and combining constants not depending on $\sigma$ and $d$, we obtain:
\begin{align*}
&\mathbb{E}\left(\bigg|\int \langle T_j^\epsilon - T^\epsilon, T_i^\epsilon-T^\epsilon\rangle d\mu - \frac{1}{n}\sum_{k=n+1}^{2n}\langle \hat{T}_j^\epsilon(X_k)-\hat{T}^\epsilon(X_k),\hat{T}_i^\epsilon(X_k)-\hat{T}^\epsilon(X_k)\rangle\bigg|\right)\\
\leq & \tilde{J}d\sigma^2\max\{\frac{1}{\sqrt{n}}, \sqrt{\theta(n)}+\sqrt{\theta(n)+\theta(n)^2}\}.
\end{align*}
\end{proof}

Theorem \ref{thm:coeff_theorem} immediately follows from Lemmas \ref{lem:ent_mat_bound} and \ref{lemma:sinkmatbound} and the following lemma, which allows us to apply (expected) bounds between the entries of two matrices to the (expected) distance between their eigenvectors, which we have extracted from Corollary 2 in \citep{werenski2022measure}:

\begin{lemma}\label{lemma: generic eigenvector bounds}
Let $M\in \mathbb{R}^{m\times m}$ be a positive semidefinite matrix with eigenvalue 0 with multiplicity 1 and associated unique eigenvector $\lambda_*\in \Delta^m$. Suppose that $\hat{M}\in \mathbb{R}^{m\times m}$ is a random positive semidefinite matrix such that \begin{equation} \mathbb{E}(|[M]_{ij}-[\hat{M}]_{ij}|)\leq \mathcal{Q}\label{eqn:generic_entrybound}, \; \text{for all }\; i,j=1,\dots,m.\end{equation} Let $\hat{\lambda}:=\argmin_{\lambda\in\Delta^m}\lambda^\top \hat{M}\lambda$.  Then:
\begin{equation*}
\mathbb{E}(\|\hat{\lambda}-\lambda_*\|^2)\leq \frac{8m^3\mathcal{Q}}{\alpha_2},
\end{equation*}
where $\alpha_2$ is the second smallest eigenvalue of $M$.
\end{lemma}
\begin{proof}
We decompose $\hat{\lambda}$ as $\hat{\beta}\lambda_*+\hat{\lambda}_\bot$, where $\hat{\beta}\in \mathbb{R}$ and $\hat{\lambda}_\bot$ is orthogonal to $\lambda_*.$ 
\begin{align}
\mathbb{E}\left(\|\hat{\lambda}-\lambda_*\|^2_2\right)   = &\mathbb{E}\left(\|\hat{\lambda}-\hat{\beta}\lambda_*+\hat{\beta}\lambda_*-\lambda_*\|^2_2\right)
\notag\\
\leq &2\mathbb{E}\left(\|\hat{\lambda}-\hat{\beta}\lambda_*\|^2_2\right) + 2\mathbb{E}\left(\|(\hat{\beta}-1)\lambda_*\|^2_2\right)  \label{eqn:generic split 1}
\end{align}
Observe that $\hat{\lambda}_\bot$ is in the span of the orthonormal eigenvectors $v_2,...,v_m$ of $M$, which we take to be ordered by their eigenvalues $0\leq \alpha_2\leq ...\leq \alpha_m$. Furthermore, as $\lambda_*\in\Delta^m$ is the unique eigenvector with eigenvalue $0$ by assumption, $0<\alpha_2.$ Then:

\begin{align}
\mathbb{E}\left(\|\hat{\lambda}-\hat{\beta}\lambda_*\|_2^2\right)  =
&\mathbb{E}\left(\|\hat{\lambda}_\bot\|_2^2\right)\notag \\
=&\mathbb{E}\left(\sum_{j=2}^m |v_j^\top \hat{\lambda}_\bot|^2\right)\notag 
\\ \leq &\frac{1}{\alpha_2}\mathbb{E}\left(\sum_{j=2}^m \alpha_j |v_j^\top\hat{\lambda}_\bot|^2\right)\notag \\
= &\frac{1}{\alpha_2}\mathbb{E}\left( |(\hat{\lambda}_\bot)^\top M \hat{\lambda}_\bot\right)\notag \\
= &\frac{1}{\alpha_2}\mathbb{E}\left(|(\hat{\lambda})^\top M\hat{\lambda}|\right)\notag 
\end{align} 
We bound: 
\begin{align*}
\mathbb{E}\left(\hat{\lambda}^\top M \hat{\lambda}\right) & =\mathbb{E}\left(\hat{\lambda}^\top(M-\hat{M})\hat{\lambda}\right)+\mathbb{E}\left(\hat{\lambda}^\top\hat{M}\hat{\lambda}\right) & \\ & \leq \mathbb{E}\left(|\hat{\lambda}^\top(M-\hat{M})\hat{\lambda}|\right)+\mathbb{E}\left(\lambda_*^\top\hat{M}\lambda_*\right) & \\ & = \mathbb{E}\left(|\hat{\lambda}^\top(M-\hat{M})\hat{\lambda}|\right)+\mathbb{E}\left(\lambda_*^\top(\hat{M}-M)\lambda_*\right) 
\end{align*}
where we used the minimality of $\hat{\lambda}$ and the fact that $\lambda_*^\top M\lambda_*=0$. We may rewrite these as the following sums:
\begin{align}
&\mathbb{E}\left( \bigg|\sum_{i,j=1}^m [\hat{\lambda}]_i[\hat{\lambda}]_j [M]_{ij}-[\hat{M}]_{ij}\bigg|\right) + \mathbb{E}\left( \sum_{i,j=1}^m [\lambda_*]_i[\lambda_*]_j ([\hat{M}]_{ij}-[M]_{ij})\right) 
\notag\\  \leq  &\sum_{i,j=1}^m \mathbb{E}\left([\hat{\lambda}]_i[\hat{\lambda}]_j |[M]_{ij}-[\hat{M}]_{ij}|\right) +  \sum_{i,j=1}^m [\lambda_*]_i[\lambda_*]_j \mathbb{E}\left([\hat{M}]_{ij}-[M]_{ij}\right) 
\notag\\ \leq &2m^{2}\mathcal{Q},\label{eqn:2mq bound}
\end{align}
where we used the triangle inequality, the fact that $\hat{\lambda},\lambda_*\in \Delta^m$ and (\ref{eqn:generic_entrybound}). This lets us bound (\ref{eqn:generic split 1}) by:
\begin{align*}
&\frac{4m^2}{\alpha_2}\mathcal{Q} + 2\mathbb{E}\left((\hat{\beta}-1)^2\|\lambda_*\|_2^2\right)\\ 
\leq &\frac{4m^2}{\alpha_2}\mathcal{Q}+2\mathbb{E}\left((\hat{\beta}-1)^2\right),
\end{align*}
where we used the fact that $\|\lambda_*\|^2_2\leq 1$. Since $\hat{\lambda}=\hat{\beta}\lambda_*+\hat{\lambda}_\bot$ with $\hat{\lambda},\lambda_*\in \Delta^m$, we have that:
\begin{align*}
\mathbb{E}((\hat{\beta}-1)^2) &=\mathbb{E}\left(\left(\sum_{j=1}^m [\hat{\lambda}_\bot]_j\right)^2\right)\notag \\ &  \leq m\mathbb{E}(\|\hat{\lambda}_\bot\|_2^2)
\end{align*}
by the Cauchy-Schwarz inequality, and hence $\mathbb{E}((1-\hat{\beta})^2)\leq m \mathbb{E}(\|\hat{\lambda}_\bot\|_2^2)\leq \frac{2m^3}{\alpha_2}\mathcal{Q}$ by another application of (\ref{eqn:2mq bound}). We conclude that
\begin{align*} \mathbb{E}\left(\|\hat{\lambda}-\lambda_*\|^2_2\right)&\leq\frac{4m^2}{\alpha_2}\mathcal{Q}+\frac{4m^3}{\alpha_2}\mathcal{Q}\le \frac{8m^3}{\alpha_2}\mathcal{Q} .\end{align*}
\end{proof}

\noindent\textbf{Proof of Theorem \ref{thm:coeff_theorem}:} Immediately follows from substituting the bounds from Lemmas \ref{lem:ent_mat_bound} and \ref{lemma:sinkmatbound} as $\mathcal{Q}$ in Lemma \ref{lemma: generic eigenvector bounds}.\qed 

\bigskip

\noindent\textbf{Proof of Corollary \ref{cor:bounded analysis}:} The proof follows from the following result:

\begin{theorem}\label{thm:rigolletstromme}
(Theorem 4 in \citep{rigollet2022sample}) Let $\Omega$ be a bounded set in $\mathbb{R}^d$, $\epsilon>0$, and $\mu,\nu\in \mathcal{P}_2(\Omega)$. Then
\begin{equation*}
\mathbb{E}\left(\|T_{\mu\rightarrow\nu}^\epsilon-T^\epsilon_{\hat{\mu}^n\rightarrow\hat{\nu}^n}\|_{L^2(\mu)}^2\right)<\frac{C_{\Omega,\epsilon}}{n}
\end{equation*}
where $C_{\Omega,\epsilon}$ is a constant only depending on $|\Omega|$ and $\epsilon.$
\end{theorem}
By Theorem \ref{thm:rigolletstromme}, we bound \begin{equation*} \mathbb{E}(\|\entmap{\mu}{\nu}-\entmap{\hat{\mu}^n}{\hat{\nu}^n}\|^2_{L^2(\mu)})\leq \frac{C_{\Omega,\epsilon}}{n}.\end{equation*}Similarly, we may bound \begin{equation*}
\mathbb{E}(\|\entmap{\mu}{\mu}-\entmap{\hat{\mu}_1}{\hat{\mu}_2}\|^2_{L^2(\mu)})\leq \frac{3C_{\Omega,\epsilon}}{n}.
\end{equation*}
Hence we may apply Theorem \ref{thm:coeff_theorem} with $\theta(n)=\frac{1}{n}$. As $\sqrt{\frac{1}{n}+\frac{1}{n^2}}\leq \sqrt{\frac{2}{n}}$ for $n\geq 1$, we may bound $\max\left\{\frac{1}{\sqrt{n}},\sqrt{\frac{1}{n}+\frac{1}{n^2}}\right\}\leq \frac{2}{\sqrt{n}}.$ The second moments of $\nu_1,...,\nu_m,\mu$ are uniformly bounded by $|\Omega|^2$, and hence we may collapse the various constants into a single constant $\tilde{C}_{\Omega,\epsilon}$.
\qed 
\bigskip

\vspace{10pt}

\noindent\textbf{Proof of Corollary \ref{cor:logconcave_analysis}}: 
It is well known that $c$-strongly log-concave measures are $\frac{L}{\sqrt{c}}$-subgaussian for some universal constant $L$ (see e.g., Theorem 5 in \citep{werenski2023estimation}). Similarly to Corollary \ref{cor:bounded analysis}, the proof follows from the following result:

\begin{theorem}\label{thm:werenskiestimation} (Theorem 7 in \citep{werenski2023estimation}) Let $\mu$ be $\sigma$-subgaussian and let $\nu$ be $c$-strongly log concave with $\mathbb{E}(\nu)=0$. Then there exists an estimator $\mathcal{L}$ such that:

\begin{equation*}
\mathbb{E}\left( \|T^\epsilon_{\mu\rightarrow \nu}-\mathcal{L}(\hat{\mu}^n,\hat{\nu}^n)\|_{L^2(\mu)}^2\right)\leq \frac{K_{d,\sigma,\epsilon,c}}{n^{1/3}},
\end{equation*}
where $K_{d,\sigma,\epsilon,c}$ is a constant depending on $c,\sigma,d$ and $\epsilon.$
\end{theorem}

Since $\mu$ is a critical point of $F_{\lambda,\mathcal{V}}^\epsilon$, it is $\frac{L}{\sqrt{c}}$-subgaussian by Proposition \ref{prop:subgauss}, and hence the constant factors only depend on $c,d$ and $\epsilon$.
\qed 

\bigskip

\subsection{Proof of Proposition \ref{prop:stab_analysis}}
To establish Proposition \ref{prop:stab_analysis}, we will require the following technical results:

\begin{lemma}\label{lem:lipschitz}
(Lemma A.1 in \citep{werenski2024rank}) Let $\Omega\subset\mathbb{R}^d$ be bounded, and let $\mu,\nu \in \mathcal{P}_2(\Omega)$. Then:
\begin{equation*}\|T^\epsilon_{\mu\rightarrow \nu}(x)-T^\epsilon_{\mu\rightarrow \nu}(y)\|\leq M_{\Omega,\epsilon}\|x-y\|,\end{equation*} for all $x,y\in \Omega$, where $M_{\Omega,\epsilon}=\frac{4|\Omega|^2}{\epsilon}$.
\end{lemma}

\begin{proposition}\label{thm:schrodmap}
(Proposition 4 in \citep{kassraie2024progressive} and Corollary 3.3 in \citep{divol2024tight})  Let $\Omega\subset \mathbb{R}^d$ be bounded, and let $\mu_1,\mu_2,\nu\in \mathcal{P}_2(\Omega)$. Then:
\begin{equation}
\|T^\epsilon_{\nu\rightarrow \mu_1}-T^\epsilon_{\nu\rightarrow \mu_2}\|_{L^2(\nu)}\leq K_{\Omega,\epsilon}OT_2(\mu_1,\mu_2)\label{eqn:target_stable}
\end{equation}
\begin{equation}
\|T^\epsilon_{\mu_1\rightarrow \nu}-T^\epsilon_{\mu_2\rightarrow \nu}\|_{L^2(\mu_1)}\leq K_{\Omega,\epsilon}OT_2(\mu_1,\mu_2)\label{eqn:source_stable}
\end{equation}
where $K_{\Omega,\epsilon}=\frac{3^{1/2}|\Omega|}{\epsilon^{1/2}}+1$.
\end{proposition}
From Proposition \ref{thm:schrodmap} we obtain the following lemma:

\begin{lemma}\label{lemma:matrix entry bounds}
Suppose $\Omega\subset\mathbb{R}^{d}$ is bounded.  Let $A_\mu^\epsilon$ and $S_\mu^\epsilon$ be defined as in Proposition \ref{prop:analysis_grad}. Then for any $1\leq i,j \leq m$:
\begin{equation*}
|[A^\epsilon_\mu]_{ij}-[A_\rho^\epsilon]_{ij}| \leq H_{\Omega,\epsilon}OT_2(\rho,\mu)
\end{equation*}
\begin{equation*}
|[S^\epsilon_\mu]_{ij}-[S_\rho^\epsilon]_{ij}| \leq H_{\Omega,\epsilon}OT_2(\mu,\rho)
\end{equation*}
where $H_{\Omega,\epsilon}$ depends on $|\Omega|$ and $\epsilon.$
\end{lemma}

\begin{proof} We begin by bounding $|[A_\mu^\epsilon]_{ij}-[A_\rho^\epsilon]_{ij}|$ for any $1\leq i,j\leq m$:
\begin{align}
&|[A_\mu^\epsilon]_{ij}-[A_\rho^\epsilon]_{ij}| \notag\\ 
= &|\langle Id-T^\epsilon_{\mu\rightarrow \nu_i},Id-T^\epsilon_{\mu\rightarrow \nu_j}\rangle_{L^{2}(\mu)}-\langle Id-T^\epsilon_{\rho\rightarrow \nu_i},Id-T^\epsilon_{\rho\rightarrow \nu_j}\rangle_{L^{2}(\rho)} |\notag \\
 \leq &|\langle Id-T^\epsilon_{\mu\rightarrow \nu_i},Id-T^\epsilon_{\mu\rightarrow \nu_j}\rangle_{L^{2}(\mu)}-\langle Id-T^\epsilon_{\mu\rightarrow \nu_i},Id-T^\epsilon_{\rho\rightarrow \nu_j}\rangle_{L^{2}(\mu)} |\notag \\ + & |\langle Id-T^\epsilon_{\mu\rightarrow \nu_i},Id-T^\epsilon_{\rho\rightarrow \nu_j}\rangle_{L^{2}(\mu)}-\langle Id-T^\epsilon_{\mu\rightarrow \nu_i},Id-T^\epsilon_{\rho\rightarrow \nu_j}\rangle_{L^{2}(\rho)}| \notag\\ + &|\langle Id-T^\epsilon_{\mu\rightarrow \nu_i},Id-T^\epsilon_{\rho\rightarrow \nu_j}\rangle_{L^{2}(\rho)}-\langle Id-T^\epsilon_{\rho\rightarrow \nu_i},Id-T^\epsilon_{\rho\rightarrow \nu_j}\rangle_{L^{2}(\rho)} |\notag \\ 
 = &|\langle Id-T^\epsilon_{\mu\rightarrow \nu_i},T^\epsilon_{\mu\rightarrow \nu_j}-T^\epsilon_{\rho\rightarrow \nu_j}\rangle_{L^{2}(\mu)}| \notag\\ +& |\langle Id-T^\epsilon_{\mu\rightarrow \nu_i},Id-T^\epsilon_{\rho\rightarrow \nu_j}\rangle_{L^{2}(\mu)}-\langle Id-T^\epsilon_{\mu\rightarrow \nu_i},Id-T^\epsilon_{\rho\rightarrow \nu_j}\rangle_{L^{2}(\rho)}|\notag \\ + & |\langle T_{\mu\rightarrow \nu_i}-T_{\rho\rightarrow \nu_i},Id-T_{\rho\rightarrow \nu_j}\rangle_{L^{2}(\rho)}|\label{eqn:stab_proof}\end{align}
where we applied the triangle inequality.  We bound the first term in (\ref{eqn:stab_proof}):
\begin{align*}
|\langle Id-T^\epsilon_{\mu\rightarrow \nu_i},T_{\mu\rightarrow \nu_j}^\epsilon -T^\epsilon_{\rho\rightarrow \nu_j}\rangle_{L^2(\mu)} | & \leq \|Id-T^\epsilon_{\mu\rightarrow \nu_i}\|_{L^2(\mu)} \|T_{\mu\rightarrow \nu_j}^\epsilon-T^\epsilon_{\rho\rightarrow \nu_j}\|_{L^2(\mu)}  \\& \leq |\Omega|\|T_{\mu\rightarrow \nu_j}^\epsilon-T^\epsilon_{\rho\rightarrow \nu_j}\|_{L^2(\mu)}
\\&  \leq |\Omega|K_{\Omega,\epsilon}OT_2(\mu,\rho),
\end{align*}
where we applied the Cauchy-Schwarz inequality, the diameter bound on $\Omega$ and (\ref{eqn:source_stable}) to control $\|T^\epsilon_{\mu\rightarrow \nu_j}-T^\epsilon_{\rho\rightarrow \nu_j}\|_{L^2(\mu)}$. The third term can be bounded identically. 

To bound the second term, we first observe that, by Lemma \ref{lem:lipschitz}, $T^\epsilon_{\mu\rightarrow \nu_i}$ and $T^\epsilon_{\rho\rightarrow \nu_j}$ are $M_{\Omega,\epsilon}$-Lipschitz vector-valued maps from $\Omega$ to the convex hull of $\Omega.$
As $\Omega$ is bounded, $x\mapsto \langle x-T^\epsilon_{\mu\rightarrow \nu_i}(x),x-T^\epsilon_{\rho\rightarrow \nu_j}(x)\rangle$ is $2|\Omega|(1+M_{\Omega,\epsilon})$-Lipschitz as a function of $x\in \Omega$.  Indeed, let $x,y\in \Omega$.  Then:
\begin{align}
&|\langle x-T^\epsilon_{\mu\rightarrow \nu_i}(x),x-T^\epsilon_{\rho\rightarrow \nu_j}(x)\rangle-\langle y-T^\epsilon_{\mu\rightarrow \nu_i}(y),y-T^\epsilon_{\rho\rightarrow \nu_j}(y)\rangle| \notag\\ 
\leq &|\langle x-T^\epsilon_{\mu\rightarrow \nu_i}(x),x-T^\epsilon_{\rho\rightarrow \nu_j}(x)\rangle-\langle x-T^\epsilon_{\mu\rightarrow \nu_i}(x),y-T^\epsilon_{\rho\rightarrow \nu_j}(y)\rangle|\notag\\
+&|\langle x-T^\epsilon_{\mu\rightarrow \nu_i}(x),y-T^\epsilon_{\rho\rightarrow \nu_j}(y)\rangle -\langle y- T^\epsilon_{\mu\rightarrow \nu_i}(y),y-T^\epsilon_{\rho\rightarrow \nu_j}(y)\rangle|\label{eqn:stab_1}\\
=&|\langle x-T^\epsilon_{\mu\rightarrow \nu_i}(x),x-T^\epsilon_{\rho\rightarrow \nu_j}(x)-y+T^\epsilon_{\rho\rightarrow \nu_j}(y)\rangle|\notag\\
+&|\langle x-T^\epsilon_{\mu\rightarrow \nu_i}(x)-y+T^\epsilon_{\mu\rightarrow \nu_i}(y),y-T^\epsilon_{\rho\rightarrow \nu_j}(y)\rangle|\notag \\ \leq &\| x-T^\epsilon_{\mu\rightarrow \nu_i}(x)\|\|x-T^\epsilon_{\rho\rightarrow \nu_j}(x)-y+T^\epsilon_{\rho\rightarrow \nu_j}(y)\|\notag\\
+&\| x-T^\epsilon_{\mu\rightarrow \nu_i}(x)-y+T^\epsilon_{\mu\rightarrow \nu_i}(y)\|\|y-T^\epsilon_{\rho\rightarrow \nu_j}(y)\| \label{eqn:stab_2}\\
\leq &|\Omega|\left(\| x-T^\epsilon_{\mu\rightarrow \nu_i}(x)-y+T^\epsilon_{\mu\rightarrow \nu_i}(y)\|+\|x-T^\epsilon_{\rho\rightarrow \nu_j}(x)-y+T^\epsilon_{\rho\rightarrow \nu_j}(y)\|\right)\label{eqn:stab_3} \\
\leq &|\Omega|\left(2\|x-y\|+\|T^\epsilon_{\mu\rightarrow \nu_i}(x)-T^\epsilon_{\mu\rightarrow \nu_i}(y)\|+\|T^\epsilon_{\rho\rightarrow \nu_j}(x)-T^\epsilon_{\rho\rightarrow \nu_j}(y)\|\right) \label{eqn:stab_4}\\
\leq &2|\Omega|(1+M_{\Omega,\epsilon})\|x-y\|\label{eqn:stab_5}.
\end{align}
In (\ref{eqn:stab_1}) we applied the triangle inequality, in (\ref{eqn:stab_2}) the Cauchy-Schwarz inequality, in (\ref{eqn:stab_3}) we used the bounds $\|x-T^\epsilon_{\mu\rightarrow \nu_i}(x)\|,\|y-T^\epsilon_{\rho\rightarrow \nu_j}(y)\|\leq |\Omega|$, in (\ref{eqn:stab_4}) the triangle inequality and finally in (\ref{eqn:stab_5}) we applied Lemma \ref{lem:lipschitz}.  Since $\langle Id-T^\epsilon_{\mu\rightarrow \nu_i},Id-T^\epsilon_{\rho\rightarrow \nu_j}\rangle$ is $2|\Omega|(1+M_{\Omega,\epsilon})$-Lipschitz, we may apply Kantorovich-Rubinstein duality for the Wasserstein-1 distance (see e.g., 5.16 in \citep{villani2009optimal}) to produce the bound:
\begin{align*}
& |\langle Id-T^\epsilon_{\mu\rightarrow \nu_i},Id-T^\epsilon_{\rho\rightarrow \nu_j}\rangle_{L^{2}(\mu)}-\langle Id-T^\epsilon_{\mu\rightarrow \nu_i},Id-T^\epsilon_{\rho\rightarrow \nu_j}\rangle_{L^{2}(\rho)}| \\  \leq &2|\Omega|(1+M_{\Omega,\epsilon}) OT_1(\mu,\rho) \\  \leq &2|\Omega|(1+M_{\Omega,\epsilon}) OT_2(\mu,\rho)
\end{align*}
where in the last line we applied Jensen's inequality. Combining the bounds on these three terms, we may conclude that $|[A_\mu^\epsilon]_{ij}-[A_\rho^\epsilon]_{ij}|\leq 2|\Omega| (M_{\Omega,\epsilon}+K_{\Omega,\epsilon}+1)OT_2(\mu,\rho)$.

We now treat the Sinkhorn case. For any $\mu\in \mathcal{P}_2(\Omega)$ and $1\leq i , j \leq m$, we have:

\begin{align}
&|[S_\mu^\epsilon]_{ij}-[S_\rho^\epsilon]_{ij}|\notag \\ 
=&|\langle T_{\mu\rightarrow \mu}^\epsilon - T_{\mu\rightarrow \nu_i}^\epsilon, T_{\mu\rightarrow \mu}^\epsilon-T^\epsilon_{\mu\rightarrow \nu_j}\rangle_{L^2(\mu)}-\langle T_{\rho\rightarrow \rho}^\epsilon-T^\epsilon_{\rho\rightarrow \nu_i}, T_{\rho\rightarrow \rho}^\epsilon-T^\epsilon_{\rho\rightarrow \nu_j}\rangle_{L^2(\rho)}|\notag\\ 
 \leq &|\langle T^\epsilon_{\mu\rightarrow \mu}- T^\epsilon_{\mu\rightarrow \nu_i}, T^\epsilon_{\mu\rightarrow \mu}-T^\epsilon_{\mu\rightarrow \nu_j}\rangle_{L^2(\mu)}- \langle T^\epsilon_{\rho\rightarrow \rho}-\entmap{\mu}{\nu_i},\entmap{\mu}{\mu}-\entmap{\mu}{\nu_j}\rangle_{L^2(\mu)}|\label{eqn:sinkstab1}\\
  + &|\langle T^\epsilon_{\rho\rightarrow \rho}-\entmap{\mu}{\nu_i},\entmap{\mu}{\mu}-\entmap{\mu}{\nu_j}\rangle_{L^2(\mu)}-\langle T^\epsilon_{\rho\rightarrow \rho}-\entmap{\mu}{\nu_i},\entmap{\rho}{\rho}-\entmap{\mu}{\nu_j}\rangle_{L^2(\mu)}|\label{eqn:sinkstab2} \\ +&|\langle \entmap{\rho}{\rho}-T^\epsilon_{\mu\rightarrow \nu_i}, \entmap{\rho}{\rho}-\entmap{\mu}{\nu_j}\rangle_{L^2(\mu)}-\langle \entmap{\rho}{\rho}-\entmap{\mu}{\nu_i},\entmap{\rho}{\rho}-\entmap{\rho}{\nu_j}\rangle_{L^2(\mu)}|\label{eqn:sinkstab3}\\+ &|\langle \entmap{\rho}{\rho}-\entmap{\mu}{\nu_i},\entmap{\rho}{\rho}-\entmap{\rho}{\nu_j}\rangle_{L^2(\mu)} - \langle \entmap{\rho}{\rho}-\entmap{\rho}{\nu_i},\entmap{\rho}{\rho}-\entmap{\rho}{\nu_j}\rangle_{L^2(\mu)}|\label{eqn:sinkstab4}\\ +&|\langle \entmap{\rho}{\rho}-\entmap{\rho}{\nu_i},\entmap{\rho}{\rho}-\entmap{\rho}{\nu_j}\rangle_{L^2(\mu)}-\langle \entmap{\rho}{\rho}-\entmap{\rho}{\nu_i},\entmap{\rho}{\rho}-\entmap{\rho}{\nu_j}\rangle_{L^2(\rho)}|\label{eqn:sinkstab5}
\end{align}
We first bound (\ref{eqn:sinkstab1}):
\begin{align*}
& |\langle T^\epsilon_{\mu\rightarrow \mu}- T^\epsilon_{\mu\rightarrow \nu_i}, T^\epsilon_{\mu\rightarrow \mu}-T^\epsilon_{\mu\rightarrow \nu_j}\rangle_{L^2(\mu)}- \langle T^\epsilon_{\rho\rightarrow \rho}-\entmap{\mu}{\nu_i},\entmap{\mu}{\mu}-\entmap{\mu}{\nu_j}\rangle_{L^2(\mu)}|\\ =&|\langle T^\epsilon_{\mu\rightarrow \mu}-T^\epsilon_{\rho\rightarrow \rho}, T^\epsilon_{\mu\rightarrow \mu}- T^\epsilon_{\mu\rightarrow \nu_j}\rangle_{L^2(\mu)}|\\ \leq &\|\entmap{\mu}{\mu}-\entmap{\rho}{\rho}\|_{L^2(\mu)} \| \entmap{\mu}{\mu}-\entmap{\mu}{\nu_j}\|_{L^2(\mu)}\\
 \leq &\|\entmap{\mu}{\mu}-\entmap{\rho}{\rho}\|_{L^2(\mu)}|\Omega|\\ 
\leq&(\|\entmap{\mu}{\mu}-\entmap{\mu}{\rho}\|_{L^2(\mu)}+\|\entmap{\mu}{\rho}-\entmap{\rho}{\rho}\|_{L^2(\mu)})|\Omega|\\
\leq &2K_{\Omega,\epsilon}|\Omega| OT_2(\mu,\rho),
\end{align*}
where we applied the Cauchy-Schwarz inequality, the fact that the image of $\entmap{\mu}{\mu}$ and $\entmap{\mu}{\nu_j}$ are contained in $\Omega$, followed by (\ref{eqn:target_stable}) and  (\ref{eqn:source_stable}).  The same bound applies to (\ref{eqn:sinkstab2}).

A similar argument applies to (\ref{eqn:sinkstab3}):
\begin{align*}
& |\langle \entmap{\rho}{\rho}-T^\epsilon_{\mu\rightarrow \nu_i}, \entmap{\rho}{\rho}-\entmap{\mu}{\nu_j}\rangle_{L^2(\mu)}-\langle \entmap{\rho}{\rho}-\entmap{\mu}{\nu_i},\entmap{\rho}{\rho}-\entmap{\rho}{\nu_j}\rangle_{L^2(\mu)}|\\ 
= &|\langle \entmap{\rho}{\rho}-T^\epsilon_{\mu\rightarrow \nu_i}, \entmap{\rho}{\nu_j}-\entmap{\mu}{\nu_j}\rangle_{L^2(\mu)}|\\
 \leq &\|\entmap{\rho}{\rho}-\entmap{\mu}{\nu_i}\|_{L^2(\mu)}\|\entmap{\rho}{\nu_j}-\entmap{\mu}{\nu_j}\|_{L^2(\mu)}\\  
 \leq &|\Omega| K_{\Omega,\epsilon}OT_2(\mu,\rho)
\end{align*}
where we applied (\ref{eqn:source_stable}) and the fact that $\|\entmap{\rho}{\rho}(x)-\entmap{\mu}{\nu_i}(x)\|\leq |\Omega|$ for all $x\in \Omega$. The same bound applies to (\ref{eqn:sinkstab4}). 

Finally, to bound (\ref{eqn:sinkstab5}), we observe that $x\rightarrow \langle \entmap{\rho}{\rho}(x)-\entmap{\rho}{\nu_i}(x),\entmap{\rho}{\rho}(x)-\entmap{\rho}{\nu_j}(x)\rangle$ is $4|\Omega|M_{\Omega,\epsilon}$-Lipschitz: 
\begin{align}
&|\langle \entmap{\rho}{\rho}(x)-\entmap{\rho}{\nu_i}(x),\entmap{\rho}{\rho}(x)-T^\epsilon_{\rho\rightarrow \nu_j}(x)\rangle-\langle \entmap{\rho}{\rho}(y)-\entmap{\rho}{\nu_i}(y),\entmap{\rho}{\rho}(y)-T^\epsilon_{\rho\rightarrow \nu_j}(y)\rangle| \notag\\ 
\leq &|\langle \entmap{\rho}{\rho}(x)-\entmap{\rho}{\nu_i}(x),\entmap{\rho}{\rho}(x)-T^\epsilon_{\rho\rightarrow \nu_j}(x)\rangle-\langle \entmap{\rho}{\rho}(x)-\entmap{\rho}{\nu_i}(x),\entmap{\rho}{\rho}(y)-T^\epsilon_{\rho\rightarrow \nu_j}(y)\rangle|\notag\\
+&|\langle \entmap{\rho}{\rho}(x)-\entmap{\rho}{\nu_i}(x),\entmap{\rho}{\rho}(y)-T^\epsilon_{\rho\rightarrow \nu_j}(y)\rangle -\langle \entmap{\rho}{\rho}(y)- \entmap{\rho}{\nu_i}(y),\entmap{\rho}{\rho}(y)-T^\epsilon_{\rho\rightarrow \nu_j}(y)\rangle|\notag\\
=&|\langle \entmap{\rho}{\rho}(x)-\entmap{\rho}{\nu_i}(x),\entmap{\rho}{\rho}(x)-T^\epsilon_{\rho\rightarrow \nu_j}(x)-\entmap{\rho}{\rho}(y)+T^\epsilon_{\rho\rightarrow \nu_j}(y)\rangle|\notag\\
+&|\langle \entmap{\rho}{\rho}(x)-\entmap{\rho}{\nu_i}(x)-\entmap{\rho}{\rho}(y)+\entmap{\rho}{\nu_i}(y),\entmap{\rho}{\rho}(y)-T^\epsilon_{\rho\rightarrow \nu_j}(y)\rangle|\notag \\ \leq &\| \entmap{\rho}{\rho}(x)-\entmap{\rho}{\nu_i}(x)\|\|\entmap{\rho}{\rho}(x)-T^\epsilon_{\rho\rightarrow \nu_j}(x)-\entmap{\rho}{\rho}(y)+T^\epsilon_{\rho\rightarrow \nu_j}(y)\|\notag\\
+&\| \entmap{\rho}{\rho}(x)-\entmap{\rho}{\nu_i}(x)-\entmap{\rho}{\rho}(y)+\entmap{\rho}{\nu_i}(y)\|\|\entmap{\rho}{\rho}(y)-T^\epsilon_{\rho\rightarrow \nu_j}(y)\| \notag\\
\leq &|\Omega|\left(\| \entmap{\rho}{\rho}(x)-\entmap{\rho}{\nu_i}(x)-\entmap{\rho}{\rho}(y)+\entmap{\rho}{\nu_i}(y)\|+\|\entmap{\rho}{\rho}(x)-T^\epsilon_{\rho\rightarrow \nu_j}(x)-\entmap{\rho}{\rho}(y)+T^\epsilon_{\rho\rightarrow \nu_j}(y)\|\right)\notag\\
\leq &|\Omega|\left(2\|\entmap{\rho}{\rho}(x)-\entmap{\rho}{\rho}(y)\|+\|\entmap{\rho}{\nu_i}(x)-\entmap{\rho}{\nu_i}(y)\|+\|T^\epsilon_{\rho\rightarrow \nu_j}(x)-T^\epsilon_{\rho\rightarrow \nu_j}(y)\|\right)\notag\\ \leq &4|\Omega|M_{\Omega,\epsilon}\|x-y\|,\notag
\end{align}
where we applied Lemma \ref{lem:lipschitz} in the final line. Thus we may apply Kantorovich-Rubinstein duality once more to conclude that
\begin{align*}
& |\langle \entmap{\rho}{\rho}-\entmap{\rho}{\nu_i},\entmap{\rho}{\rho}-T^\epsilon_{\rho\rightarrow \nu_j}\rangle_{L^{2}(\mu)}-\langle \entmap{\rho}{\rho}-\entmap{\rho}{\nu_i},\entmap{\rho}{\rho}-T^\epsilon_{\rho\rightarrow \nu_j}\rangle_{L^{2}(\rho)}| \\  \leq &4|\Omega|M_{\Omega,\epsilon} OT_1(\mu,\rho) \\  \leq &4|\Omega|M_{\Omega,\epsilon} OT_2(\mu,\rho).
\end{align*}
Combining the bounds on (\ref{eqn:sinkstab1})-(\ref{eqn:sinkstab5}), we have that:
\begin{equation*}
|[S_\mu^\epsilon]_{ij}-[S_\rho^\epsilon]_{ij}|\leq |\Omega|(6K_{\Omega,\epsilon}+4M_{\Omega,\epsilon})OT_2(\mu,\rho).
\end{equation*}
Taking the maximum of these to be $H_{\Omega,\epsilon}$, we conclude.
\end{proof}

\noindent \textbf{Proof of Proposition \ref{prop:stab_analysis}:} We apply Lemma \ref{lemma: generic eigenvector bounds} with $M=A^\epsilon_\mu$ and $\hat{M}=A^\epsilon_\rho$ (note this is deterministic) and can take $\mathcal{Q}=H_{\Omega,\epsilon}OT_2(\mu,\rho)$ by Lemma \ref{lemma:matrix entry bounds}. The same holds for $S^\epsilon_\mu$. \qed

\section{IMPLEMENTATION DETAILS}

\subsection{Synthesis Implementation}\label{sec:synthesis_algo}
\begin{algorithm}
\caption{ Wasserstein gradient descent for $\phi$ (WGD)}\label{alg:WGD}
\KwData{Initial measure $\mu_0\in \mathcal{P}_2(\Omega)$;

Functional $\phi:\mathcal{P}_2(\Omega)\rightarrow \mathbb{R}$;

Step-size $\eta>0$;

Iterations $k$.}

    \vspace{2mm} 
    \hrule 
    \vspace{2mm} 

$\mathcal{T}_{[0:0]}:=\mathcal{T}_0\gets Id - \eta \nabla \delta \phi(\mu_0)$ \tcp*{Compute Wasserstein-2 gradient of $\phi$ at $\mu_0$;}
$\mu_1\gets [\mathcal{T}_{[0:0]}]_{\#}(\mu_0)$ \tcp*{Transport $\mu_0$ according to $\mathcal{T}_{[0:0]}$}

\For{$0\leq i\leq k-1$}{
    $\mathcal{T}_i\gets Id-\eta\nabla\delta\phi(\mu_i)$ \tcp*{Compute Wasserstein-2 gradient of $\phi$ at $\mu_i$}
    $\mathcal{T}_{[0:i]}\gets \mathcal{T}_i\circ \mathcal{T}_{[0:i-1]}$ \tcp*{Update by composing maps}
    $\mu_{i+1}\gets [\mathcal{T}_{[0:i]}]_{\#}(\mu_0)$ \tcp*{Transport $\mu_0$ according to composed map}
}

\vspace{1mm} 
    \hrule 
    \vspace{1mm}
\textbf{Output:} $\mu_{k}$.
\end{algorithm}

We detail the implementation of the free support synthesis algorithm used in Section \ref{sec:numericalexp} (see Algorithm \ref{alg:WGD} for pseudocode). For the case of $F^\epsilon_{\lambda,\mathcal{V}}$, this corresponds to a modified version of Algorithm 3 in \citep{cuturi2014fast}, without the density update step. More generally, we may view Algorithm \ref{alg:WGD} as a forward Euler discretization of the Wasserstein-2 gradient flow \citep{ambrosio2006gradient} for a functional $\phi: \mathcal{P}_2(\Omega)\rightarrow \mathbb{R}$ with some fixed step-size $\eta$. It is also implemented for $F^\epsilon_{\lambda,\mathcal{V}}$ for the popular optimal transport package Python Optimal Transport (POT) \citep{POT2023}. See Theorem 4.1 in \citep{shen2020sinkhorn} for a proof of convergence in the bounded setting.

We use Algorithm \ref{alg:WGD} to compute (approximate) critical points of $\mathcal{F}^\epsilon_{\lambda,\mathcal{V}}$. In each trial, we construct $\hat{\nu}_j^n$ by sampling $n=20000$ samples according to $\nu_j$. We construct $\mu_0$ by sampling an empirical measure (on $n=20000$ samples) from $Unif([0,1]^5)$. We apply Algorithm \ref{alg:WGD} to $F^\epsilon_{\lambda,\hat{\mathcal{V}}^n}$ with initial measure $\mu_0$ and $\hat{\mathcal{V}}^n:=\{\hat{\nu}_j^n\}_{j=1}^m$, with stopping criterion $\|Id-\eta\nabla \delta F^\epsilon_{\lambda,\hat{\mathcal{V}}^n} (\mu_i)\|_{L^2(\mu_i)}<10^{-7}$. We denote the output by $\mu^*$. 

\subsection{Sample Complexity Experiment Implementation Details}\label{sec:samplecomplex_implementation}
 \noindent\textbf{1-Dimensional Gaussians:}  We build a random set of reference measures \begin{equation*}\mathcal{V}=\{\mathcal{N}(M_1,\sigma^2_1),\mathcal{N}(M_2,\sigma^2_2),...,\mathcal{N}(M_m,\sigma^2_m)\}\end{equation*}
 for $m\geq 1$ where $M_j $ are i.i.d. $\sim \mbox{Unif}([a_j, b_j])$ and $\sigma_i$  i.i.d. $ \sim \mbox{Unif}([c_j, d_j])$. We then sample $\tilde{\lambda}_1,\tilde{\lambda}_2,...,\tilde{\lambda}_m$ i.i.d. $\sim \mbox{Unif}([0,1])$, and define $\lambda_j:=\tilde{\lambda}_j/\sum_{i=1}^m\tilde{\lambda}_i$ for all $j$.  We then compute the exact mean $\bar{M}$ and variance $S$ of the true minimizer $\argmin_{\mu\in \mathcal{P}_2(\mathbb{R})}F_{\lambda,\mathcal{V}}^\epsilon(\mu)$, which is known in closed form via Theorem 2 in \citep{janati2020debiased} (we solve for the variance numerically using \texttt{scipy.optimize least$\_$squares}).  We then run Algorithm \ref{alg:coefficient_recovery} over $n$ samples, with $\mu=\mathcal{N}(\bar{M},S^2)$ and $\mathcal{V}$, $\lambda$ and $\epsilon$ chosen as above to recover a set of estimated coefficients $\hat{\lambda}^n$. We compute the $\ell^2$ error between $\hat{\lambda}^n$ and $\lambda$. 

In the case of $m=2$, we sample $M_1\sim \mbox{Unif}([0,2])$, $M_2\sim \mbox{Unif}([3,5])$ and $\sigma_1,\sigma_2\sim \mbox{Unif}([2,4])$. In the case of $m=3$, we sample $M_1,M_2,\sigma_1,\sigma_2$ as before, and sample $M_3\sim \mbox{Unif}([-2,0])$ and $\sigma_3\sim \mbox{Unif}([2,4]).$ We found that the $\ell^2$-distance between the coefficients essentially does not decay in $n$ (top right). This is due to the fact that in one dimension, the matrix $A^\epsilon_{\mu}$ is often singular, which violates the conditions in Theorem \ref{thm:coeff_theorem}, leading to problems with non-unique recovery of coefficients. Despite this, we empirically show that the expected $OT_2^2$-error between the barycenters associated to $\lambda$ and $\hat{\lambda}^n$ rapidly decays as $n$ increases, implying successful recovery of a \emph{valid} set of barycentric coefficients even when they are non-unique (see Appendix \ref{sec:nonunique_coeffs}).\\ 

\noindent\textbf{5-Dimensional Gaussians:}  We work with a fixed, random choice of $\lambda$, which we obtain by sampling $\tilde{\lambda}_1,\tilde{\lambda}_2,...,\tilde{\lambda}_m$ i.i.d. $\sim \mbox{Unif}([0,1])$, and define $\lambda_j:=\tilde{\lambda}_j/\sum_{i=1}^m\tilde{\lambda}_i$. For each trial, we set $\epsilon=1$, and generate each Gaussian $\nu_j:=\mathcal{N}(M_j,A_j)$ as follows.  We sample $M_1\sim \mbox{Unif}([0.2,0.4]^5),M_2\sim \mbox{Unif}([0.4,0.6]^5)$ and $M_3\sim \mbox{Unif}([0.6,0.8]^5)$. 

For each $j$, we generate covariance matrix $\Sigma_j=Rv_jv_j^\top+0.7*Id$, where $v_j\in \mathbb{R}^5$ is generated by sampling $[v_j]_{i}\sim Unif([0.2,0.5])$ for $1\leq i\leq 5$ and $R$ is a uniformly random $d$-dimensional rotation matrix. We work with a fixed choice of coefficients $\lambda$ across all trials, which we generate by sampling $\tilde{\lambda_i}\sim \mbox{Unif}([0,1])$ and setting $\lambda_i=\tilde{\lambda}_i/\sum_{j=1}^m \tilde{\lambda}_j$.  Unlike the 1D case, we no longer have access to the closed-form solution for the entropic barycenter. We implement a free support synthesis algorithm (see details in Appendix \ref{sec:synthesis_algo}) to compute an approximate fixed point $\mu^*$. We then sample $2n$ samples from $\mu^*$ and resample $n$ samples from $\{\nu_j\}_{j=1}^m$ and use these as inputs for Algorithm \ref{alg:coefficient_recovery}, as we let $n$ range over $[10,20,40,...,10240]$.\\

\noindent{\textbf{5-Dimensional Cubes}:}  We work with a fixed, random choice of $\lambda$, which we obtain by sampling $\tilde{\lambda}_1,\tilde{\lambda}_2,...,\tilde{\lambda}_m$ i.i.d. $\sim \mbox{Unif}([0,1])$, and define $\lambda_j:=\tilde{\lambda}_j/\sum_{i=1}^m\tilde{\lambda}_i$. For each trial, we set $\epsilon=1$, and generate our reference measures using random translations of $Unif([0,2]^5).$ Specifically, we let $\nu_j=Unif([0,2]^5+M_j])$ where $M_1\sim \mbox{Unif}([0.2,0.4]^5),M_2\sim \mbox{Unif}([0.4,0.6]^5)$ and $M_3\sim \mbox{Unif}([0.6,0.8]^5)$. \\

We solve the minimization problem $\min_{\lambda\in \Delta^3}\lambda^\top \hat{A}_{\mu^*}^\epsilon \lambda$ using the \texttt{cp.Minimize} function from the convex optimization module cvx.py, and denote the output $\hat{\lambda}^n$.

\subsection{Point Cloud Classification Implementation Details}\label{sec:pointcloud implementation}

\begin{algorithm}[t!]
    \caption{Classification Experiment}\label{alg:Classification Experiment}
    \KwData{
        $m$ classes;\\
        Labeled point clouds for each class: $\mathcal{A}^y = \{P_1^y, P_2^y, \ldots, P_q^y\}$ for $1 \leq y \leq m$;\\
        Number of reference measures per class: $b$;\\
        Train-test split numbers: $n_{train} + n_{test} = q$;\\
        Regularization parameter: $\epsilon$;\\
        Functional: $\mathcal{F}^\epsilon_{\lambda, \mathcal{V}} \in \{F^\epsilon_{\lambda, \mathcal{V}}, S^\epsilon_{\lambda, \mathcal{V}}\}$;\\
        Number of iterations: $K$;\\
    }
    \vspace{2mm} 
    \hrule 
    \vspace{2mm} 

    \For{$1 \leq i \leq K$}{
        \For{$1 \leq y \leq m$}{
            Partition $\mathcal{A}^y$ into training and testing sets, $\mathcal{A}^y_{train}$ and $\mathcal{A}^y_{test}$, with $|\mathcal{A}^y_{train}| = n_{train}$ and $|\mathcal{A}^y_{test}| = n_{test}$\;
Randomly sample $b$ point clouds $\mathcal{V}^y:=\{P^y_{j_1}, P^y_{j_2}, \ldots, P^y_{j_b}\}\subset \mathcal{A}^y_{train}$\;
        }
        
            \For{$1 \leq y \leq m$}{
                \For{$Q \in \mathcal{A}^y_{test}$}{
                        
                    $\tilde{y}_Q \leftarrow$ Output of Algorithm~\ref{alg:point_cloud_classification} with inputs: unlabelled point cloud $Q$, $m$ classes, labeled point clouds $\{\mathcal{V}^y\}_{y=1}^m$, regularization parameter $\epsilon$, and functional $\mathcal{F}^\epsilon_{\lambda, \mathcal{V}}$\;
                    
                    $s_{i,Q}^y \leftarrow \begin{cases} 
                        1 & \text{if } \tilde{y}_Q = y \\
                        0 & \text{otherwise} 
                    \end{cases}$\;
                }
            }
            $s_{i}\leftarrow \frac{1}{m n_{test}}\sum_{y=1}^m\sum_{Q\in \mathcal{A}^y_{test}}s_{i,Q}^y$
        }
        
    $s \leftarrow \frac{1}{K} \sum_{i=1}^{K} s_i$\;

\vspace{1mm} 
    \hrule 
    \vspace{1mm}
    
    \textbf{Output:} Overall estimated classification accuracy $s$.
\end{algorithm}

\textbf{Analysis for the barycenter functional:}  We compute the barycentric coefficients for the unregularized barycenter functional by computing a set of maps $\{T^0_{\mu\rightarrow \nu_j}\}_{j=1}^m$, where $T^0_{\mu\rightarrow \nu_j}:=\mathbb{E}_{\zeta_j}(Y|X=x)$ is the conditional expectation of an optimal coupling $\zeta_j\in \Pi(\mu,\nu_j)$ for all $1\leq j\leq m$, which we use as inputs to define a matrix $A^0_\mu$, defined by $[A_\mu^0]_{ij}= \langle Id-T^0_{\mu\rightarrow \nu_i},Id-T^0_{\mu\rightarrow \nu_j}\rangle_{L^2(\mu)}$. We solve the corresponding optimization problem $\min_{\lambda \in\Delta^m}\lambda^\top A_\mu^0\lambda$.  \\

\noindent \textbf{Analysis for the doubly-regularized functional: }For the doubly regularized problem, we must assume that $\mu$ admits a density $P_\mu$. Then the corresponding matrix is $D^{\epsilon,\tau}_\mu$, defined by $[D_\mu^{\epsilon,\tau}]_{ij}=\langle Id-T^\epsilon_{\mu\rightarrow \nu_i}+\tau\nabla\log P_{\mu},Id-T^\epsilon_{\mu\rightarrow \nu_j}+\tau\nabla\log P_{\mu}\rangle_{L^2(\mu)}$; see Appendix \ref{sec:derive_doubly_reg} for details.  In practice, we estimate $\nabla \log P_\mu$ using the $\nu$-method implemented in \citep{zhou2020nonparametric}. Our choice of kernel is the curl-free inverse multiquadric kernel with bandwidth 5.0, and regularization strength $10^{-5}$ (see \citep{zhou2020nonparametric} and the associated implementation for details). We additionally convolved $\mu$ with isotropic Gaussian noise with standard deviation $0.05$, which empirically improved the performance of our method.

\section{RECOVERY WITH NON-UNIQUE BARYCENTRIC COEFFICIENTS}\label{sec:nonunique_coeffs}

We further study the analysis problem for three one-dimensional Gaussians. In Section \ref{sec:numericalexp}, we showed that coefficient recovery fails in this setting. We claim this is due to the fact that there are multiple sets of coefficients giving rise to the same barycenter, i.e., $A_\mu^\epsilon$ is often singular. 

Let $\mu^*$ be the barycenter with respect to the original coefficients $\lambda$ and $\hat{\mu}^n$ be the barycenter with the recovered coefficients $\hat{\lambda}^n$ as $n$ ranges over the same set of values. To estimate $\mathbb{E}(OT_2^2(\mu^*,\hat{\mu}^n))$, we estimate the mean and variance of $\hat{\mu}^n$ (denoted $\hat{m},\hat{\sigma}$) and applying the closed-form formula for the optimal transport distance between two Gaussians \citep{chen2018optimal}, which specializes in this case to $OT_2^2(\mu^*,\hat{\mu}^n)=(\bar{\mu}-\hat{m})^2+S^2+\hat{\sigma}^2-2\hat{\sigma}S$. We average this quantity over 100 trials to estimate $\mathbb{E}(OT_2^2(\mu^*,\hat{\mu}^n))$. We plot our results in Figure \ref{fig:eps2_3gaussw2}, and see that $\mathbb{E}(OT_2^2(\mu^*,\hat{\mu}^n))$ decays faster than $n^{-1}$. Hence from the perspective of barycentric coefficients, our method is successful even in the case of non-unique recovery.

\begin{figure}[t!]
\centering
\includegraphics[width=0.5\textwidth]{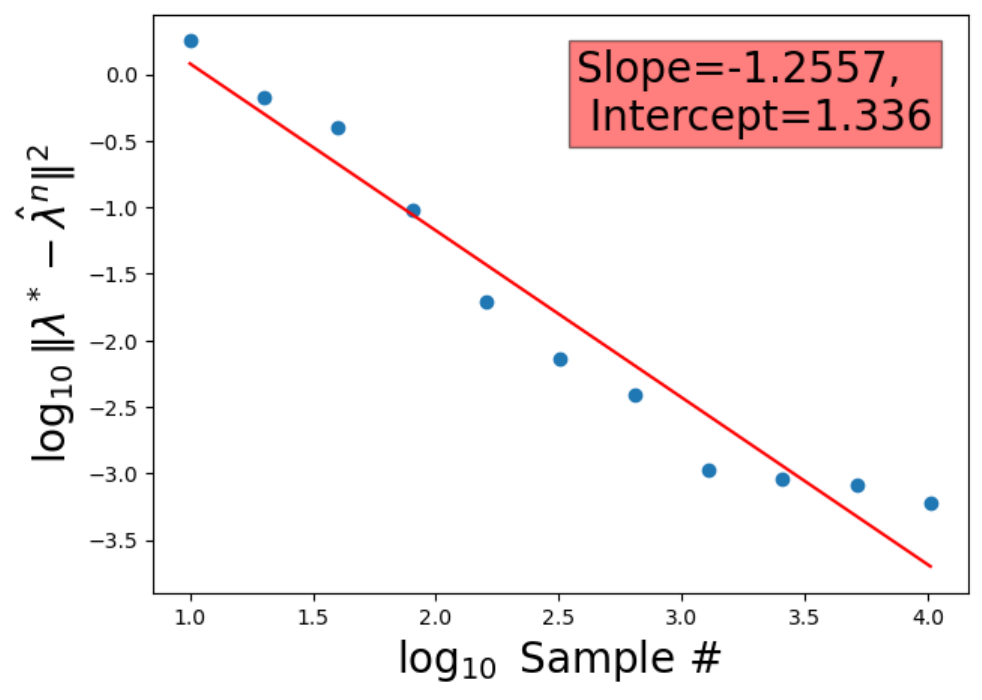}
\caption{Average $\log$ squared-$OT_2$-loss for three random 1D Gaussian measures with random weights, $\epsilon=2$.}\label{fig:eps2_3gaussw2}
\end{figure}

\section{SYNTHESIS AND ANALYSIS WITH MNIST}\label{sec:synth_and_analysis_MNIST}

We give an in-depth example of both synthesis and analysis using the entropic barycenter functional and probability measures constructed from the MNIST data set \citep{lecun1998mnist}. Given images of digits $D_1,D_2,...,D_m$, we construct reference measures $\nu_1,\nu_2,...,\nu_m$ using Algorithm \ref{alg:image_to_measure}. We fix once and for all $\epsilon=0.0002$. For a given $\lambda\in \Delta^m$ and an initial measure $\mu_0 \in \mathcal{P}_2(\mathbb{R}^d)$, we run Algorithm \ref{alg:WGD} ($k=1000$, $\eta=0.025$) to optimize $F^\epsilon_{\lambda,\mathcal{V}}$, outputting an approximate critical point $\mu^{\lambda}_k$. We then compute the matrix $A_{\mu_k}^\epsilon$ (see Proposition \ref{prop:analysis_grad}) and perform the minimization $\text{min}_{\lambda\in\Delta^3} \lambda^\top A^\epsilon_\mu\lambda$ using the \texttt{cp.Minimize} function from the convex optimization module cvx.py, and denote the output $\hat{\lambda}^i$. We perform this for $\{\lambda^i\}_{i=1}^{21}$ ranging over $21$ evenly spaced points in the simplex, which correspond to all tuples of the form $\{(\frac{i_1}{5},\frac{i_2}{5},\frac{i_3}{5})| \;i_1,i_2,i_3\in \mathbb{Z}_{\geq 0}, i+j+k=5\}$. In Figure \ref{fig:pyramid} we visualize our results. We convert $\mu_k^{\lambda^i}$ to an image using Algorithm \ref{alg:measure_to_image} which we position at $\lambda^i$ on the left hand triangle. At the corners of the triangle, we display the original images which were converted to reference measures. In the right triangle, we plot $\lambda^i$ (in red) and $\hat{\lambda}^i$ (in blue). 

From Figure \ref{fig:pyramid} we see that the fixed points produced by our synthesis method intuitively interpolate between the reference images. We note that there is some degree of shrinkage compared to the original digits, which could be due to the fact that fixed points are dominated in convex order by $\sum_{j=1}^3\lambda_j \nu_j$ by Proposition \ref{prop:domination}, where $\{\nu_j\}_{j=1}^{3}$ are the reference measures. From the triangle on the right we see that the recovered coefficients $\hat{\lambda}^i$ closely match the true coefficients $\lambda^i$. To verify this, we compute the average $\ell^2$-error $\frac{1}{21}\sum_{i=1}^{21}\|\lambda^i-\hat{\lambda}^i\|_2$ and find that it is approximately $1.269\times 10^{-5}$.

\subsection{Image-Measure Processing}
To convert images to measures and to visualize measures in $\mathbb{R}^2$ as images, we use Algorithms \ref{alg:image_to_measure} and \ref{alg:measure_to_image} respectively, which are adapted from the methods in \citep{werenski2024linearized}. Algorithm \ref{alg:image_to_measure} produces a measure by simply normalizing the intensities of an image (represented as a matrix). Algorithm \ref{alg:measure_to_image} first applies a Gaussian kernel with bandwidth $b=0.1$ to each of the data points, and builds a matrix $\mathtt{K}$ of shape $rd\times rd$, where $r=5$ is a resolution parameter and $d=28$ is the desired output size such that the $\mathtt{K}_{ij}$ is the density at $(\frac{i}{rd},\frac{j}{rd})$. The matrix is then reduced by removing entries below $\max_{ij}\texttt{K}_{ij}*\ell$, where $\ell=0.0002$ is a tolerance parameter. We then reduce $\texttt{K}$ to a $d\times d$ matrix $\texttt{I}$ and normalize intensities to return an image.
\newpage
\begin{figure}[t!]
    \centering
    \rotatebox{270}{ 
        \includegraphics[width=0.45\textwidth]{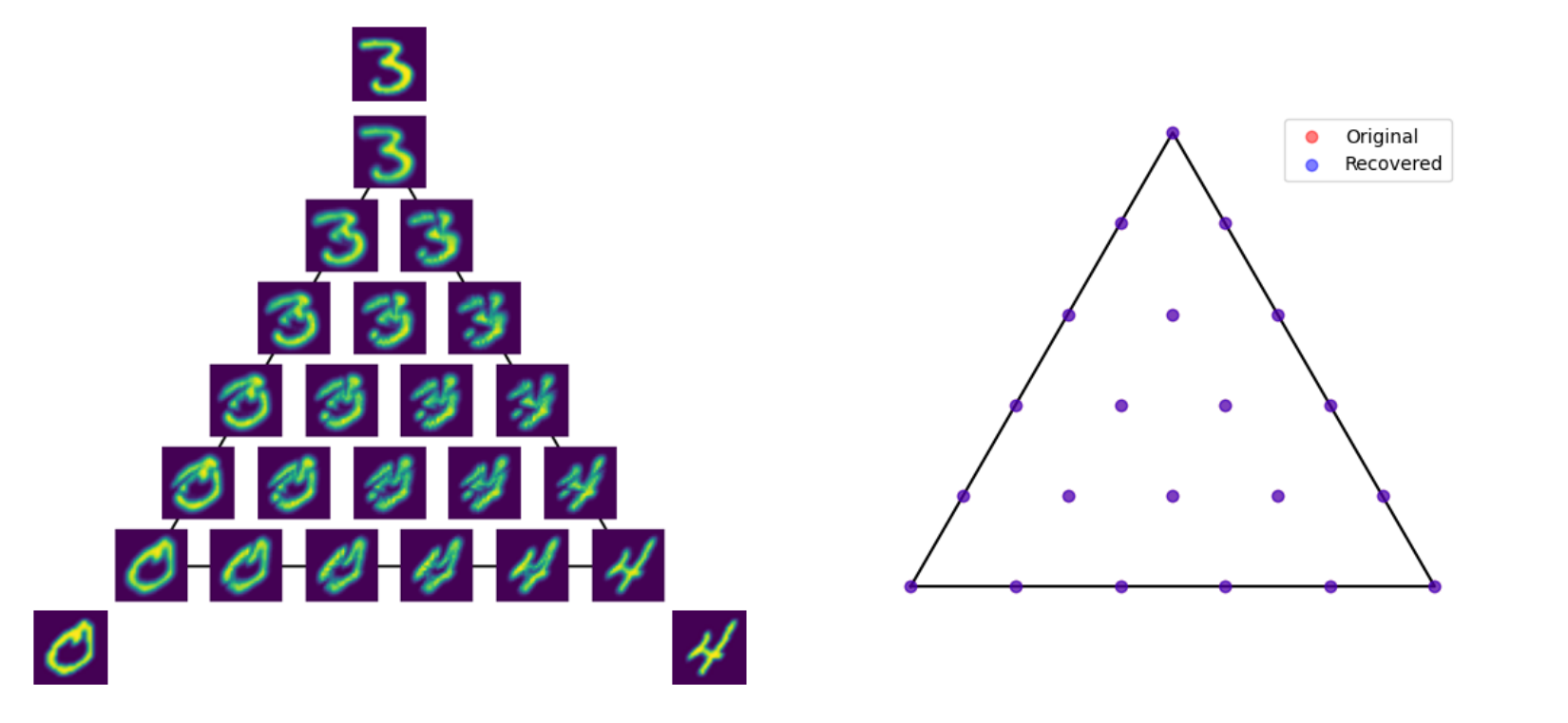} 
    }
    \caption{    \label{fig:pyramid}We show synthesized entropy-regularized barycenters on the left, with the reference measures on the corners.  On the right, original coefficients are shown as well as what our analysis algorithm recovers, showing good recovery.}

\end{figure}

\begin{algorithm}[t]
    \caption{Image to Measure Conversion}\label{alg:image_to_measure}
    \KwData{Matrix representation of an image $\mathtt{I} \in \mathbb{R}_+^{d \times d}$.}
    \vspace{2mm} 
    \hrule 
    \vspace{2mm} 
    
    Set $s \leftarrow \sum_{i,j=1}^d \mathtt{I}_{ij}$\;

    \vspace{1mm} 
    \hrule 
    \vspace{1mm}
    \textbf{Output:} $\frac{1}{s} \sum_{i,j=1}^d \mathtt{I}_{ij} \delta_{(i/d,j/d)}$\;
\end{algorithm}

\begin{algorithm}[t]
    \caption{Measure to Image Conversion}\label{alg:measure_to_image}
    \KwData{Discrete measure $\nu\in \mathcal{P}_2([0,1]^2)$\\
            Output size $d$; \\
            Resolution $r$; \\
            Bandwidth $b$; \\
            Tolerance $\ell$.
    }
    \vspace{2mm} 
    \hrule 
    \vspace{2mm} 
     Define $S=\{S_1,S_2,...,S_n\}\leftarrow \texttt{supp}(\nu)$, $S_i\in [0,1]^2$;

    Define $b_k\leftarrow \nu(S_k)$;

    Create $\mathtt{K} \in \mathbb{R}^{rd \times rd}_+$ with  
    \[
    \mathtt{K}_{ij} \leftarrow \sum_{k=1}^n b_k \exp\left ( -\frac{\|(i/rd,j/rd) - S_k\|_2^2}{b^2} \right )
    \]
    
    Set $\mathtt{K} \leftarrow  \left ( \sum_{i,j=1}^{rd} \mathtt{K}_{ij} \right )^{-1}\cdot\mathtt{K}$\;
    
    \For{ $1\leq i,j\leq rd$}{
        $\mathtt{K}_{ij} \leftarrow \mathtt{K}_{ij} \cdot \pmb{1}[\mathtt{K}_{ij} > \max_{i,j}\mathtt{K}_{ij}*\ell]$\;
    }
    
    Create $\mathtt{I} \in \mathbb{R}^{d\times d}$ with  
    \[
    \mathtt{I}_{ij} \leftarrow \sum_{k,l=1}^r \mathtt{K}_{(i-1)*r+k, (l-1)*r+l}
    \]
    \vspace{1mm} 
    \hrule 
    \vspace{1mm}
    \textbf{Output:} $\left ( \sum_{i,j=1}^d \mathtt{I}_{ij }\right )^{-1}\cdot\mathtt{I}$\;
\end{algorithm}

\newpage

\section{DERIVATION OF QUADRATIC PROGRAM FOR DOUBLY-REGULARIZED FUNCTIONAL}\label{sec:derive_doubly_reg}

In this section, we derive the form of the quadratic program used for solving the analysis problem for the $(\epsilon,\tau)$-doubly-regularized barycenter functional, $\mathcal{D}_{\lambda,\mathcal{V}}^{\epsilon,\tau}:=F^\epsilon_{\lambda,\mathcal{V}}+\tau H(\mu)$, under the assumption that the measure $\mu$ being analyzed is a critical point of the doubly-regularized functional. Here, $H$ denotes the negative differential entropy: 
\begin{equation*}\label{relative_entropy}
H(\mu) := 
\Bigg\{
    \begin{array}{lr}
        \displaystyle\int_{\mathbb{R}^d}\log\left(P_\mu(x)\right)P_\mu(x)dx, \text{if}\; \mu \; \text{is absolutely continuous with respect to the Lebesgue measure on $\Omega$}\; \\
        \;+\infty, \;\text{otherwise}
    \end{array}
\end{equation*}
where $P_\mu$ denotes the density of $\mu$.

Let $\mu\in \mathcal{P}_2(\mathbb{R}^d)$ be absolutely continuous with connected support. For convenience, we will also assume that $\mu$ has $C^\infty(\mathbb{R}^d)$ density $P_\mu$. Then we compute: 
\begin{align*} \delta\mathcal{D}^{\epsilon,\tau}_{\lambda,\mathcal{V}}(\mu)&=\delta F^\epsilon_{\lambda,\mathcal{V}}(\mu)+\delta \tau H(\mu) \\ & =\sum_{j=1}^m\lambda_j f^\epsilon_{\mu\rightarrow \nu_j}+\tau \log P_\mu.
\end{align*}
Since $\mu$ is a critical point, we have that $\|\nabla\delta \mathcal{D}^{\epsilon,\tau}_{\lambda,\mathcal{V}}(\mu)\|^2_{L^2(\mu)}=0$, which implies that:
\begin{align*}
Id-\sum_{j=1}^m\lambda_j T^\epsilon_{\mu\rightarrow \nu_j}+\tau\nabla \log P_\mu =0
\end{align*}
$\mu$-almost everywhere. We may then compute:
\begin{align*}
\|\nabla\delta \mathcal{D}^{\epsilon,\tau}_{\lambda,\mathcal{V}}(\mu)\|_{L^2(\mu)}^2 & = \|Id -\sum_{j=1}^m\lambda_j T^\epsilon_{\mu\rightarrow \nu_j}+\tau\nabla \log P_\mu \|^2_{L^2(\mu)} \\ &  =\int \langle Id-\sum_{j=1}^m\lambda_j T^\epsilon_{\mu\rightarrow \nu_j}+\tau\nabla \log P_\mu,Id-\sum_{k=1}^m\lambda_k T^\epsilon_{\mu\rightarrow \nu_k}+\tau\nabla \log P_\mu\rangle d\mu \\ & =\int \sum_{j,k=1}^m\lambda_j\lambda_k\langle Id-T^\epsilon_{\mu\rightarrow \nu_j}+\tau\nabla \log P_\mu , Id-T^\epsilon_{\mu\rightarrow \nu_k}+\tau\nabla \log P_\mu\rangle d\mu \\& = \sum_{j,k=1}^m\lambda_j\lambda_k [D^{\epsilon,\tau}_\mu]_{jk}=\lambda^\top D^{\epsilon,\tau}_\mu \lambda,
\end{align*}
where $D^{\epsilon,\tau}_\mu$ is defined by $[D^{\epsilon,\tau}_\mu]_{jk}:=\langle Id-T^\epsilon_{\mu\rightarrow \nu_j}+\tau\nabla \log P_\mu, Id-T^\epsilon_{\mu\rightarrow \nu_k}+\tau\nabla \log P_\mu\rangle_{L^2(\mu)}$. Hence if $\mu$ is a critical point for $\mathcal{D}^{\epsilon,\tau}_{\lambda,\mathcal{V}}$, then solving the analysis problem is equivalent to solving $0=\lambda^\top D^{\epsilon,\tau}_\mu \lambda$ for $\lambda\in \Delta^m$. Note that the assumed smoothness of $P_{\mu}$ and the guaranteed smoothness of the potentials, together with the connectedness of $\texttt{supp}(\mu)$ imply the reverse implication as well: $\|\nabla\delta \mathcal{D}_{\lambda,V}^{\epsilon,\tau}(\mu)\|_{L^{2}(\mu)}=0$ implies $\mu$ is a critical point.

\section{APPLICATION: POINT CLOUD COMPLETION}\label{sec:completion}
We include a preliminary demonstration of an application of the synthesis and analysis methods to point cloud completion. We work with the PointCloud-C data set \citep{ren2022benchmarking} which consists of point cloud representations of objects corrupted by various regimes of noise and occlusions. In this demonstration, we use data from the \texttt{clean}, \texttt{dropout$\_$global$\_$4} and \texttt{dropout$\_$local$\_$4} datasets. The \texttt{clean} dataset consists of point clouds with $1024$ points. The \texttt{dropout$\_$global$\_$4} dataset is obtained by selecting point clouds from the \texttt{clean} dataset and (uniformly) randomly removing $75\%$ of the data points. The \texttt{dropout$\_$local$\_$4} dataset by selecting a random set of points $x_1,x_2,...,x_\ell$ ($\ell\leq 8$) and removing the $N_i$ nearest points to each $x_i$, where $\sum_{i=1}^\ell N_i=500$ (see \citep{ren2022benchmarking} for more details). We give some examples in Figure \ref{fig:pointclouds}. The object of point cloud completion is to take a partial point cloud and to generate a new point cloud which as closely as possible matches the original point cloud. Unlike most methods in the literature, our method does not require explicit training.

\begin{figure}[t!]
    \centering
    \begin{minipage}[b]{0.3\textwidth}
        \centering
        \includegraphics[width=\textwidth]{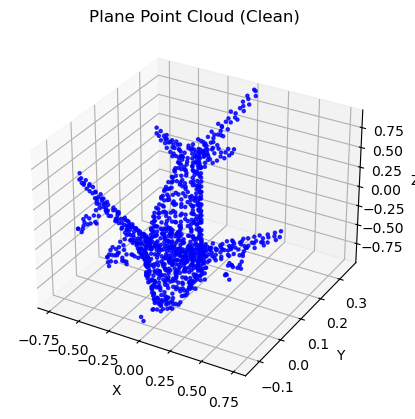}
    \end{minipage}
    \hfill
    \begin{minipage}[b]{0.3\textwidth}
        \centering
        \includegraphics[width=\textwidth]{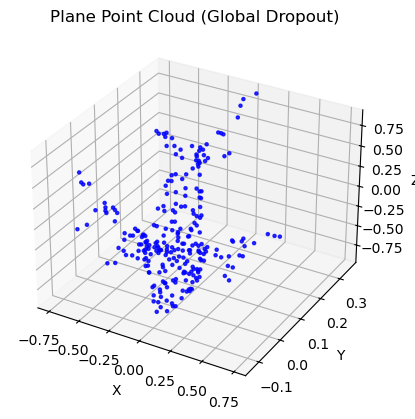}
    \end{minipage}
    \hfill
    \begin{minipage}[b]{0.3\textwidth}
        \centering
        \includegraphics[width=\textwidth]{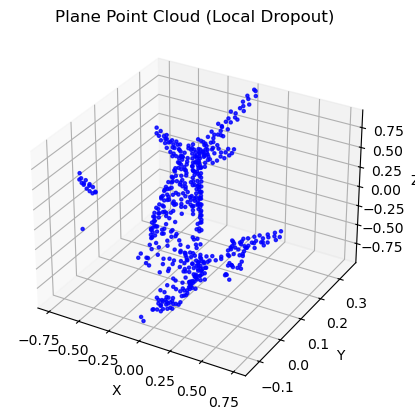}
    \end{minipage}
    \caption{(Left to right) Point clouds from \texttt{clean}, \texttt{dropout$\_$global$\_$4} and \texttt{local$\_$dropout$\_$4}.}\label{fig:pointclouds}
\end{figure}

Out method proceeds as follows.  We let $\mu_0$ be the uniform measure on a corrupted point cloud. We then select a set of random clean point clouds  $\nu_1,...,\nu_5$ from the same class of objects. We solve the analysis problem for $\mu_0$ with Algorithm \ref{alg:coefficient_recovery}, producing a set of coefficients $\lambda\in\Delta^{5}$. We then use these coefficients and reference measures to synthesize a new point cloud $\bar{\mu}_0$, using Algorithm \ref{alg:WGD}, initialized at an empirical measure (sampled from a uniform measure on the 3D cube) supported on $1024$ points. We remark that this method is similar to one developed in \citep{werenski2022measure} for image completion, with a key difference being that their method operates in a fixed-support setting, as opposed to free-support like ours. 

\begin{figure}[t!]
    \centering
    \begin{minipage}[b]{0.45\textwidth}
        \centering
        \includegraphics[width=\textwidth]{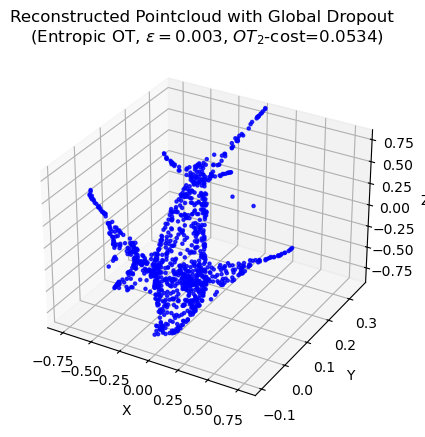}
    \end{minipage}
    \begin{minipage}[b]{0.45\textwidth}
        \centering
        \includegraphics[width=\textwidth]{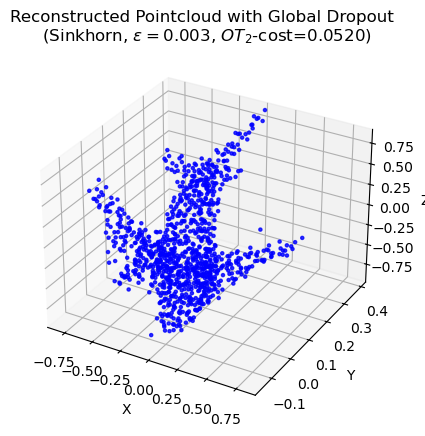}
    \end{minipage}
    \vspace{0pt} 
    \begin{minipage}[b]{0.45\textwidth}
        \centering
        \includegraphics[width=\textwidth]{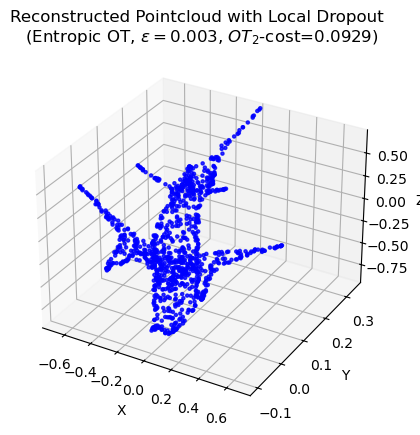}
    \end{minipage}
    \begin{minipage}[b]{0.45\textwidth}
        \centering
        \includegraphics[width=\textwidth]{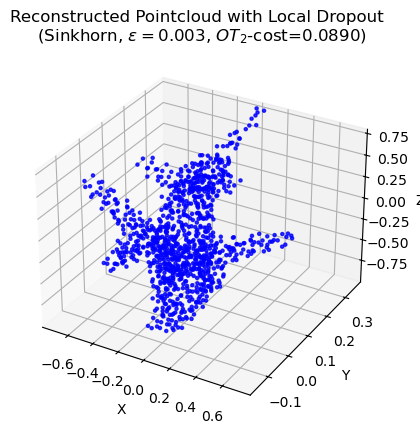}
    \end{minipage}
    \caption{(Top left) Global dropout, $F^\epsilon_{\lambda,\mathcal{V}}$ reconstruction, (top right) global dropout $S^\epsilon_{\lambda,\mathcal{V}}$ reconstruction, (bottom left) local dropout, $F^\epsilon_{\lambda,\mathcal{V}}$ reconstruction, (bottom right) local dropout $S^\epsilon_{\lambda,\mathcal{V}}$ reconstruction. $OT_2$-cost denotes the $OT_2$ distance between the (uniform measures on the) clean point cloud and the reconstruction.}
    \label{fig:pointclouds_completed}
\end{figure}

In Figure \ref{fig:pointclouds_completed} we show results for our completion method using both $F^\epsilon_{\lambda,\mathcal{V}}$ and $S^\epsilon_{\lambda,\mathcal{V}}$. We see that in all cases our methods are able to produce a completed point cloud that roughly match the original. The reconstructions using $F^\epsilon_{\lambda,\mathcal{V}}$ appear to be more singular than the original, whereas $S^\epsilon_{\lambda,\mathcal{V}}$ appear to be more diffuse. However some fine details of the original point cloud are not recovered, such as the second turbine on the left wing.  

To give a quantitative comparison, we compute the $OT_2$-distance between the (uniform measures on the) clean point cloud and the reconstruction. We see that the Sinkhorn reconstruction method performs slightly better in both cases with respect to this metric.

\end{document}